\documentclass[colorlinks]{article}
\usepackage[scaled=0.75]{beramono}
\usepackage[T1]{fontenc}
\usepackage{textcomp}
\usepackage[utf8]{inputenc}
\usepackage{xcolor}
\usepackage{array}
\usepackage{float}
\usepackage{wrapfig}
\usepackage{mathtools}
\usepackage{multirow}
\usepackage{amsmath}
\usepackage{amsthm}
\usepackage{graphicx}
\usepackage{setspace}
\usepackage[numbers]{natbib}
\usepackage{microtype}
\usepackage[unicode=true,
 bookmarks=true,bookmarksnumbered=true,bookmarksopen=true,bookmarksopenlevel=1,
 breaklinks=false,pdfborder={0 0 0},pdfborderstyle={},backref=false,colorlinks=true]
 {hyperref}

\makeatletter

\providecommand{\tabularnewline}{\\}
\floatstyle{ruled}
\newfloat{algorithm}{tbp}{loa}
\providecommand{\algorithmname}{Algorithm}
\floatname{algorithm}{\protect\algorithmname}

\theoremstyle{plain}
\newtheorem{thm}{\protect\theoremname}
\theoremstyle{plain}
\newtheorem{cor}[thm]{\protect\corollaryname}

\usepackage{xcolor}
\definecolor{mycol}{rgb}{0,0,0.65}
\hypersetup{
    colorlinks,
    linkcolor={mycol},
    citecolor={mycol},
    urlcolor={mycol}
}

\usepackage{amsfonts,bm,amsmath,amssymb}
\usepackage{lipsum}
\usepackage{nicefrac}

\PassOptionsToPackage{numbers, compress}{natbib}
\usepackage{natbib}

\usepackage{mathtools} 
\usepackage{abraces} 

\usepackage[main, final]{neurips_2025}

\usepackage{url}
\usepackage{amsfonts}
\usepackage{nicefrac}

\newcommand{\renewtheorem}[1]{%
  \expandafter\let\csname #1\endcsname\relax
  \expandafter\let\csname c@#1\endcsname\relax
  \expandafter\let\csname end#1\endcsname\relax
  \newtheorem{#1}%
}

\DeclareMathAlphabet{\mathbfsf}{\encodingdefault}{\sfdefault}{bx}{n}
\newcommand{\upgreektemplate}[2]{#2{
\renewcommand{\alpha}{\upalpha}
\renewcommand{\beta}{\upbeta}
\renewcommand{\theta}{\uptheta}
\renewcommand{\gamma}{\upgamma}
\renewcommand{\lambda}{\uplambda}
\renewcommand{\xi}{\upxi}
\renewcommand{\epsilon}{\upepsilon}
\renewcommand{\delta}{\updelta}
\renewcommand{\phi}{\upphi}
\renewcommand{\zeta}{\upzeta}
\renewcommand{\Lambda}{\Uplambda}
\renewcommand{\Gamma}{\Upgamma}
\renewcommand{\Delta}{\Updelta}
\renewcommand{\Theta}{\Uptheta}
#1
}}
\newcommand{\upgreek}[1]{\upgreektemplate{#1}{\mathsf}}
\newcommand{\bupgreek}[1]{\upgreektemplate{#1}{\mathbfsf}}
\usepackage{upgreek}

\usepackage{thmtools}
\usepackage{thm-restate}
\usepackage{hyperref}

\usepackage{enumitem}
\setlist[itemize]{leftmargin=*}
\setenumerate{itemsep=-2pt,topsep=-5pt,partopsep=-1pt,parsep=-1pt}
\setlist[enumerate]{leftmargin=*}

\usepackage[nameinlink,capitalize]{cleveref}
\AtBeginDocument{%
\let\ref\cref
}
\Crefname{equation}{Eq.}{Eqs.}
\Crefname{section}{Sec.}{Secs.}
\Crefname{algorithm}{Alg.}{Algs.}
\Crefname{thm}{Thm.}{Thms.}
\Crefname{cor}{Cor.}{Cors.}

\usepackage[utf8]{inputenc} 
\usepackage{fancyvrb}
\usepackage{tikz}
\usepackage{listingsutf8}

\usetikzlibrary{arrows.meta}

\definecolor{darkred}{rgb}{0.5,0,0}
\definecolor{darkgreen}{rgb}{0,0.5,0}
\definecolor{darkblue}{rgb}{0,0,0.5}
\definecolor{changes}{rgb}{0,0,0.75}
\definecolor{gray}{rgb}{0.5,0.5,0.5}

\usepackage{algorithm,algpseudocode}
\usepackage{setspace}

\makeatletter
\newcommand*{\inlineequation}[2][]{%
  \begingroup
    \refstepcounter{equation}%
    \ifx\\#1\\%
    \else
      \label{#1}%
    \fi
    \relpenalty=10000 %
    \binoppenalty=10000 %
    \ensuremath{%
      #2%
    }%
    ~\@eqnnum
  \endgroup
}
\makeatother

\ifdefined\showcaptionsetup
 \PassOptionsToPackage{caption=false}{subfig}
\fi
\usepackage{subfig}
\makeatother

\usepackage{listings}
\providecommand{\corollaryname}{Corollary}
\providecommand{\theoremname}{Theorem}

\begin{document}
\global\long\def\argmin{\operatornamewithlimits{argmin}}%

\global\long\def\argmax{\operatornamewithlimits{argmax}}%

\global\long\def\prox{\operatornamewithlimits{prox}}%

\global\long\def\diag{\operatorname{diag}}%

\global\long\def\lse{\operatorname{lse}}%

\global\long\def\R{\mathbb{R}}%

\global\long\def\E{\operatornamewithlimits{\mathbb{E}}}%

\global\long\def\P{\operatornamewithlimits{\mathbb{P}}}%

\global\long\def\V{\operatornamewithlimits{\mathbb{V}}}%

\global\long\def\N{\mathcal{N}}%

\global\long\def\L{\mathcal{L}}%

\global\long\def\C{\mathbb{C}}%

\global\long\def\tr{\operatorname{tr}}%

\global\long\def\norm#1{\left\Vert #1\right\Vert }%

\global\long\def\norms#1{\left\Vert #1\right\Vert ^{2}}%

\global\long\def\pars#1{\left(#1\right)}%

\global\long\def\pp#1{(#1)}%

\global\long\def\bracs#1{\left[#1\right]}%

\global\long\def\bb#1{[#1]}%

\global\long\def\verts#1{\left\vert #1\right\vert }%

\global\long\def\vv#1{\vert#1\vert}%

\global\long\def\Verts#1{\left\Vert #1\right\Vert }%

\global\long\def\VV#1{\Vert#1\Vert}%

\global\long\def\angs#1{\left\langle #1\right\rangle }%

\global\long\def\KL#1{[#1]}%

\global\long\def\KL#1#2{{\scriptstyle \operatorname{KL}}\hspace{-2pt}\pars{#1\middle\Vert#2}}%

\global\long\def\SKL#1#2{{\scriptstyle \operatorname{SKL}}\hspace{-2pt}\pars{#1\middle\Vert#2}}%

\global\long\def\mean{{\displaystyle \operatorname{ave}}}%

\global\long\def\div{\text{div}}%

\global\long\def\erf{\text{erf}}%

\global\long\def\vvec{\text{vec}}%

\global\long\def\b#1{\bm{#1}}%

\global\long\def\r#1{\upgreek{#1}}%

\global\long\def\br#1{\bupgreek{\bm{#1}}}%

\global\long\def\w{\b w}%

\global\long\def\v{\b v}%

\global\long\def\wr{\br w}%

\global\long\def\z{\b z}%

\global\long\def\y{\b y}%

\global\long\def\yr{\r y}%

\global\long\def\x{\b x}%

\global\long\def\xr{\r x}%

\global\long\def\zr{\r z}%

\global\long\def\h{\b h}%

\global\long\def\hr{\r h}%

\global\long\def\u{\b u}%

\global\long\def\ur{\r u}%

\global\long\def\gr{\r g}%

\global\long\def\hr{\r h}%

\global\long\def\ee{\r{\scalebox{.675}{\ensuremath{\varnothing}}}}%

\global\long\def\te{\r{\scalebox{.58}{\ensuremath{\varnothing}}}}%

\global\long\def\ttheta{\r{\scalebox{.815}{\ensuremath{\theta}}}}%

\global\long\def\tphi{\r{\scalebox{.62}{\ensuremath{\phi}}}}%

\global\long\def\talpha{\r{\scalebox{.68}{\ensuremath{\alpha}}}}%

\global\long\def\tbeta{\r{\scalebox{.68}{\ensuremath{\beta}}}}%

\global\long\def\tmu{\r{\scalebox{.68}{\ensuremath{\mu}}}}%

\global\long\def\tsigma{\r{\scalebox{.68}{\ensuremath{\sigma}}}}%

\title{Large Language Bayes}
\author{Justin Domke\\
University of Massachusetts Amherst}

\maketitle
\maketitle
\begin{abstract}
Many domain experts do not have the time or expertise to write formal
Bayesian models. This paper takes an informal problem description
as input, and combines a large language model and a probabilistic
programming language to define a joint distribution over formal models,
latent variables, and data. A posterior over latent variables follows
by conditioning on observed data and integrating over formal models.
This presents a challenging inference problem. We suggest an inference
recipe that amounts to generating many formal models from the large
language model, performing approximate inference on each, and then
doing a weighted average. This is justified and analyzed as a combination
of self-normalized importance sampling, MCMC, and importance-weighted
variational inference. Experimentally, this produces sensible predictions
from only data and an informal problem description, without the need
to specify a formal model.
\end{abstract}
{\ttfamily \fontdimen2\font=0.6em}

\section{Introduction}

\definecolor{col1}{HTML}{4c72b0}
\definecolor{col2}{HTML}{c44e52}
\definecolor{col3}{HTML}{55a868}
\definecolor{col4}{HTML}{dd8452}
\definecolor{col5}{HTML}{8172b3}
\definecolor{col6}{HTML}{937860}
\definecolor{col7}{HTML}{da8bc3}
\definecolor{col8}{HTML}{8c8c8c}
\definecolor{col9}{HTML}{ccb974}
\definecolor{col10}{HTML}{64b5cd}

\label{submission}

Why isn't Bayesian inference more popular? Arguably, a major reason
is simply that creating probabilistic models is hard. Most people
interested in analyzing data are neither programmers nor experts in
statistics. Yet, writing a probabilistic model requires learning a
probabilistic programming language (PPL), fluency with a range of
statistical distributions, and experience formalizing problem intuitions.
Even for experts, this is difficult and error-prone.

A natural idea is to ask the user to describe their problem in plain
language and then use a Large Language Model (LLM) to generate a formal
probabilistic model. While LLMs can write serviceable models, in practice
they struggle in the same way as humans---sometimes the models they
create are good and sometimes they aren't \citep{Price_2023_ChatGPT4WritesStan}.

Yet LLMs have one major advantage over humans: they don't mind being
asked to create many different candidate models for the same problem.

The central idea of this paper is to mathematically ``glue'' an
LLM to a PPL. Given an informal description of a problem, a joint
distribution is defined over (1) formal probabilistic models, (2)
observed data, and (3) unobserved target variables. We then condition
on data and marginalize out the space of formal models to get a final
posterior over target variables.

The main contributions of this paper are:
\begin{itemize}
\item A new problem definition, in which an informal description and a dataset
define a posterior via an LLM and a PPL. \eqref{sec:The-basic-idea}
\item A broad algorithmic recipe for solving the resulting inference problem.
\eqref{sec:Computational-Issues}
\item Experiments illustrating that the final approximated posterior captures
user intent and is typically better than taking a naive average of
formal models. \eqref{sec:Examples}
\item Theory analyzing the expected accuracy of the suggested inference
recipe. \eqref{sec:Theory}
\end{itemize}

\section{The basic idea\label{sec:The-basic-idea}}

\begin{figure*}[t]
\includegraphics[width=1\textwidth,page=35]{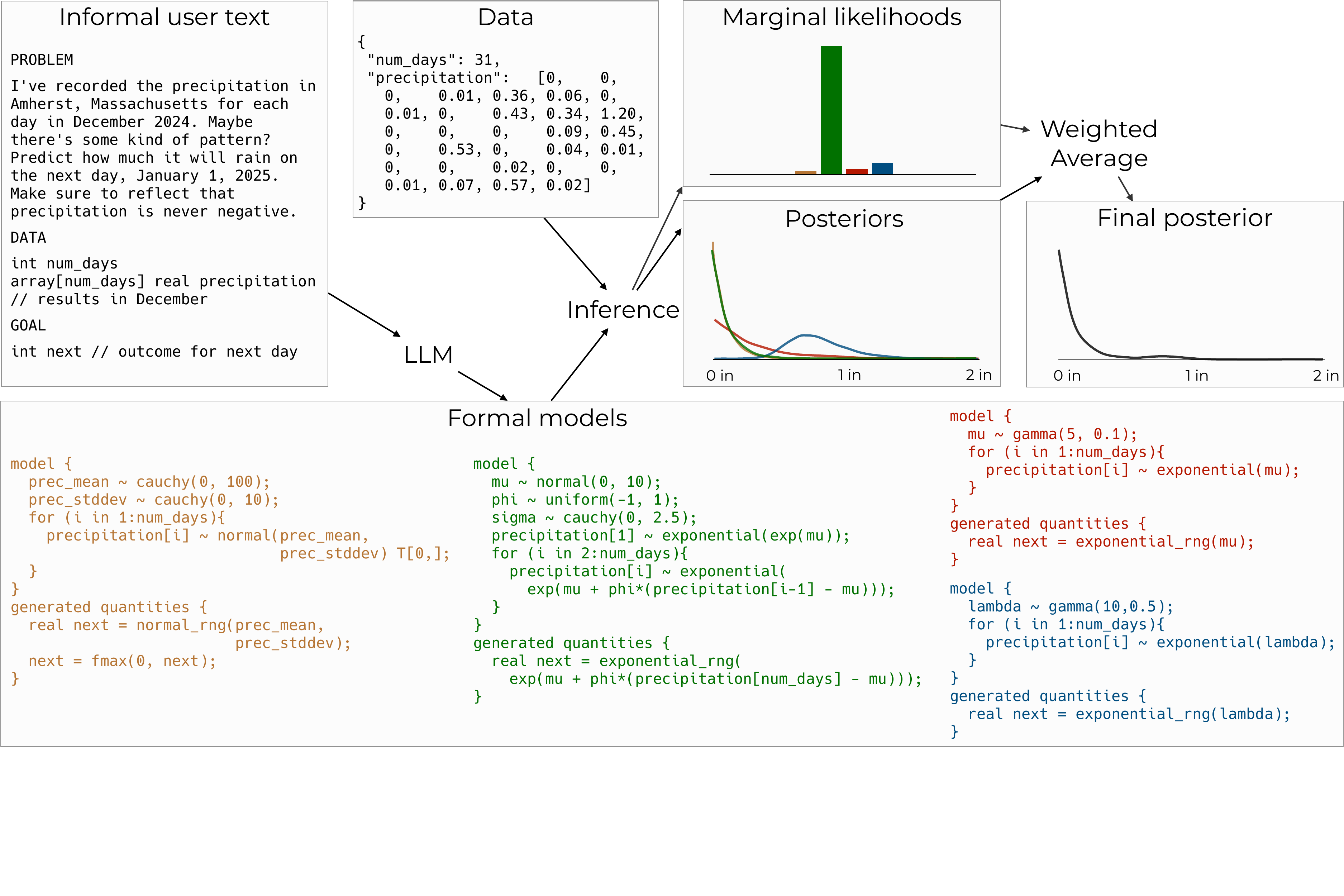}

\vspace{-1.55cm}

\caption{\textbf{The basic idea.} Given informal user text, an LLM generates
a set of candidate formal models. Inference is performed on each and
the posteriors are combined with weight proportional to the marginal
likelihood. Here four (real) LLM-generated formal models in the Stan
language are shown in different colors, with corresponding colors
for marginal likelihoods and posteriors.\label{fig:The-basic-idea.}}
\vspace{-0.1cm}
\end{figure*}

The user provides a plain-language description of the problem, including
anything they know about the phenomena under study, the types and
shapes of the input data, what assumptions should be used, and what
target variables should be predicted. For example, a user might write
the informal text shown at the top left of \ref{fig:The-basic-idea.}.
Given such a string, and asked to create models in the Stan PPL \citep{Carpenter_2017_StanProbabilisticProgramming},
an LLM could output any of the formal models shown at the bottom of
\ref{fig:The-basic-idea.}.

Denote the informal user text as $t$ and a formal model as $m$.
Since LLMs are stochastic, we can think of one (along with a system
prompt, hyperparameters, etc.) as defining a distribution
\begin{equation}
p\ \bigl(\underbrace{m}_{\substack{\text{formal}\\
\text{model}
}
}\vert\underbrace{t}_{\substack{\text{textual}\\
\text{description}
}
}\bigr).\label{eq:p_LLM}
\end{equation}

\vspace{-7pt}Meanwhile, we can think of a PPL as defining a distribution
over target variables $z$ and data $x$, conditional on the formal
model. That is, we may think of a PPL as defining a distribution
\begin{equation}
p\ \bigl(\underbrace{z}_{\substack{\text{target}\\
\text{variables}
}
},\underbrace{x}_{\substack{\text{observed}\\
\text{data}
}
}\vert\underbrace{m}_{\substack{\text{formal}\\
\text{model}
}
}\bigr).\label{eq:p_PPL}
\end{equation}

\vspace{-7pt}The core idea of this paper is to glue the above two
distributions together to define the joint
\begin{equation}
p(z,x,m|t)=\underbrace{p(m|t)}_{\text{LLM}}\ \underbrace{p(z,x|m)}_{\text{PPL}}.\label{eq:p(z_x_m|t)}
\end{equation}

\vspace{-5pt}Our hypothesis is that the distribution defined in \ref{eq:p(z_x_m|t)}
is a good one. We are interested in the posterior over $z$, conditioning
on $x$ and $t$, but integrating out $m$. It is not hard to show
that this is
\begin{equation}
\underbrace{p\pp{z|x,t}}_{\text{final posterior}}=\sum_{m}\ \underbrace{p\pp{m|x,t}}_{\substack{\text{posterior}\\
\text{model weight}
}
}\ \underbrace{p\pp{z\vert x,m}}_{\substack{\text{posterior of}\\
\text{model }m\text{ (PPL)}
}
}.\label{eq:p(z|x_t)}
\end{equation}
\vspace{-7pt}

Here, $p\pp{z|x,m}$ is the posterior for model $m$ (as defined by
the PPL in \ref{eq:p_PPL}) and
\begin{equation}
\underbrace{p\pp{m|x,t}}_{\substack{\text{posterior}\\
\text{model weight}
}
}\propto\underbrace{p\pp{m\vert t}}_{\substack{\text{prior model}\\
\text{weight (LLM)}
}
}\ \underbrace{p\pp{x\vert m}}_{\substack{\text{marginal likelihood}\\
\text{of model }m\text{ (PPL)}
}
}.\label{eq:p(m|x_t)}
\end{equation}

\vspace{-6pt}

Note that the final posterior in \ref{eq:p(z|x_t)} can be seen as
an instance of Bayesian model averaging \citep{Hoeting_1999_BayesianModelAveraging,Wasserman_2000_BayesianModelSelection},
just with a prior over models that's defined using an LLM.

\vspace{-7pt}

\subsection{Varying latent spaces}

Note that the user is \textit{not} expected to specify all latent
variables, only the target variables to predict. This is crucial since,
as illustrated in \ref{fig:The-basic-idea.}, different formal models
typically have different latent spaces. This is fine. All that's needed
is that each model produced by the LLM contains the target variables
$z$ specified in the user prompt $t$. Formally, suppose that model
$m$ defines some distribution
\begin{equation}
p\pp{z,x,u^{(m)}\vert m},\label{eq:varying_latent_space}
\end{equation}
where $u^{(m)}$ varies in meaning (and dimensionality) for different
models $m$. Then \Cref{eq:p_LLM,eq:p_PPL,eq:p(z_x_m|t),eq:p(z|x_t),eq:p(m|x_t)}
are all still correct. The final posterior remains as in \ref{eq:p(z|x_t)},
just with $p\pp{z,x\vert m}$ interpreted as \ref{eq:varying_latent_space}
after marginalizing out $u^{(m)}.$ The posterior weights in \ref{eq:p(m|x_t)}
make no reference to the latent space, and so need no modification.

\subsection{Comparison to flat averaging\label{subsec:Comparison-to-flat}}

It is useful to contrast the final posterior $p\pp{z\vert x,t}$ to
the result of taking a ``flat'' average of of models sampled from
the LLM, i.e.
\begin{equation}
p_{\text{flat}}\pp{z\vert x,t}=\sum_{m}\underbrace{p\pp{m\vert t}}_{\substack{\text{prior model}\\
\text{weight (LLM)}
}
}\ \underbrace{p\pp{z\vert x,m}}_{\substack{\text{posterior of}\\
\text{model }m\text{ (PPL)}
}
}.\label{eq:flat-posterior}
\end{equation}
\vspace{-5pt}

Clearly $p_{\text{flat}}\pp{z\vert x,t}$ is not equal to $p\pp{z\vert x,t}$
as defined by \ref{eq:p(z|x_t)}, but it can be seen as an ``ensemble
of posteriors''. The difference is that $p_{\mathrm{flat}}$ gives
model $m$ weight equal to the prior $p\pp{m\vert t}$ whereas $p\pp{z\vert x,t}$
gives weight that is also proportional to the marginal likelihood
$p\pp{x\vert m}$ (See \ref{eq:p(m|x_t)}). Thus, the fundamental
difference between the true posterior and ``flat'' posterior is
that the former gives more influence to models that are more consistent
with the observed data.

\section{Inference\label{sec:Computational-Issues}}

While mathematically simple, the posterior $p\pp{z\vert x,t}$ in
\ref{eq:p(z|x_t)} is computationally difficult. In principle, one
might imagine computing it as in \ref{alg:Exact-LLB}, but this will
rarely be practical. One familiar issue is that for a given model
$m$, it's usually difficult to compute the posterior $p\pp{z\vert x,m}$
or the marginal likelihood $p\pp{x\vert m}$. In addition, the space
of models $m$ is very large or possibly infinite.
\begin{algorithm}[t]
\begin{enumerate}
\item Input textual description $t$ and data $x$.
\item For all possible model strings $m$:
\begin{enumerate}
\item Compute model probability $p\pp{m|t}$.\hfill{}\textcolor{gray}{//
using LLM}
\item Compute posterior $p\pp{z\vert x,m}$ and marginal likelihood $p\pp{x\vert m}$.\hfill{}\textcolor{gray}{//
under PPL}
\end{enumerate}
\item Set $w^{(m)}\propto p\pp{m\vert t}p\pp{x\vert m}$, where $\sum_{m}w^{(m)}=1$.
\item Return final posterior $p\pp{z|x,t}=\sum_{m}w^{(m)}p\pp{z\vert x,m}.$
\end{enumerate}
\caption{Theoretical exact LLB algorithm (intractable)\label{alg:Exact-LLB}}
\end{algorithm}

A more subtle issue is that while we assume one can sample from $p\pp{m\vert t}$,
the probability $p\pp{m\vert t}$ is often unavailable. Many commercial
LLM providers decline to share this information. Further, it's often
beneficial to ask LLMs to ``think out loud'' before producing a
final answer. This too makes $p\pp{m\vert t}$ intractable, since
the same model could appear after many different ``thinking'' passages.

So, while the posterior in \ref{eq:p(z|x_t)} can be seen as an instance
of Bayesian model averaging, existing \textit{algorithms} for Bayesian
model averaging (e.g. reversible jump MCMC \citep[§11]{ChristianRobert_2004_MonteCarloStatistical})
seem difficult to apply, as they involve explicitly iterating over
all models and/or using evaluations of $p\pp{m\vert t}$. For these
reasons, $p\pp{z\vert x,t}$ appears to represent a novel inference
challenge.

To approximate the final posterior, this paper will use the recipe
in \ref{alg:practical-LLB}. The simplest interpretation of this recipe
is as a heuristic approximation of \ref{alg:Exact-LLB} where the
sum over all models is replaced by random sampling, and $\tilde{p}\pp{z\vert x,m}$
and $\tilde{p}\pp{x\vert m}$ denote approximations of the posterior
and marginal likelihood for a given model $m$. \ref{sec:Theory}
will give more a formal justification and analysis.

\begin{algorithm}[t]
\begin{enumerate}
\item Input textual description $t$ and data $x$.
\item For $n=1,2,\cdots,N$:
\begin{enumerate}
\item Sample model $m^{(n)}\sim p\pp{m\vert t}$.\hfill{}\textcolor{gray}{//
using LLM}
\item Approximate posterior $\tilde{p}\pp{z\vert x,m^{(n)}}$ and marginal
likelihood $\tilde{p}\pp{x\vert m^{(n)}}$.\hfill{}\textcolor{gray}{//
under PPL}
\end{enumerate}
\item Set $w^{(n)}\propto\tilde{p}\pp{x\vert m^{(n)}}$, where $\sum_{n=1}^{N}w^{(n)}=1$.
\item Return final approximate posterior $p\pp{z|x,t}\approx\sum_{n=1}^{N}w^{(n)}\tilde{p}\pp{z\vert x,m^{(n)}}.$
\end{enumerate}
\caption{Suggested generic approximate LLB recipe. \label{alg:practical-LLB}}
\end{algorithm}

\subsection{Preliminary observations}

We experimented with using various LLMs to generate formal Bayesian
models in various PPLs, including Stan \citep{Carpenter_2017_StanProbabilisticProgramming},
NumPyro \citep{Phan_2019_ComposableEffectsFlexible} and PyMC \citep{Abril-Pla_2023_PyMCModernComprehensive}.
LLMs seemed better at generating Stan code, perhaps since more Stan
code is available and included in LLM datasets. We also found it was
helpful to give LLMs a system prompt that instructed them to describe
their modeling strategy in words before producing a formal model.
Finally, to aid in-context learning, we found it was helpful to provide
examples of user prompts along with ``good'' responses.

\subsection{Generating models\label{subsec:Generating-models}}

These experiments used a system prompt \eqref{sec:System-prompt}
that told the LLM to first generate a ``thoughts'' block that would
describe the modeling strategy in words, and then create the formal
Stan model. A few additional instructions were used to avoid common
mistakes and to encourage the LLM to create well-normalized distributions.

For in-context learning, we provided the LLM with six examples of
ostensible user inputs, along with high quality outputs \eqref{sec:examples}.
Each input followed the format illustrated in \ref{fig:The-basic-idea.}
with {\small\texttt{PROBLEM}} block, describing the problem, a {\small\texttt{DATA}}
block describing the data, and a {\small\texttt{GOAL}} block describing
the target variable(s). Each output followed the instructions in the
system prompt, with a {\small\texttt{THOUGHTS}} block describing the
modeling strategy and a {\small\texttt{MODEL}} block with formal Stan
code.

For the problems below, 1024 models were generated using Llama-3.3-70B
\citep{Grattafiori_2024_Llama3Herd,Meta_2024_Llama33Model} with 4-bit
AWQ quantization \citep{lin2023awq}. Since many models were generated
from the same prompt, continuous batching (parallel inference) greatly
increased inference speed. Using a single A100 GPU, generating 1024
models took 10-15 minutes, depending on the problem.

\subsection{Inference}

Outputs from the LLM were rejected if they did not contain {\small\texttt{GOAL}}
and {\small\texttt{MODEL}} blocks or if the code in the model block
did not compile. For models that compiled, 2 chains of Stan's implementation
of NUTS \citep{Hoffman_2014_NoUturnSamplerAdaptively} were run for
10,000 iterations each. These samples were then considered as representing
the approximate posterior $\tilde{p}\pp{z\vert x,m^{(n)}}.$

It remains only to estimate the marginal likelihood $p\pp{x\vert m^{(n)}}.$
While methods exist to estimate the marginal likelihood from a set
of posterior samples \citep{Newton_1994_ApproximateBayesianInference},
these are often considered unreliable \citep{Neal_2008_HarmonicMeanLikelihood}.
Instead, we elected to use variational bounds. The typical approach
is to create some variational distribution $q\pp z$ (e.g. a Gaussian)
and optimize it to maximize the lower-bound
\begin{equation}
\mathrm{ELBO}\pp q=\E_{\r z\sim q}\log\frac{p\pp{\r z,x\vert m^{(n)}}}{q\pp{\r z}}\leq\log p\pp{x\vert m^{(n)}}.\label{eq:ELBO-q}
\end{equation}
This bound will be loose if the true posterior $p\pp{z\vert x,m^{(n)}}$
is not in the variational family. For a tighter-bound, we pursued
importance-weighted bounds \citep{Burda_2015_ImportanceWeightedAutoencoders}
of the form
\begin{equation}
L\pp q=\E_{\r z^{1},\cdots,\r z^{K}\sim q}\log\frac{1}{K}\sum_{k=1}^{K}\frac{p\pp{\r z^{k},x\vert m^{(n)}}}{q\pp{\r z^{k}}}\leq\log p\pp{x\vert m^{(n)}}.\label{eq:IW-ELBO}
\end{equation}
where $\r z^{1},\cdots,\r z^{K}$ are independent samples from $q$.

It remains to choose the distribution $q$. The typical approach with
importance-weighted bounds \citep{Burda_2015_ImportanceWeightedAutoencoders,Bachman_2015_TrainingDeepGenerative,Cremer_2017_ReinterpretingImportanceWeightedAutoencoders,Domke_2018_ImportanceWeightingVariational}
is to explicitly optimize $q$ to maximize \ref{eq:IW-ELBO}. However,
it is difficult to make stochastic optimization fully automatic and
reliable across thousands of heterogenous models. Instead, we exploit
the fact that we have already (approximately) sampled from $p\pp{z\vert x,m^{(n)}}$
using MCMC and simply set $q$ to be a Gaussian distribution matching
the empirical mean and variance of the MCMC samples.

This seemingly-heuristic choice is justified by the fact that \ref{eq:IW-ELBO}
is asymptotically tight (as $K\rightarrow\infty$) with an asymptotic
rate of $1/K$ and a constant determined by the $\chi^{2}$ divergence
between $q$ and $p$ \citep{Maddison_2017_FilteringVariationalObjectives,Domke_2018_ImportanceWeightingVariational}.
From the perspective of alpha-divergences, maximizing $\mathrm{ELBO}\pp q$
is equivalent to minimizing $\KL qp=D_{0}\pp{p\Vert q},$ an ``exclusive''
divergence. \ref{eq:IW-ELBO} would be (asymptotically) tightest if
we could minimize the ``inclusive'' divergence $\chi^{2}\pp{p\Vert q}=D_{2}\pp{p\Vert q}$,
but this is difficult to optimize \citep{Geffner_2021_DifficultyUnbiasedAlpha}.
Matching the mean and covariance of the posterior is equivalent to
minimizing $\KL qp=D_{1}\pp{p\Vert q}$,  and so is a pragmatic compromise.
Experimentally, this performed at least as well as explicitly optimizing
\ref{eq:IW-ELBO}, but was faster and more reliable.

The full inference recipe is summarized as \ref{alg:LLB-using-MCMC-and-VI}.
More details are in \ref{subsec:Experimental-Details}.

\begin{algorithm}[t]
\begin{enumerate}
\item Input textual description $t$ and data $x$.
\item For $n=1,2,\cdots,N$:
\begin{enumerate}
\item Sample model $m^{(n)}\sim p\pp{m|t}$.\hfill{}\textcolor{gray}{//
using LLM}
\item Draw samples $z^{(n,1)},\cdots,z^{(n,K)}\sim p\pp{z\vert x,m^{(n)}}$
using MCMC.\hfill{}\textcolor{gray}{// under PPL}
\item Compute a lower bound $L^{(n)}$ on the marginal likelihood $\log p\pp{x\vert m^{(n)}}$
using \ref{eq:IW-ELBO}. 
\end{enumerate}
\item Set $w^{(n)}\propto\exp L^{(n)}$, where $\sum_{n=1}^{N}w^{(n)}=\frac{1}{K}$.
\item Return the samples $\{z^{(n,k)}\}$ where $z^{(n,k)}$ is given weight
$w^{(n)}.$
\end{enumerate}
\caption{The variant of \ref{alg:practical-LLB} used in the experiments of
this paper.\label{alg:LLB-using-MCMC-and-VI}}
\end{algorithm}

\section{Experiments\label{sec:Examples}}

\begin{figure}[b]
\begin{minipage}[t]{0.375\columnwidth}%
\vspace{-0.1cm}
\begin{lstlisting}[
    basicstyle=\small\ttfamily,     % Typewriter font (monospaced)
   language={},
   tabsize=1,
   frame=single,
   rulecolor=\color{gray!100},
   backgroundcolor=\color{gray!03},
    breaklines=true,
    xleftmargin=3pt,
    xrightmargin=3pt,
    breakindent=0pt,
    lineskip=-3pt,
]
PROBLEM

I've been recording if it rains each day, with a 1 for rain and 0 for no rain. Maybe there's some kind of pattern? Predict if it will rain the next day.

DATA

int num_days
array[num_days] int rain // results so far, in order

GOAL

int next // outcome for next day
\end{lstlisting}%
\end{minipage}\hfill{}%
\begin{minipage}[t]{0.615\columnwidth}%
\begin{lstlisting}[
    basicstyle=\small\ttfamily,
   frame=single,
   rulecolor=\color{gray!100},
   backgroundcolor=\color{gray!03},
    breaklines=true,
    xleftmargin=3pt,
    xrightmargin=3pt,
    breakindent=2\baselineskip,
]
{"num_days":22,
 "rain":[1,1,0,0,0,0,0,1,1,1,0,0,0,0,0,0,0,0,0,1,1,1]}
\end{lstlisting}

\includegraphics[scale=0.39,trim = {0.2cm 0 0.1cm 0}, clip]{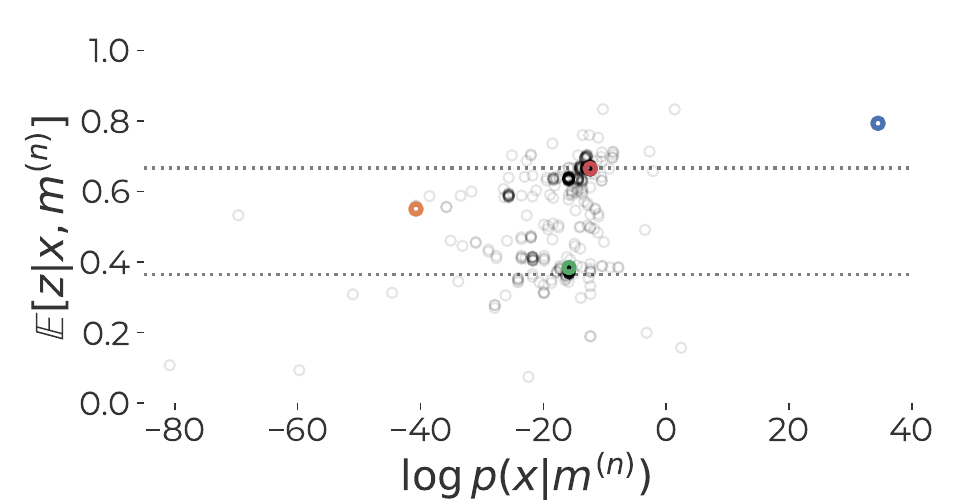}
\hfill{}\includegraphics[scale=0.39,trim = {2.4cm 0 0.1cm 0}, clip]{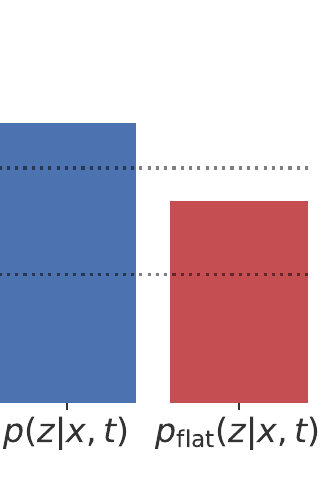}%
\end{minipage}

\vspace{-0.2cm}

\caption{The \textbf{rain} problem. Left: Informal user text $t$. Top right:
The given data $x$. Bottom center: Estimated marginal likelihoods
$p\protect\pp{x\vert m^{(n)}}$ and posterior means $\protect\E\protect\bb{\protect\r z\vert x,m^{(n)}}$
for each generated model $m^{(n)}$. Markers for the four models in
\ref{fig:rain-models} are colored. Bottom right: The final posterior
mean, compared to a flat average.\label{fig:rain-maintext}}
\end{figure}

It is likely that all standard models and datasets are included in
LLM training data. To avoid the risk that the LLM would simply ``remember''
human-written models, these experiments use all-new problems and datasets.

Since this paper will presumably also be included in future LLM datasets,
any future work in this direction should not use these problems for
evaluation with any LLM with a knowledge cutoff after the date this
paper was first made public, namely April 21, 2025.

\subsection{Rain}

In this problem \eqref{fig:rain-maintext}, a user recorded if it
rained on a sequence of days and wishes to predict if it will rain
on the following day. The data provides two contradictory signals.
One the one hand, most days did not have rain. However, adjacent days
are correlated and it rained on the last three days.

Out of 1024 generated models, 960 (93.8\%) compiled and allowed for
inference. Four example models are shown in \ref{fig:rain-models}
(supplement). Many models treat each day as independent and give predictions
near the base rate of $8/22\approx0.364$. Others give predictions
near the between-day consistency rate of $14/21\approx0.667$. (Both
of these base rates are shown as dotted lines in \ref{fig:rain-maintext}.)
A flat average includes many models of both types and predicts $0.573$.
However, the final posterior is dominated by a single model that models
multi-day dependencies and predicts an even-higher value of $0.793$.
Detailed results are in \ref{subsec:Rain-full}.

\subsection{Coin}

\begin{figure}[t]
\noindent\begin{minipage}[t]{1\columnwidth}%
\begin{tabular}{ccc}
\begin{minipage}[t]{0.25cm}%
{\small\color{col1}$t^{(1)}$}%
\end{minipage} & {\scriptsize{\ttfamily{}%
\begin{minipage}[t][0.9cm]{0.42\columnwidth}%
\begin{flushleft}
{\footnotesize\texttt{{\color{col1}I just got the coin from the US
mint, so I'm almost completely sure that it's a standard US penny.}}}
\par\end{flushleft}%
\end{minipage}}} & \multirow{3}{*}{\includegraphics[viewport=0bp 0bp 400bp 230bp,clip,width=0.46\columnwidth]{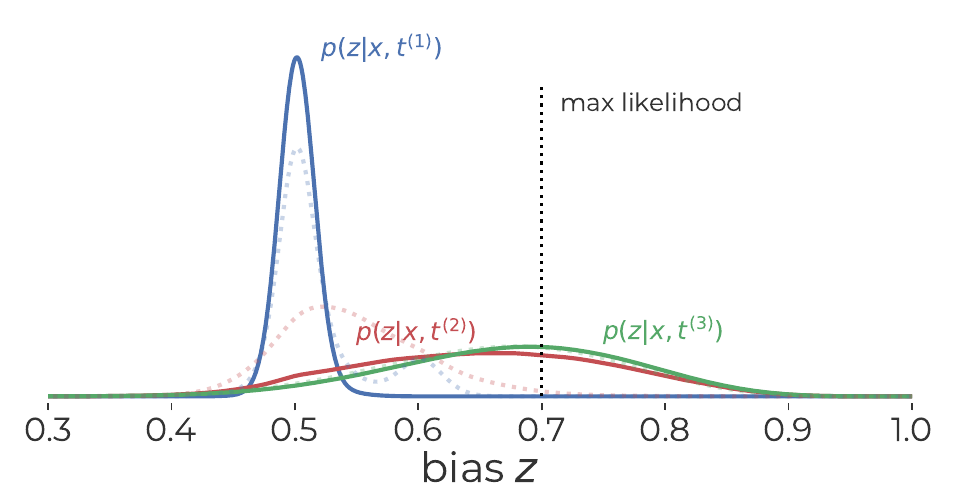}}\tabularnewline
\begin{minipage}[t]{0.25cm}%
{\small\color{col2}$t^{(2)}$}%
\end{minipage} & {\scriptsize{\ttfamily{}%
\begin{minipage}[t][0.9cm]{0.42\columnwidth}%
\begin{spacing}{0.6}
\begin{flushleft}
{\footnotesize\texttt{{\color{col2}At first glance, it appears to
be a standard US penny, although I haven't examined it closely.}}}
\par\end{flushleft}
\end{spacing}
\end{minipage}}} & \tabularnewline
\begin{minipage}[t]{0.25cm}%
{\small\color{col3}$t^{(3)}$}%
\end{minipage} & {\scriptsize{\ttfamily{}%
\begin{minipage}[t][1.2cm]{0.42\columnwidth}%
\begin{spacing}{0.6}
\begin{flushleft}
{\footnotesize\texttt{{\color{col3}I can see that the penny is quite
bent, so I'm pretty sure the bias is different from 0.5, though I
can't tell in what direction.}}}
\par\end{flushleft}
\end{spacing}
\end{minipage}}} & \tabularnewline
\end{tabular}\vspace{0.35cm}%
\end{minipage}

\caption{The \textbf{coin} problem. Left: Snippets from three different user
prompts. Right: Resulting final posteriors, which appear to reflect
user intent. Predictions from $p_{\mathrm{flat}}$ are shown as faint
dotted lines. \label{fig:coins-maintext}}
\end{figure}

This problem \eqref{fig:coins-maintext} tests the impact of different
assumptions stated in the user text. Three prompts are given describing
flipping a coin but with different assumptions about the true bias.
In all cases, 20 flips were observed, out of which 14 were heads.

Out of 1024 generated models, between 848 (82.8\%) and 989 (96.6\%)
compiled and allowed for inference, depending on the prompt variant.
For all three variants, many generated models had marginal likelihoods
high enough to contribute to the final estimated posterior. The final
estimated posteriors are compared in \ref{fig:coins-maintext}, where
the different assumptions seem reflected in the final posteriors.
Detailed results are in \ref{subsec:coin-standard-full}, \ref{subsec:coin-looks-full},
and \ref{subsec:coin-bent-full}.

\subsection{Polling}

\begin{figure}[b]
\begin{centering}
\includegraphics[width=0.45\columnwidth,trim = {0 0.92cm 0 0}, clip]{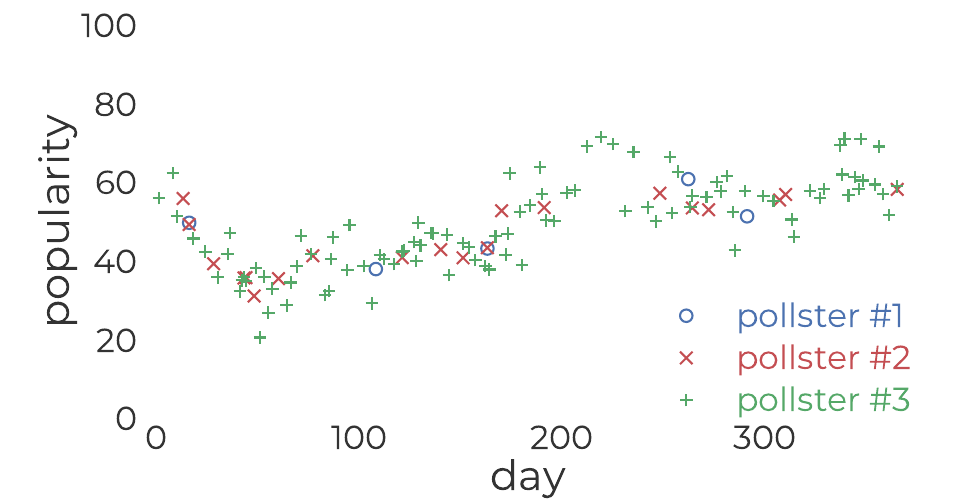}\includegraphics[width=0.45\columnwidth,trim = {0 0 0 0}, clip]{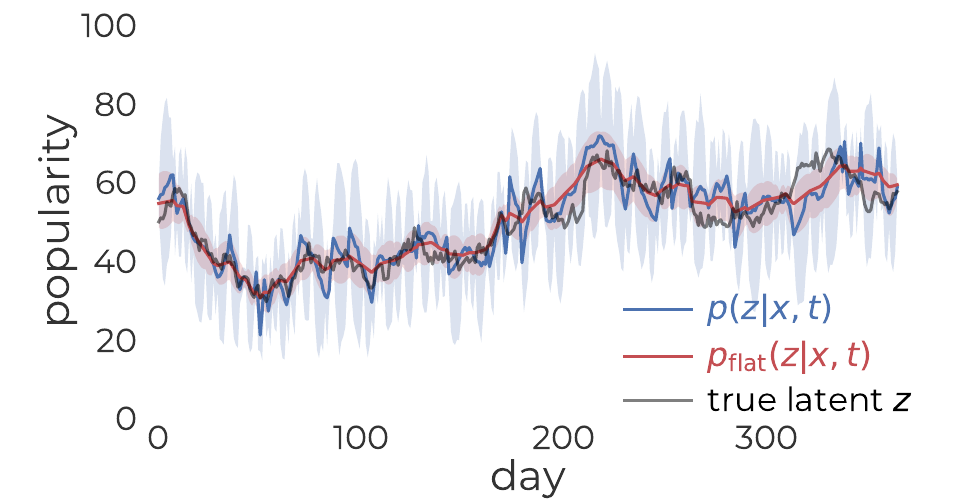}
\par\end{centering}
\caption{The \textbf{polling} problem. Left: Observed data $x$. Right: Final
estimated posterior, compared to a flat average.\label{fig:polling-maintext}}
\end{figure}

In this problem \eqref{fig:polling-maintext}, a user describes a
candidate with fluctuating popularity who is polled by three different
pollsters at various times throughout a year, and asks to predict
the true popularity on each day. See \ref{subsec:polling-full} for
the full prompt.

Out of 1024 generated models, 568 (55.5\%) compiled and allowed for
inference. For this problem, the estimated posterior shows little
benefit over a flat average. A single model captured essentially all
posterior weight, which may indicate inaccurate inference. Detailed
results are in \ref{subsec:polling-full}.

\subsection{City Temperature}

\begin{wrapfigure}[12]{o}{0.55\columnwidth}%
\begin{centering}
\vspace{-1.3cm}\includegraphics[viewport=20bp 0bp 620bp 350bp,clip,width=0.6\columnwidth]{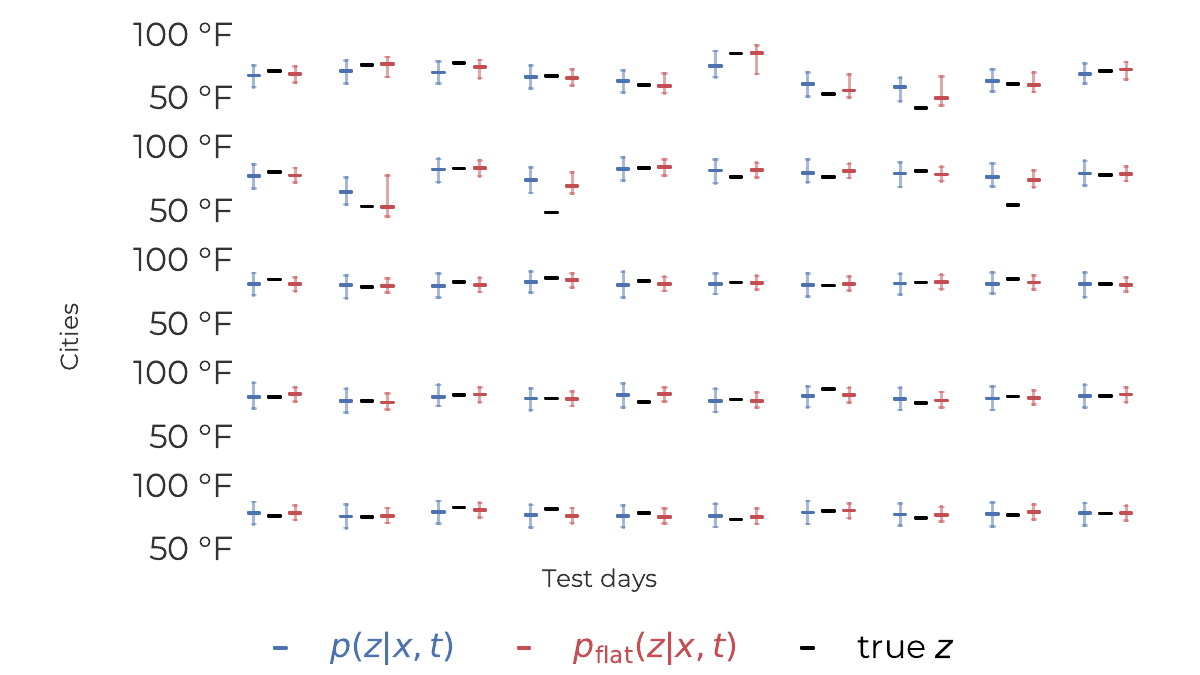}
\par\end{centering}
\caption{Medians and 90\% credible intervals for each city and test day for
the the \textbf{city temperature} problem.\label{fig:city-temp-maintext}}
\end{wrapfigure}%
In this problem \eqref{fig:city-temp-maintext}, the temperature is
given for a set of random days in a handful of random cities around
the world, along with the temperature on subsequent days. The user
gives a set of test days and asks to predict the temperature on the
following days.

Out of 1024 generated models, 736 (71.9\%) compiled and allowed for
inference. In this problem, again the final estimated posterior isn't
obviously better than a flat average. Detailed results are in \ref{subsec:City-temperature}.

\subsection{Gold}

In this problem \eqref{fig:gold-maintext}, the user describes a rod
with a varying density of gold. Binary tests have been done at different
positions and the goal is to infer the true density. Inference is
done with two different datasets, one with 30 observations, and one
with 150.

This problem proved quite challenging, with only 1-2\% of generated
models being syntactically valid and allowing inference to proceed.
Thus, 16,384 models were generated, out of which only 192 (1.2\%)
compiled and allowed for inference with 30 data and 71 (0.4\%) with
150 data. The marginal likelihoods result in almost all weight concentrating
in just three models in the smaller dataset and just one model with
the larger dataset. Using these gives sensible results, and clearly
better than a ``flat'' average, at least with 150 data . Detailed
results are in \ref{subsec:Gold-(small)} and \ref{subsec:Gold-(large)}.

\begin{figure}[b]
\begin{centering}
\includegraphics[viewport=0bp 0bp 450bp 250bp,clip,scale=0.38]{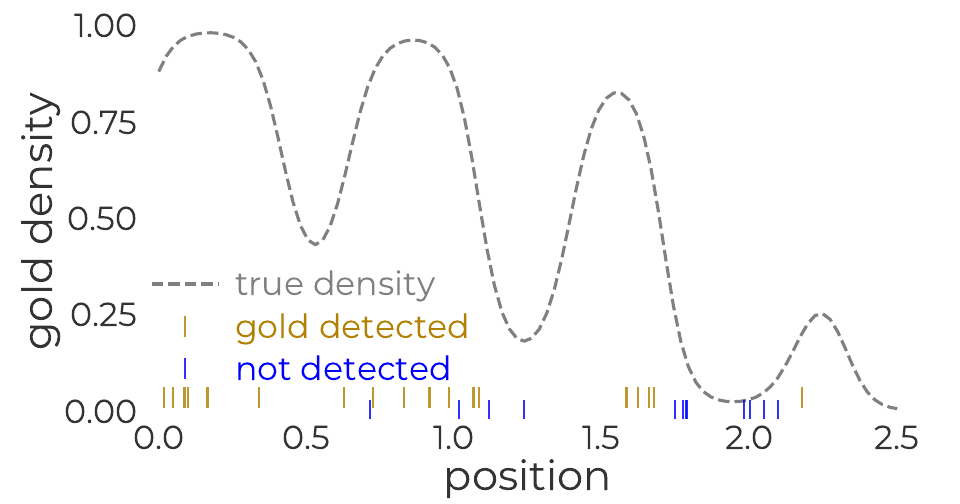}\includegraphics[viewport=64bp 0bp 450bp 250bp,clip,scale=0.38]{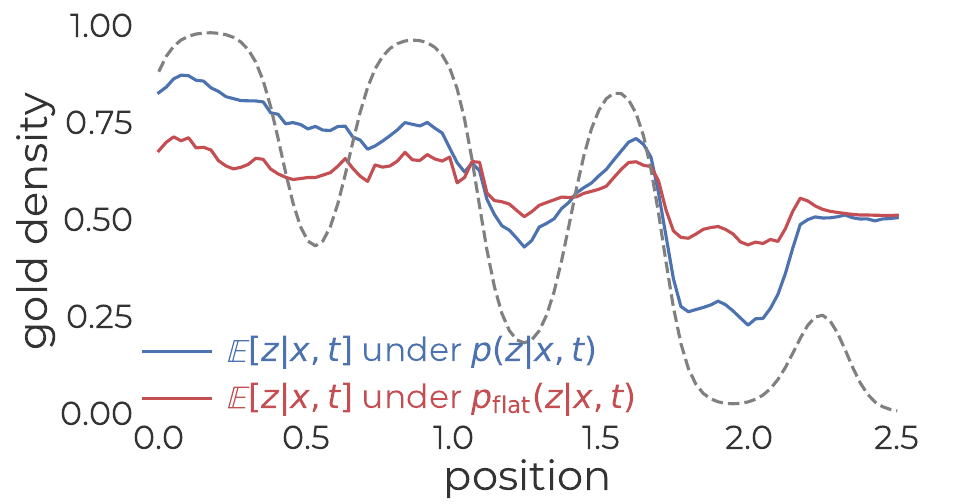}
\par\end{centering}
\begin{centering}
\includegraphics[viewport=0bp 0bp 450bp 250bp,clip,scale=0.38]{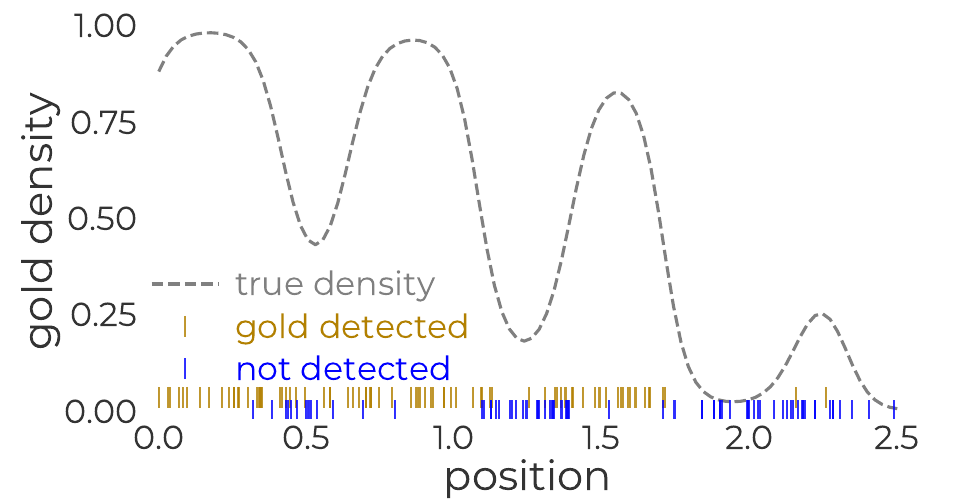}\includegraphics[viewport=64bp 0bp 450bp 250bp,clip,scale=0.38]{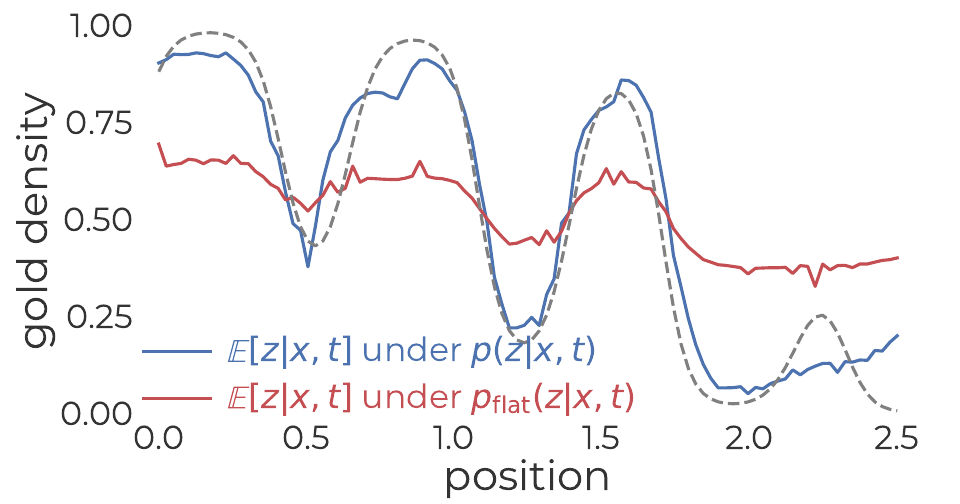}
\par\end{centering}
\caption{The \textbf{Gold} problem. Left column: Observed data. Right column:
Final estimated posterior, compared to a flat average. Top row: Results
with 30 observations. Bottom row: 150 observations.\label{fig:gold-maintext}}
\end{figure}

Finally, note that some of the above results may be impacted by the
fact that if truncated distributions are created through the use of
constraints in Stan \citep{Carpenter_2017_StanProbabilisticProgramming},
densities are not re-normalized to reflect the truncation. So, for
example, if a model includes a standard half-normal, computed log-densities
for that model will be lower they should be by a factor of $\log\frac{1}{2}\approx0.693$
nats. Unfortunately, this seems difficult to address automatically.
This is ``safe'' in the sense that it can decrease but not increase
estimated marginal likelihoods, but it means that some ``good''
models may have lower weights then they deserve in all the above experiments.

\section{Theory\label{sec:Theory}}

This section seeks to understand how accurately \ref{alg:practical-LLB}
will approximate the true posterior, in particular how the accuracy
depends on the number of samples $N$, and the quality of the posterior
and marginal likelihood approximations, and the LLM-defined prior
over models $p\pp{m\vert t}.$

\subsection{Importance sampling in model space\label{subsec:Importance-sampling-in-model-space}}

Recall the final posterior $p\pp{z\vert x,t}=\sum_{m}p\pp{m\vert x,t}p\pp{z\vert x,m}$
from \ref{eq:p(z|x_t)}. To cope with the large model space, one idea
would be to perform importance sampling---to draw $m\sim p\pp{m\vert t}$
and reweight $p\pp{z\vert x,m}$ by the ratio $p\pp{m\vert x,t}/p\pp{m\vert t}$.
Unfortunately, this is infeasible, since the normalizing constant
for $p\pp{m\vert x,t}$ in \ref{eq:p(m|x_t)} is intractable and the
probabilities $p\pp{m\vert t}$ are unavailable \eqref{sec:Computational-Issues}.

An alternative is \textit{self-normalized} importance sampling (SNIS)---to
draw a set of models from $p\pp{m\vert t}$ and give each a weight
proportional to $p\pp{m\vert x,t}/p\pp{m\vert t},$ but normalized
to sum to one. It can be shown \eqref{sec:Form-of-SNIS-weights} that
this estimator simplifies into
\begin{equation}
p\pp{z|x,t}\approx\sum_{n=1}^{N}w^{(n)}\ p\pp{z\vert x,m^{(n)}},\quad w^{(n)}=\frac{p\pp{x\vert m^{(n)}}}{\sum_{n'=1}^{N}p\pp{x\vert m^{(n')}}},\label{eq:SNIS-with-exact-inference}
\end{equation}
where $m^{(1)},\cdots,m^{(N)}\sim p\pp{m\vert t}.$ This is equivalent
to \ref{alg:self-normalized-LLB} \eqref{sec:Other-example-algorithms}.

This will rarely be practical, since the posteriors $p\pp{z\vert x,m^{(n)}}$
and marginal likelihoods $p\pp{x\vert m^{(n)}}$ must be approximated.
But it is useful to analyze it to isolate error due purely to SNIS.
Suppose we are interested in the expectation of some function $f$
with respect to the posterior, i.e. $\mu=\E_{p\pp{z\vert x,t}}f\pp z.$
If $\tilde{p}\pp{z\vert x,t}$ is the approximation in \eqref{eq:SNIS-with-exact-inference},
one can estimate $\mu$ with\vspace{-10pt}

\begin{equation}
\hat{\mu}_{N}=\E_{\tilde{p}\pp{z\vert x,t}}f\pp z=\sum_{n=1}^{N}w^{(n)}\E_{p\pp{\r z\vert x,m^{(n)}}}\bracs{p\pp{\r z\vert x,m^{(n)}}}.\label{eq:IS-estimator}
\end{equation}
\vspace{-7pt}

Standard results \eqref{sec:Variance-of-the-IS-estimator} show that
$\hat{\mu}_{N}$ is asymptotically unbiased with asymptotic variance
\[
\V\bb{\hat{\mu}_{N}}\approx V_{N}:=\frac{1}{N}\E_{p\pp{\r m\vert t}}\bracs{\pars{\frac{p\pp{\r m\vert x,t}}{p\pp{\r m\vert t}}}^{2}\pars{g\pp{\r m}-\mu}^{2}}.
\]
Further, if $\verts{g\pp m-\mu}\leq\delta$ (as would be true, for
example, if $\verts{f\pp z-\mu}\leq\delta$), then,
\begin{equation}
V_{N}\leq\frac{\delta^{2}}{N}\Bigl(1+\chi^{2}\pars{p\pp{\r m\vert x,t}\Vert p\pp{\r m\vert t}}\Bigr),\label{eq:SNIS_var_bound}
\end{equation}
where $\chi^{2}\pp{p\Vert q}$ again denotes the chi-squared divergence.
\ref{eq:SNIS_var_bound} suggests that for \ref{alg:self-normalized-LLB}
to achieve a given accuracy, $N$ must be proportional to $\chi^{2}\pars{p\pp{\r m\vert x,t}\Vert p\pp{\r m\vert t}}$.
This is an ``inclusive'' divergence, meaning that if $p\pp{m\vert t}$
is small but $p\pp{m\vert x,t}$ is large for some model $m$, this
greatly increases the divergence, while the reverse situation only
causes a modest increase.

This suggests that it is important that the LLM-defined prior $p\pp{m\vert t}$
be relatively broad: The accuracy of the estimated posterior is harmed
more by the prior being ``overconfident'' than being ``underconfident''.
In particular, since $p\pp{m\vert x,t}$ only varies from $p\pp{m\vert t}$
by a factor of $p\pp{x\vert m}$ (and normalization), it is important
that $p\pp{m\vert t}$ should not give vanishingly small probability
to models $m$ where the marginal likelihood $p\pp{x\vert m}$ is
large.

\subsection{Approximate inference}

This section considers the impact of approximating the marginal likelihood
and the posterior for a given model. Suppose temporarily that for
each $m$, some approximation $q\pp{z\vert x,m}\approx p\pp{z\vert x,m}$
is known. Define the joint approximation
\begin{equation}
q\pp{z,m\vert x}=q\pp{m\vert x}q\pp{z\vert x,m},\label{eq:joint-approximation}
\end{equation}
where $q\pp{m\vert x}$ remains to be chosen. In principle, we would
like to optimize $q\pp{m\vert x}$ to make the marginal $q\pp{z\vert x}$
close to $p\pp{z\vert x,t}$, but this seems difficult, since even
evaluating $p\pp{z\vert x,t}$ is intractable \eqref{eq:p(z|x_t)}.
Instead we propose to choose $q\pp{m\vert x}$ to minimize the joint
divergence
\begin{equation}
\KL{q\pp{\r z,\r m\vert x}}{p\pp{\r z,\r m\vert x,t}},\label{eq:joint-divergence}
\end{equation}
 which (by the chain rule of KL-divergence \citep[Thm 2.5.3]{Cover_2006_ElementsInformationTheory})
is an upper bound on the divergence from $q\pp{z\vert x}$ to $p\pp{z\vert x,t}.$
The following result gives the the optimal $q\pp{m\vert x}$. (Proof
in \ref{subsec:proof-of-variational-result})
\begin{thm}
Suppose $p\pp{z,x,m\vert t}$ and $q\pp{z\vert x,m}$ are fixed. Then
$\KL{q\pp{\r z,\r m\vert x}}{p\pp{\r z,\r m\vert x,t}}$ is minimized
by\label{thm:variational-lemma}
\begin{align}
q\pp{m\vert x} & \propto p\pp{m\vert t}p\pp{x\vert m}\times\exp\Bigl(-\KL{q\pp{\r z\vert x,m}}{p\pp{\r z\vert x,m}}\Bigr),\label{eq:optimal-q}
\end{align}
with a resulting joint divergence of
\begin{align}
\KL{q\pp{\r z,\r m\vert x}}{p\pp{\r z,\r m\vert x,t}} & =-\log\E_{p\pp{\r m\vert x,t}}\exp\Bigl(-\KL{q\pp{\r z\vert x,\r m}}{p\pp{\r z\vert x,\r m}}\Bigr).\label{eq:optimal-q-joint-divergence}
\end{align}
\end{thm}

Informally, the optimal $q\pp{m\vert x}$ in \ref{eq:optimal-q} can
be seen as taking $p\pp{m\vert x,t}\propto p\pp{m\vert t}p\pp{x\vert m}$
from \ref{eq:p(m|x_t)} and down-weighting models $m$ where $q\pp{z\vert x,m}$
is a worse approximation of $p\pp{z\vert x,m}$. \ref{eq:optimal-q-joint-divergence}
is less intuitive. As KL-divergence is non-negative, the inner expectation
is between 0 and 1, with larger values giving a smaller joint divergence.
Thus, inaccuracy in $q\pp{z\vert x,m}$ ``hurts'' more when $p\pp{m\vert x,t}$
is larger. This same intuition follows from noting that \ref{eq:optimal-q-joint-divergence}
can be upper-bounded by $\E_{p\pp{\r m\vert x,t}}\KL{q\pp{\r z\vert x,\r m}}{p\pp{\r z\vert x,\r m}}$
(See \ref{cor:relaxed-bound} in \ref{sec:Relaxed-bound}).

\ref{thm:variational-lemma} is not immediately useful, since the
marginal likelihood $p\pp{x\vert m}$ and KL divergence for a given
$m$ are typically intractable. However, recall the ELBO decomposition
\begin{equation}
\log p\pp{x\vert m}=\mathrm{ELBO}(m)+\KL{q\pp{\r z\vert x,m}}{p\pp{\r z\vert x,m}},\label{eq:ELBO-decomp-1}
\end{equation}
where the (tractable) evidence lower-bound (ELBO) (now written as
a function of the model $m$ rather than the distribution $q$ as
in \ref{eq:ELBO-q}) is 
\begin{equation}
\mathrm{ELBO}(m):=\E_{q\pp{\r z\vert x,m}}\log\frac{p\pp{\r z,x\vert m}}{q\pp{\r z\vert x,m}}.\label{eq:ELBO}
\end{equation}
Variational inference (VI) algorithms maximize the ELBO which is equivalent
(since $\log p\pp{x\vert m}$ is fixed) to minimizing the KL-divergence.
Using this decomposition, \ref{thm:variational-lemma} can be given
in the following alternate form. (Proof in \ref{subsec:Proof-of-ELBO-form}.)
\begin{cor}
Under the same assumptions as \ref{thm:variational-lemma}, the optimal
$q\pp{m\vert x}$ can also be written as\label{cor:variational-q(m)}
\begin{equation}
q\pp{m\vert x}\propto p\pp{m\vert t}\exp\Bigl(\mathrm{ELBO}(m)\Bigr),\label{eq:optimal-q-joint-divergence-ELBO-form}
\end{equation}
with a resulting joint divergence of
\begin{align}
\KL{q\pp{\r z,\r m\vert x}}{p\pp{\r z,\r m\vert x,t}} & =\log p\pp{x\vert t}-\log\E_{p\pp{\r m\vert t}}\exp\pars{\mathrm{ELBO}(\r m)}.\label{eq:optimal-joint-KL-ELBO-form}
\end{align}
\end{cor}

Note that \ref{eq:optimal-q-joint-divergence-ELBO-form} was previously
shown by \citep[Sec. 2.2]{Kejzlar_2023_BlackBoxVariational} and \citet[Thm 2.1]{Ohn_2024_AdaptiveVariationalBayes}.

Taken literally, \ref{cor:variational-q(m)} suggests that to find
the joint approximation $q\pp{z,m\vert x}$ closest to $p\pp{z,m\vert x,t},$
one should loop over all models $m$, independently do variational
inference on each, and then average the variational distributions
using the weights from \ref{eq:optimal-q-joint-divergence-ELBO-form}.
Such an algorithm is shown as \ref{alg:Variational-LLB} \eqref{sec:Other-example-algorithms},
but is not practical since again the space of models $m$ is large
and the LLM probabilities $p\pp{m\vert t}$ are unavailable.

However, one can also incorporate the SNIS ideas from \ref{subsec:Importance-sampling-in-model-space}.
Suppose we would like to estimate $q\pp{z\vert x}=\E_{q\pp{\r m\vert x}}q\pp{z\vert x,\r m},$
where $q\pp{z\vert x,m}$ is an approximation of $p\pp{z\vert x,m}$
and $q\pp{m\vert x}$ are the weights from \ref{eq:optimal-q-joint-divergence-ELBO-form}.
If we use a proposal distribution $p\pp{m\vert t},$ then it's easy
to show \eqref{sec:Form-of-SNIS-weights-with-VI} that the self-normalized
weights are proportional to $\exp(\mathrm{ELBO}(m))$. This results
in \ref{alg:Variational-LLB-IS}. This is an instance of the recipe
in \ref{alg:practical-LLB}, since $\mathrm{ELBO}(m)$ is a lower-bound
on $\log p\pp{x\vert m}$. 

\subsection{Algorithmic variants}

It can be beneficial to use approximating distributions $q\pp{z\vert x,m}$
where only a \textit{lower-bound} on the ELBO is available. One example
of this would be to take advantage of Monte Carlo VI methods like
importance-weighted VI \citep{Burda_2015_ImportanceWeightedAutoencoders,Domke_2019_DivideCoupleUsing,Bachman_2015_TrainingDeepGenerative,Cremer_2017_ReinterpretingImportanceWeightedAutoencoders,ChristianAnderssonNaesseth_2018_MachineLearningUsing}.
Another example would be when one believes MCMC can efficiently sample
from $p\pp{z\vert x,m}$. If one thinks of those samples as representing
an approximating distribution $q\pp{z\vert x,m}$, one might believe
it is essentially exact. But even if this is true, estimating the
marginal likelihood from samples is famously difficult \citep{Newton_1994_ApproximateBayesianInference,Neal_2008_HarmonicMeanLikelihood},
so one might bound the marginal likelihood using a different method
(e.g. variational inference).

To justify this, imagine that one first finds a variational approximation
$q\pp{z\vert x,m}$ for each model $m$, and sets the model weights
$q\pp{m\vert x}$ using \ref{eq:optimal-q-joint-divergence-ELBO-form}
Then, imagine \textit{replacing} each distribution $q\pp{z\vert x,m}$
with one with a higher ELBO (lower KL-divergence) while leaving $q\pp{m\vert x}$
fixed. By the chain rule of KL-divergence, $\KL{q\pp{\r z,\r m\vert x}}{p\pp{\r z,\r m\vert x,t}}=\KL{q\pp{\r m\vert x}}{p\pp{\r m\vert x,t}}+\KL{q\pp{\r z\vert\r m,x}}{p\pp{\r z\vert\r m,x}}.$
Improving $q\pp{z\vert m,x}$ will decrease the second KL-divergence
on the right while leaving the first divergence unchanged, meaning
the joint divergence can only decrease.

\ref{sec:Analysis-with-inexact-ELBO} gives a generalized version
of the above results when the model weights $q\pp{m\vert t}$ are
computed based on inexact ELBO values. A particularly interesting
special case is when the distributions $p\pp{z\vert x,m}$ can be
computed exactly (e.g. using MCMC) but the weights $q\pp{z\vert m}$
are computed based on some bound on the marginal likelihood. (See
\ref{alg:LLB-using-MCMC-and-VI} in \ref{sec:Other-example-algorithms}.)
Then the joint divergence result from \ref{sec:Analysis-with-inexact-ELBO}
reduces to $\KL{q\pp{\r z,\r m\vert x}}{p\pp{\r z,\r m\vert x,t}}=-\log\E_{p\pp{\r m\vert x,t}}\exp(-\delta^{(\r m)}+\bar{\delta}),$
where $\delta^{(m)}$ is the gap between $\log p\pp{x\vert m}$ and
the lower bound used when computing $q\pp{m\vert x}$ and $\bar{\delta}=\E_{q\pp{\r m\vert x}}\delta^{(\r m)}.$
This result has a somewhat similar intuition as \ref{eq:optimal-q-joint-divergence}:
Constant errors have no effect, and errors on models where $p\pp{m\vert x,t}$
very is small have little effect. What really matters is if $\delta^{(m)}$
\textit{varies} among models where $p\pp{m\vert x,t}$ is large.

\section{Discussion and limitations}

This paper proposed the LLB scheme for defining a posterior from natural
language \eqref{sec:The-basic-idea}, suggested a broad approximate
inference strategy for it \eqref{sec:Computational-Issues}, tested
it experimentally \eqref{sec:Examples} and analyzed the error of
the inference strategy theoretically \eqref{sec:Theory}.

One direction for future work would be to investigate better ways
of using LLMs to define distributions over models---essentially better
distributions $p\pp{m\vert t}.$ There are many obvious options, such
as different system prompts, more/better examples, different formatting
of user inputs, or training (or fine-tuning) an LLM specifically for
this task.

Another direction is better inference methods. The suggested recipe
involves doing inference independently on each sampled model. While
fully parallelizable, this is expensive, and practical only on relatively
small problems. We intend no claim of optimality. When doing inference
on many models for the same task, many models are often quite similar,
suggesting it might be possible to share work between models. It might
also be possible to use constrained generation to guarantee that only
valid models were created, eliminating the overhead of generating
and then rejecting some syntactically invalid models.

\subsection{Related work}

In terms of usage of LLMs for Bayesian modeling, \citet{Wong_2023_WordModelsWorld}
suggest a system where an LLM translates natural language text into
a PPL. This is related to our approach, but assumes that ``facts''
as provided to the LLM as text, rather than being \textit{described}
in text and then provided to inference as numerical data. \citet{Capstick_2024_UsingLargeLanguage}
consider using an LLM to take a natural language description and predict
the parameters of Gaussian priors (to be used in concert with known
likelihood models). \citet{Selby_2024_HadEnoughExperts} also consider
a similar idea, and compare the results from several LLMs to priors
elicited from human experts \citet{Stefan_2022_ExpertAgreementPrior}.
\citet{Requeima_2024_LLMProcessesNumerical} suggest a scheme to make
probabilistic prediction from natural language plus data. Other work
considers the use of LLMs for time-series predictions \citep{Gruver_2023_LargeLanguageModels,Xue_2024_PromptCastNewPromptBased}
or as regressors \citep{Vacareanu_2024_WordsNumbersYour}. \citet{Choi_2022_LMPriorsPreTrainedLanguage}
suggest a general framework for using natural language information
to inform learning procedures via an LLM. 

\section*{Acknowledgement}

This material is based upon work supported in part by the National
Science Foundation under Grant No. 2045900.

\bibliographystyle{icml2025}
\bibliography{justindomke_bibtex3}

\cleardoublepage{}

\appendix

\onecolumn 

\section{Acronyms and notation}

\textbf{Acronyms}
\begin{itemize}
\item LLM: Large Language Model
\item PPL: Probabilistic Programming Language
\item BMA: Bayesian Model Averaging
\item SNIS: Self-normalized importance sampling
\item VI: Variational inference
\end{itemize}
\textbf{Notation}
\begin{itemize}
\item $t$: Plain-language description of an inference problem.
\item $m$: Formal model definition (in a PPL)
\item $x$: Input data
\item $z$: Query variables / latent variables
\item $n$: Index for sampled model
\item $m^{(n)}:$ $n$th sampled model
\item $w^{(n)}$: weight given to $n$th sampled model
\end{itemize}
\textbf{Terminology}
\begin{itemize}
\item Single-model posterior: $p\pp{z\vert x,m}$
\item Posterior model weight: $p\pp{m\vert x,t}=\frac{p\pp{m\vert t}p\pp{x\vert m}}{\sum_{m'}p\pp{m'\vert t}p\pp{x\vert m'}}=\frac{p\pp{m\vert t}p\pp{x\vert m}}{p\pp{x\vert t}}$
\item Final posterior: $p\pp{z\vert x,t}$
\end{itemize}
\cleardoublepage{}

\section{Experimental details \label{sec:Additional-experimental-results}}

This section gives detailed results for the experiments. For each
problem, we describe the prompt and data, show example generated models,
attempt to visualize the posterior of each model, show marginal likelihoods
and weights, and the final estimated posterior, compared against a
``flat'' average as in \ref{eq:flat-posterior}.

\subsection{Rain\label{subsec:Rain-full}}

In this problem, the user describes recording if it rained on not
on a series of days, and wishes to predict if it will rain on the
next day. Meanwhile, the given data $x$ provides two contradictory
signals: On the one hand, most days did not have rain. However, adjacent
days are correlated and it rained on the final day.

\begin{minipage}[t]{0.48\textwidth}%
\begin{center}
{\scriptsize{\ttfamily{}%
\begin{minipage}[t]{0.97\columnwidth}%
\begin{lstlisting}[
    basicstyle=\scriptsize\ttfamily,     % Typewriter font (monospaced)
   tabsize=1,
   frame=single,
   rulecolor=\color{gray!100},
   backgroundcolor=\color{gray!03},
    breaklines=true,
    xleftmargin=0pt,
    xrightmargin=0pt,
    breakindent=0pt,
]
PROBLEM

I've been recording if it rains each day, with a 1 for rain and 0 for no rain. Maybe there's some kind of pattern? Predict if it will rain the next day.

DATA

int num_days
array[num_days] int rain // results so far, in order

GOAL

int next // outcome for next day
\end{lstlisting}%
\end{minipage}}}
\par\end{center}%
\end{minipage}\hfill{}%
\begin{minipage}[t]{0.5\textwidth}%
\begin{center}
{\scriptsize{\ttfamily{}%
\begin{minipage}[t]{0.97\columnwidth}%
\begin{lstlisting}[
    basicstyle=\scriptsize\ttfamily,
   frame=single,
   rulecolor=\color{gray!100},
   backgroundcolor=\color{gray!03},
    breaklines=true,
    xleftmargin=0pt,
    xrightmargin=0pt,
    breakindent=2\baselineskip,
]
{"num_days":22,
 "rain":[1,1,0,0,0,0,0,1,1,1,0,0,0,0,0,0,0,0,0,1,1,1]}
\end{lstlisting}%
\end{minipage}}}
\par\end{center}%
\end{minipage}

Out of 1024 models generated, 742 (72.5\%) compiled and allowed for
inference. \ref{fig:rain-models} shows four example models. \ref{fig:rain-inference}
shows the estimated marginal likelihoods and posteriors for each of
the valid models. Finally, \ref{fig:rain-bma} shows the estimated
marginal likelihoods and corresponding weights for the top-scoring
50 valid models, as well as the final estimated posterior $p\pp{z\vert x,t}$
compared to a ``flat'' average $p_{\mathrm{flat}}\pp{z\vert x,t}$.

In this case, there are two common ``clusters'' of models. One consists
of models similar to {\color{col3}$m^{(247)}$} \eqref{fig:rain-model-nocorr}
that treat each day as independent and give predictions near the base
rate of $\frac{8}{22}\approx.364$. A second consists of models similar
to {\color{col2}$m^{(463)}$} \eqref{fig:rain-model-corr} that
model the probability that rain is followed by no-rain or vice-versa
and give predictions near the base rate of $\frac{14}{21}\approx0.667$
for between-day consistency. Models from the first cluster tend to
have lower-marginal likelihoods than models from the second cluster.
However, the top-scoring model {\color{col1}$m^{(742)}$} \eqref{fig:rain-top-model}
has a still-higher marginal likelihood, and so receives essentially
all weight in the estimated final posterior $p\pp{z\vert x,t}.$

\begin{figure}
\begin{minipage}[t]{0.48\textwidth}%
\begin{center}
{\scriptsize\subfloat[{\color{col1}$m^{(742)}$}, the top-scoring model.\label{fig:rain-top-model}]{\begin{centering}
{\scriptsize{}%
\begin{minipage}[t]{0.97\columnwidth}%
\begin{lstlisting}[
    basicstyle=\scriptsize\ttfamily\color{col1},    
   tabsize=2,
   frame=single,
   rulecolor=\color{gray!100},
   backgroundcolor=\color{gray!03},
    breaklines=true,
    xleftmargin=0pt,
    xrightmargin=0pt,
    breakindent=2\baselineskip,
]
data {
	int num_days;
	array[num_days] int<lower=0, upper=1> rain;
}
parameters {
	real<lower=-10, upper=10> logit_initial;
	real autoregressive_coefficient;
	array[num_days-1] real noise;
}
transformed parameters{
	array[num_days] real logit_prob;
	logit_prob[1] = logit_initial;
	for(i in 2:num_days){
		logit_prob[i] = logit_prob[i-1] * autoregressive_coefficient + noise[i-1];
	}
}
model {
	logit_initial ~ uniform(-10,10);
	autoregressive_coefficient ~ normal(0,1);
	noise ~ cauchy(0,10);
	for (i in 1:num_days){
		rain[i] ~ bernoulli_logit(logit_prob[i]);
	}
}
generated quantities {
	real logit_next = logit_prob[num_days] * autoregressive_coefficient + cauchy_rng(0,10);
	int next = bernoulli_logit_rng(logit_next);
}
\end{lstlisting}%
\end{minipage}}{\scriptsize\par}
\par\end{centering}
{\scriptsize}{\scriptsize\par}}}
\par\end{center}%
\end{minipage}\hfill{}%
\begin{minipage}[t]{0.5\textwidth}%
\begin{center}
{\scriptsize\subfloat[{\color{col2}$m^{(463)}$}, typical of one \textquotedblleft cluster\textquotedblright{}
of models.\label{fig:rain-model-corr}]{\begin{centering}
{\scriptsize{}%
\begin{minipage}[t]{0.97\columnwidth}%
\begin{lstlisting}[
    basicstyle=\scriptsize\ttfamily\color{col2},  
   tabsize=2,
   frame=single,
   rulecolor=\color{gray!100},
   backgroundcolor=\color{gray!03},
    breaklines=true,
    xleftmargin=0pt,
    xrightmargin=0pt,
    breakindent=2\baselineskip,
]
data {
	int num_days;
	array[num_days] int<lower=0, upper=1> rain;
}
parameters {
	real<lower=0, upper=1> p_rain_given_rain;
	real<lower=0, upper=1> p_rain_given_no_rain;
}
model {
	p_rain_given_rain ~ beta(1,1);
	p_rain_given_no_rain ~ beta(1,1);
	rain[1] ~ bernoulli(0.5); // uniform prior for first day
	for (i in 2:num_days){
		if (rain[i-1]==1){
			rain[i] ~ bernoulli(p_rain_given_rain);
		} else {
			rain[i] ~ bernoulli(p_rain_given_no_rain);
		}
	}
}
generated quantities {
	int next;
	if (rain[num_days]==1){
		next = bernoulli_rng(p_rain_given_rain);
	} else {
		next = bernoulli_rng(p_rain_given_no_rain);
	}
}
\end{lstlisting}%
\end{minipage}}{\scriptsize\par}
\par\end{centering}
{\scriptsize}{\scriptsize\par}}}
\par\end{center}%
\end{minipage}

\begin{minipage}[t]{0.48\textwidth}%
\begin{center}
{\scriptsize\subfloat[{\color{col3}$m^{(247)}$}, typical of a second cluster.\label{fig:rain-model-nocorr}]{\begin{centering}
{\scriptsize{}%
\begin{minipage}[t]{0.97\columnwidth}%
\begin{lstlisting}[
    basicstyle=\scriptsize\ttfamily\color{col3},  
   tabsize=2,
   frame=single,
   rulecolor=\color{gray!100},
   backgroundcolor=\color{gray!03},
    breaklines=true,
    xleftmargin=0pt,
    xrightmargin=0pt,
    breakindent=2\baselineskip,
]
data {
	int num_days;
	array[num_days] int<lower=0, upper=1> rain;
}
parameters {
	real<lower=0, upper=1> p;
}
model {
	p ~ uniform(0,1);
	for (i in 1:num_days){
		rain[i] ~ bernoulli(p);
	}
}
generated quantities {
	int next = bernoulli_rng(p);
}
\end{lstlisting}%
\end{minipage}}{\scriptsize\par}
\par\end{centering}
{\scriptsize}{\scriptsize\par}}}
\par\end{center}%
\end{minipage}\hfill{}%
\begin{minipage}[t]{0.5\textwidth}%
\begin{center}
{\scriptsize\subfloat[{\color{col4}$m^{(6)}$}, a low scoring model.]{\begin{centering}
{\scriptsize{}%
\begin{minipage}[t]{0.97\columnwidth}%
\begin{lstlisting}[
    basicstyle=\scriptsize\ttfamily\color{col4},  
   tabsize=2,
   frame=single,
   rulecolor=\color{gray!100},
   backgroundcolor=\color{gray!03},
    breaklines=true,
    xleftmargin=0pt,
    xrightmargin=0pt,
    breakindent=2\baselineskip,
]
data {
	int num_days;
	array[num_days] int<lower=0, upper=1> rain; // results so far
}
parameters {
	real<lower=0, upper=1> p01;
	real<lower=0, upper=1> p11;
}
model {
	p01 ~ beta(20, 20);
	p11 ~ beta(20, 20);
	rain[1] ~ bernoulli(0.5); // no information about the first day
	for (i in 2:num_days){
		rain[i] ~ bernoulli(rain[i-1] ? p11 : p01);
	}
}
generated quantities {
	int next;
	if(rain[num_days]){
		next = bernoulli_rng(p11);
	} else {
		next = bernoulli_rng(p01);
	}
}
\end{lstlisting}%
\end{minipage}}{\scriptsize\par}
\par\end{centering}
{\scriptsize}{\scriptsize\par}}}
\par\end{center}%
\end{minipage}

\caption{Example models for the \textbf{rain} problem, sorted by estimated
marginal likelihoods. \label{fig:rain-models}}
\end{figure}

\begin{figure}
\noindent\begin{minipage}[t]{1\columnwidth}%
\includegraphics[scale=0.57]{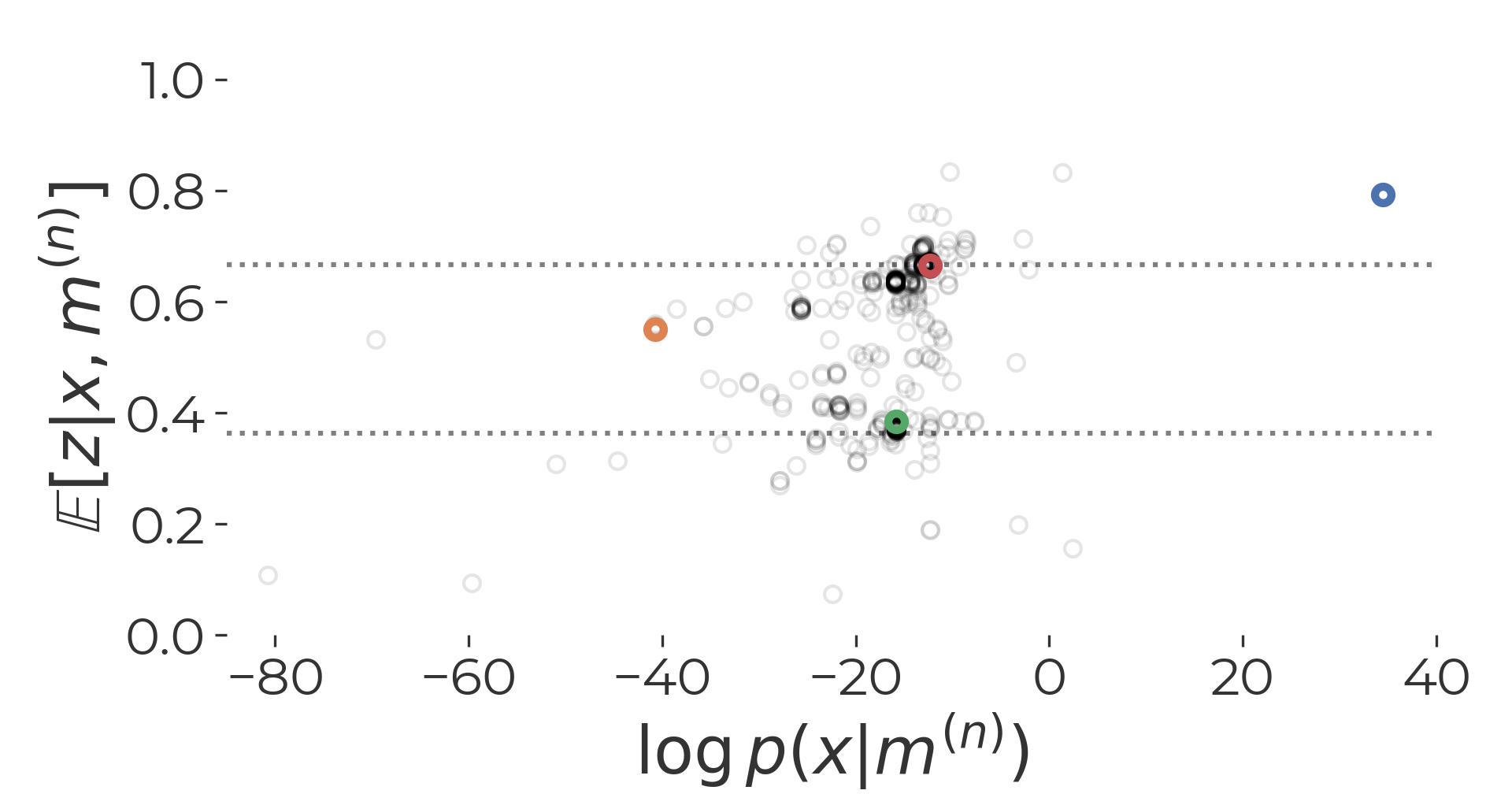}\includegraphics[viewport=2.3cm 0bp 800bp 420bp,clip,scale=0.57]{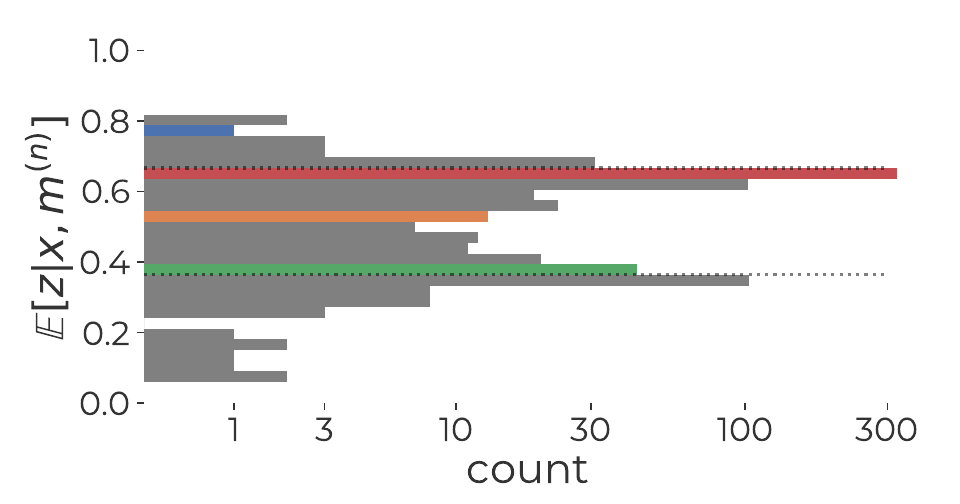}

\vspace{-0.89cm}

\includegraphics[scale=0.57]{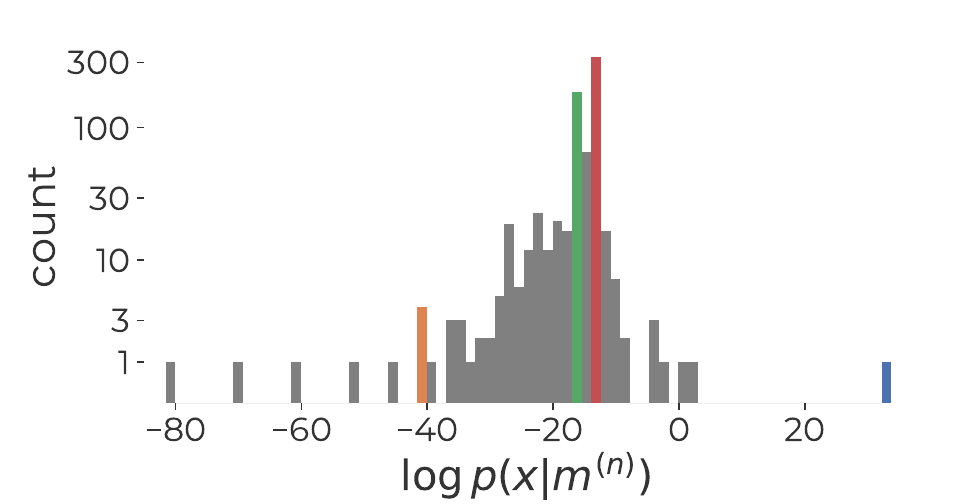}%
\end{minipage}

\caption{Top left: Estimated log marginal likelihoods $\log p\protect\pp{x\vert m^{(n)}}$
and the estimated probability of rain on the next day $\protect\E\protect\bb{z\vert x,m^{(n)}}$
for each of the valid models $n$ for the \textbf{rain} problem. Markers
are colored for four example models from \ref{fig:rain-models}. Top
right: Histogram for estimated probability of rain. Bottom left: Histogram
for estimated marginal likelihoods. The dotted lines at $\frac{8}{22}\approx0.364$
and $\frac{14}{21}\approx.667$ show the base rates for rain and between-day
consistency in rain, respectively. \label{fig:rain-inference}}
\end{figure}

\begin{figure}
\begin{centering}
\hfill{}\includegraphics[scale=0.6]{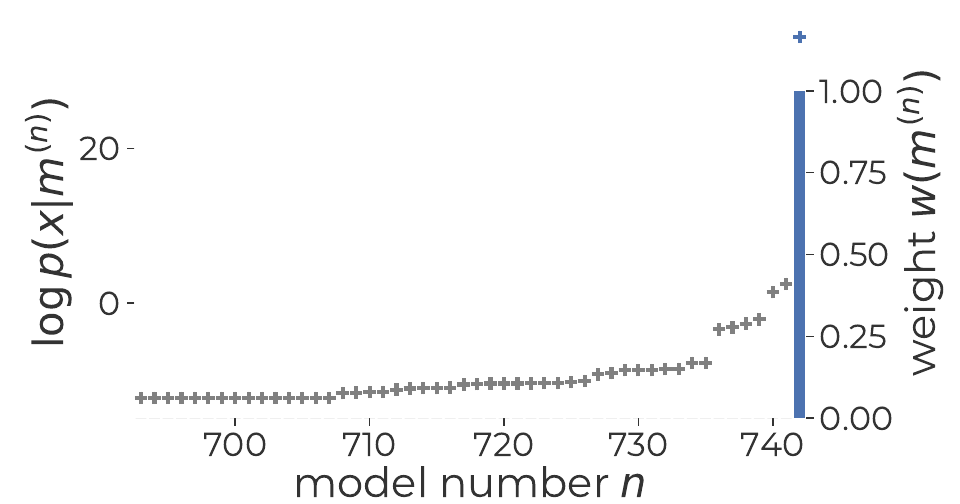}\hfill{}\includegraphics[scale=0.6]{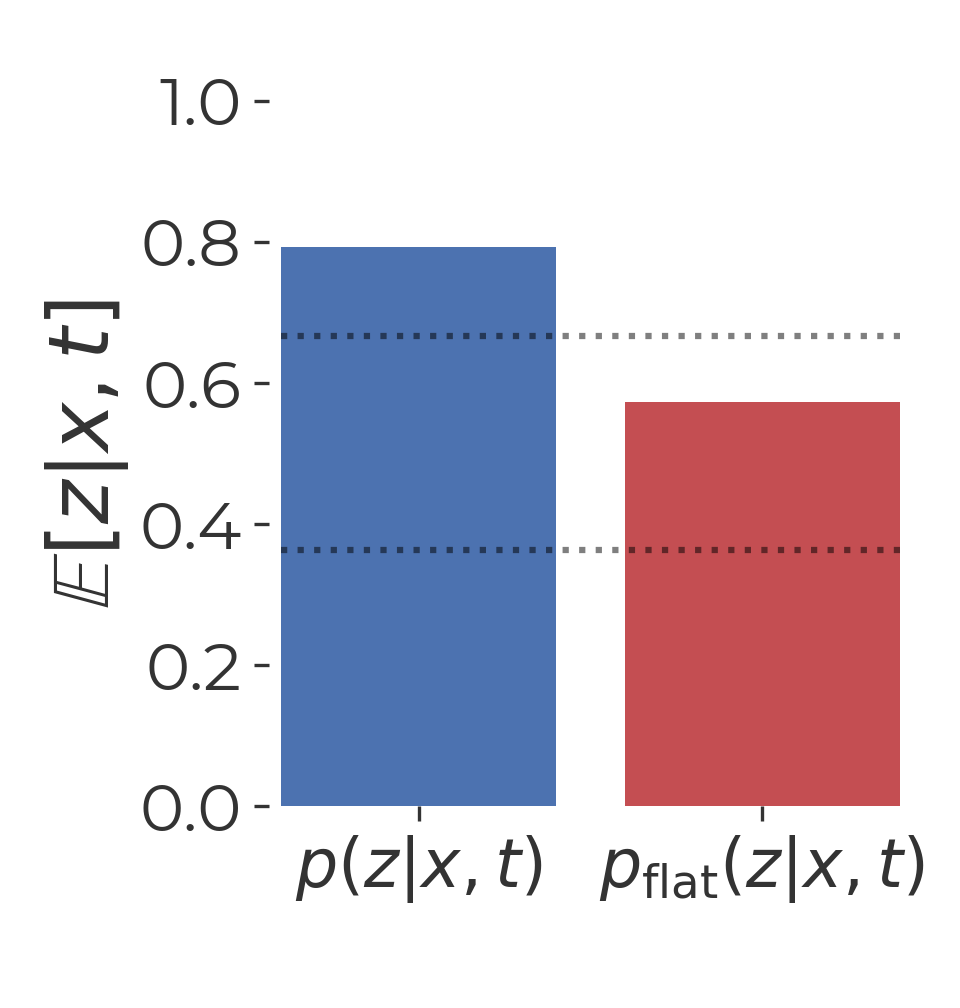}\hfill{}
\par\end{centering}
\caption{Left: Estimated log-marginal likelihoods $\log p\protect\pp{x\vert m^{(n)}}$
(\textquotedblleft +\textquotedblright{} symbols, left axis) and weights
$w(m^{(n)})$ (bars, right axis) for the top 50 models $n$ for the
\textbf{rain} problem. Right: The final estimated posterior, compared
to a \textquotedblleft flat\textquotedblright{} average. The dotted
lines at $\frac{8}{22}\approx0.364$ and $\frac{14}{21}\approx.667$
show the base rates for rain and between-day consistency in rain,
respectively.\label{fig:rain-bma}}
\end{figure}

\cleardoublepage{}

\newpage{}

\subsection{Coin (standard)\label{subsec:coin-standard-full}}

In this problem, the user has flipped a coin several times and would
like to predict the true bias. In this first version, the user strongly
implies that it is a ``standard'' coin with a bias near 0.5. The
observed data are 20 flips, out of which 14 were heads.

\vspace{-0.5cm}

\begin{minipage}[t]{0.48\textwidth}%
\begin{center}
{\scriptsize{\ttfamily{}%
\begin{minipage}[t]{0.97\columnwidth}%
\begin{lstlisting}[
    basicstyle=\scriptsize\ttfamily,     % Typewriter font (monospaced)
   tabsize=1,
   frame=single,
   rulecolor=\color{gray!100},
   backgroundcolor=\color{gray!03},
    breaklines=true,
    xleftmargin=0pt,
    xrightmargin=0pt,
    breakindent=0pt,
]
PROBLEM

I have a coin. I've flipped it a bunch of times. I'm wondering what the true bias of the coin is.

I just got the coin from the US mint, so I'm almost completely sure that it's a standard US penny.

DATA

int num_flips; // number of times the coin was flipped
int num_heads; // number of times heads was observed in the flips

GOAL

real bias; // true bias of the coin
\end{lstlisting}%
\end{minipage}}}
\par\end{center}%
\end{minipage}\hfill{}%
\begin{minipage}[t]{0.5\textwidth}%
\begin{center}
{\scriptsize{\ttfamily{}%
\begin{minipage}[t]{0.97\columnwidth}%
\begin{lstlisting}[
    basicstyle=\scriptsize\ttfamily,
   frame=single,
   rulecolor=\color{gray!100},
   backgroundcolor=\color{gray!03},
    breaklines=true,
    xleftmargin=0pt,
    xrightmargin=0pt,
    breakindent=2\baselineskip,
]
{"num_flips": 20, "num_heads": 14}
\end{lstlisting}%
\end{minipage}}}
\par\end{center}%
\end{minipage}

Out of 1024 models generated, 989 (96.6\%) compiled and allowed for
inference. \ref{fig:coin-standard-models} shows four models. \ref{fig:coin-standard-inference}
shows the estimated marginal likelihoods for each of the valid models.
\ref{fig:coin-standard-posteriors} compares the estimated posteriors
for each valid models. Finally, \ref{fig:coin-standard-bma} shows
the estimated marginal likelihoods and corresponding weights for the
top-scoring 50 valid models, as well as the final estimated posterior
$p\pp{z\vert x,t}$ compared to a ``flat'' average $p_{\mathrm{flat}}\pp{z\vert x,t}$.

In this case, while there are two high-scoring models similar to {\color{col1}$m^{(968)}$}
\eqref{fig:coin-standard-top-model}, they are essentially ``outvoted''
in the posterior by lower-scoring but more numerous models similar
to {\color{col2}$m^{(938)}$} \eqref{fig:coin-standard-cluster}.
Thus, the final posterior in \ref{fig:coin-standard-bma} is similar
to the posterior from {\color{col2}$m^{(938)}$}.
\begin{flushleft}
\begin{figure}[H]
\begin{minipage}[t]{0.48\textwidth}%
\begin{center}
{\scriptsize\subfloat[{\color{col1}$m^{(968)}$}, the highest-scoring model. \label{fig:coin-standard-top-model}]{\begin{centering}
{\scriptsize{}%
\begin{minipage}[t]{0.97\columnwidth}%
\begin{lstlisting}[
    basicstyle=\scriptsize\ttfamily\color{col1},    
   tabsize=2,
   frame=single,
   rulecolor=\color{gray!100},
   backgroundcolor=\color{gray!03},
    breaklines=true,
    xleftmargin=0pt,
    xrightmargin=0pt,
    breakindent=2\baselineskip,
]
data {
	int num_flips;
	int num_heads;
}
parameters {
	real logit_bias;
}
transformed parameters {
	real bias = inv_logit(logit_bias);
}
model {
	logit_bias ~ normal(0, 0.1);
	num_heads ~ binomial(num_flips, bias);
}
generated quantities {
	// Nothing to generate here, just inference on the bias parameter
}
\end{lstlisting}%
\end{minipage}}{\scriptsize\par}
\par\end{centering}
{\scriptsize}{\scriptsize\par}}}
\par\end{center}%
\end{minipage}\hfill{}%
\begin{minipage}[t]{0.5\textwidth}%
\begin{center}
{\scriptsize\subfloat[{\color{col2}$m^{(938)}$}, a model from the most common \textquotedblleft cluster\textquotedblright .\label{fig:coin-standard-cluster}]{\begin{centering}
{\scriptsize{}%
\begin{minipage}[t]{0.97\columnwidth}%
\begin{lstlisting}[
    basicstyle=\scriptsize\ttfamily\color{col2},  
   tabsize=2,
   frame=single,
   rulecolor=\color{gray!100},
   backgroundcolor=\color{gray!03},
    breaklines=true,
    xleftmargin=0pt,
    xrightmargin=0pt,
    breakindent=2\baselineskip,
]
data {
	int num_flips;
	int num_heads;
}
parameters {
	real<lower=0, upper=1> bias;
}
model {
	bias ~ normal(0.5, 0.01);
	num_heads ~ binomial(num_flips, bias);
}
\end{lstlisting}%
\end{minipage}}{\scriptsize\par}
\par\end{centering}
{\scriptsize}{\scriptsize\par}}}
\par\end{center}%
\end{minipage}

\begin{minipage}[t]{0.48\textwidth}%
\begin{center}
{\scriptsize\subfloat[{\color{col3}$m^{(385)}$}, a low-scoring model]{\begin{centering}
{\scriptsize{}%
\begin{minipage}[t]{0.97\columnwidth}%
\begin{lstlisting}[
    basicstyle=\scriptsize\ttfamily\color{col3},  
   tabsize=2,
   frame=single,
   rulecolor=\color{gray!100},
   backgroundcolor=\color{gray!03},
    breaklines=true,
    xleftmargin=0pt,
    xrightmargin=0pt,
    breakindent=2\baselineskip,
]
data {
	int num_flips;
	int num_heads;
}
parameters {
	real<lower=0, upper=1> bias;
}
model {
	bias ~ beta(600, 400);
	num_heads ~ binomial(num_flips, bias);
}
\end{lstlisting}%
\end{minipage}}{\scriptsize\par}
\par\end{centering}
{\scriptsize}{\scriptsize\par}}}
\par\end{center}%
\end{minipage}\hfill{}%
\begin{minipage}[t]{0.5\textwidth}%
\begin{center}
{\scriptsize\subfloat[{\color{col4}$m^{(29)}$}, a very low-scoring model]{\begin{centering}
{\scriptsize{}%
\begin{minipage}[t]{0.97\columnwidth}%
\begin{lstlisting}[
    basicstyle=\scriptsize\ttfamily\color{col4},  
   tabsize=2,
   frame=single,
   rulecolor=\color{gray!100},
   backgroundcolor=\color{gray!03},
    breaklines=true,
    xleftmargin=0pt,
    xrightmargin=0pt,
    breakindent=2\baselineskip,
]
data {
	int num_flips;
	int num_heads;
}
parameters {
	real<lower=0, upper=1> bias;
}
model {
	bias ~ beta(10000, 10000);
	num_heads ~ binomial(num_flips, bias);
}
\end{lstlisting}%
\end{minipage}}{\scriptsize\par}
\par\end{centering}
{\scriptsize}{\scriptsize\par}}}
\par\end{center}%
\end{minipage}

\caption{Example models for the \textbf{coin (standard)} problem, sorted by
estimated marginal likelihoods. \label{fig:coin-standard-models}}
\end{figure}
\par\end{flushleft}

\begin{flushleft}
\begin{figure}[H]
\begin{centering}
\includegraphics[width=0.5\textwidth]{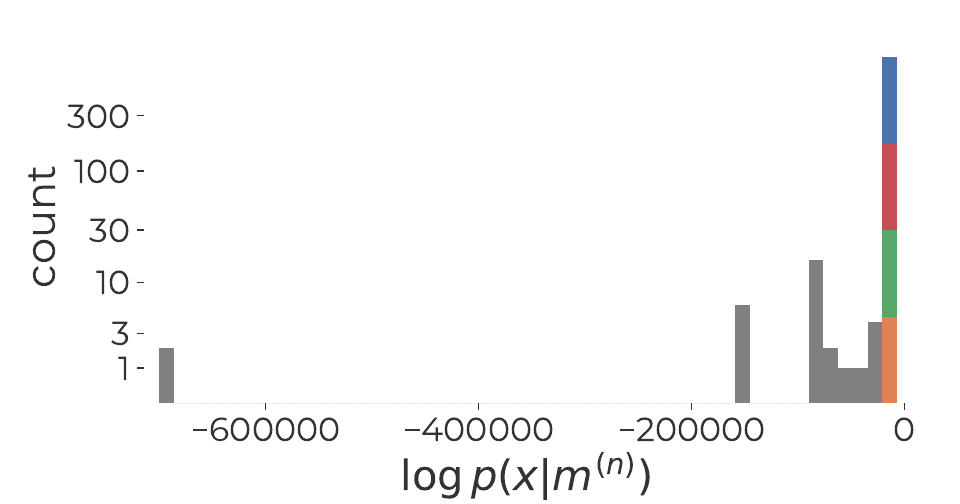}\includegraphics[width=0.5\textwidth]{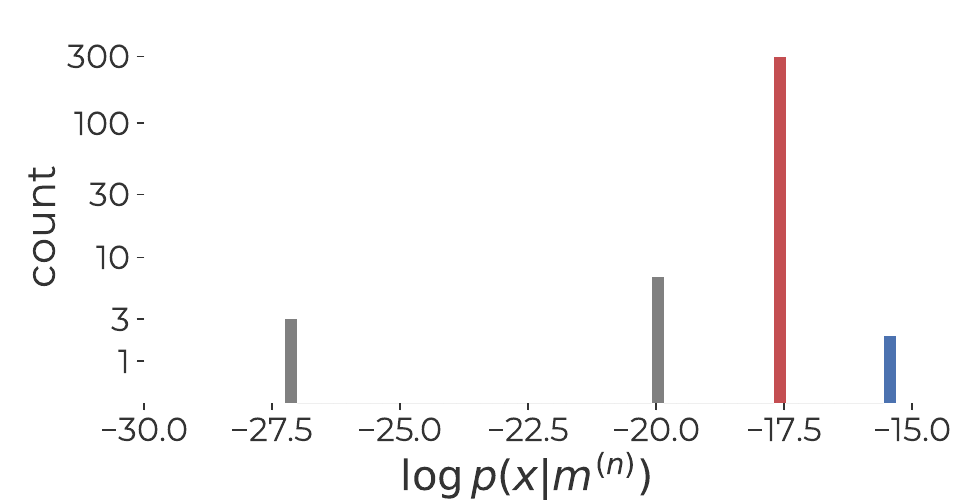}
\par\end{centering}
\caption{Estimated marginal likelihoods $p\protect\pp{x\vert m^{(n)}}$ for
each of the valid models $n$ for the \textbf{coin (standard)} problem.
Because of the many order of magnitude, two different ranges are shown.
Bars are colored for four example models from \ref{fig:coin-standard-models}.
If multiple models map to the same bin, all colors are shown stacked.
\label{fig:coin-standard-inference}}
\end{figure}
\par\end{flushleft}

\begin{flushleft}
\begin{figure}[H]
\begin{centering}
\includegraphics[width=0.5\textwidth]{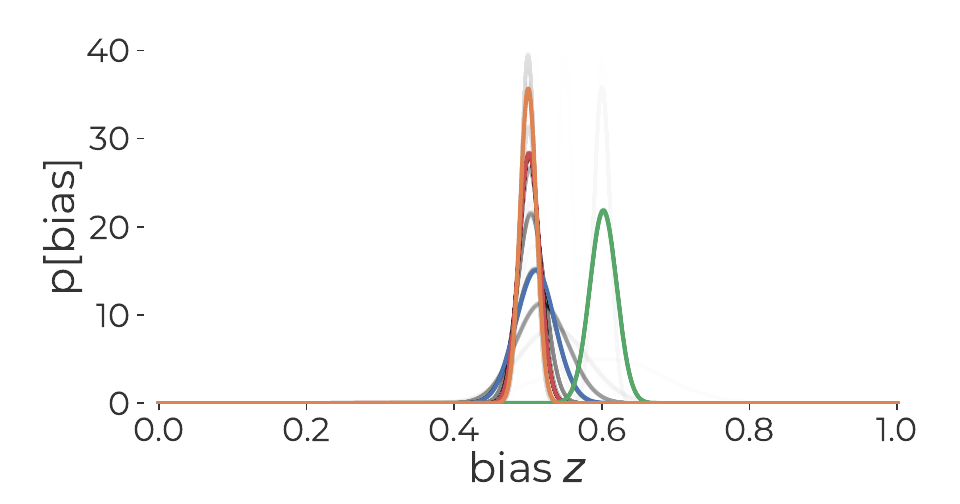}
\par\end{centering}
\caption{The estimated posteriors for each of the valid models $n$ for the
\textbf{coin (standard)} problem. Posteriors are colored for the four
example models from \ref{fig:coin-standard-models}. \label{fig:coin-standard-posteriors}}
\end{figure}
\par\end{flushleft}

\begin{flushleft}
\begin{figure}[H]
\begin{centering}
\includegraphics[width=0.5\columnwidth]{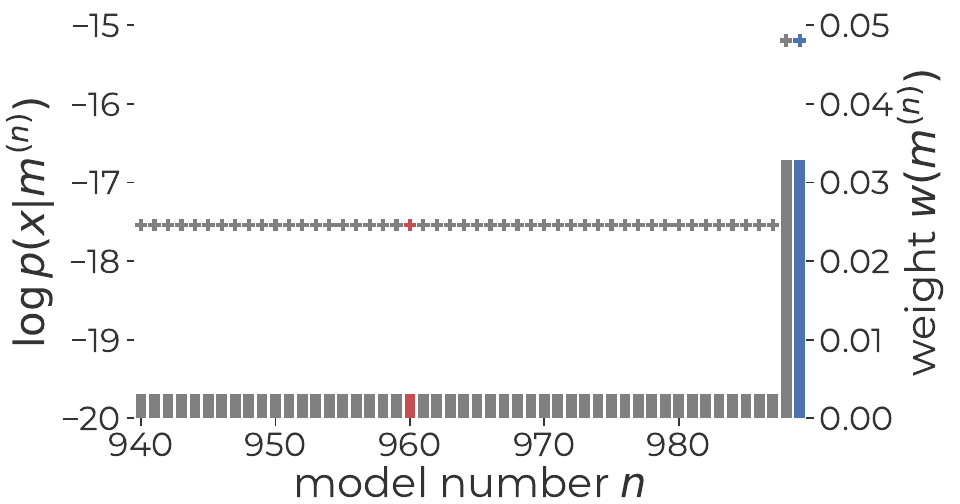}\includegraphics[width=0.5\textwidth]{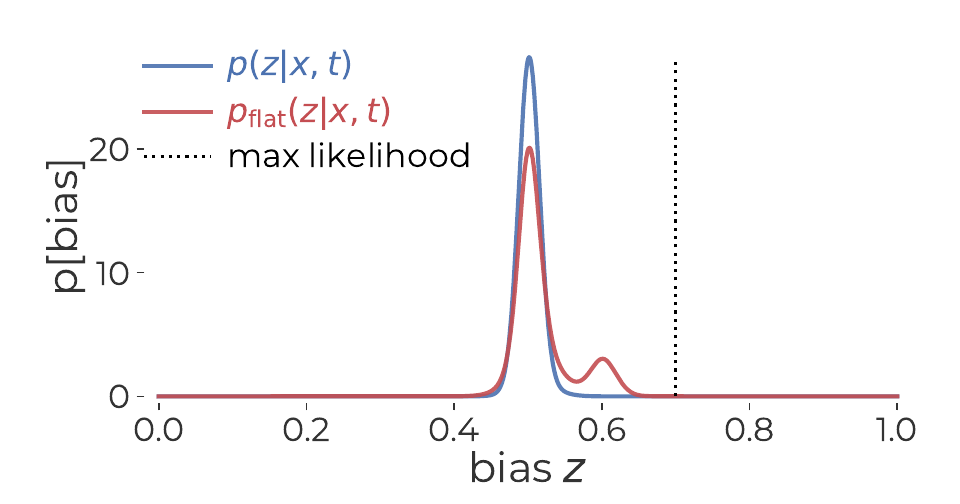}
\par\end{centering}
\caption{Left: Estimated log-marginal likelihoods $\log p\protect\pp{x\vert m^{(n)}}$
(\textquotedblleft +\textquotedblright{} symbols, left axis) and weights
$w(m^{(n)})$ (bars, right axis) for the top 50 models $n$ for the
\textbf{coin (standard)} problem. There are two models similar to
{\color{col1}$m^{(968)}$}, which each receive a weight of approximately
0.036, while there are several hundred models similar to {\color{col2}$m^{(938)}$},
which each receive a weight of approximately 0.0032. Right: The final
estimated posterior, compared to a \textquotedblleft flat\textquotedblright{}
average. The maximum-likelihood estimator of $14/20=0.7$ is also
shown for reference.\label{fig:coin-standard-bma}}
\end{figure}
\par\end{flushleft}

\cleardoublepage{}

\newpage{}

\subsection{Coin (looks)\label{subsec:coin-looks-full}}

In this second version of the coin problem, the user states that the
coin ``looks'' like a standard coin, but implies less confidence.
As before, the observed data are 20 flips, out of which 14 were heads.

\vspace{-0.5cm}

\begin{minipage}[t]{0.48\textwidth}%
\begin{center}
{\scriptsize{\ttfamily{}%
\begin{minipage}[t]{0.97\columnwidth}%
\begin{lstlisting}[
    basicstyle=\scriptsize\ttfamily,     % Typewriter font (monospaced)
   tabsize=1,
   frame=single,
   rulecolor=\color{gray!100},
   backgroundcolor=\color{gray!03},
    breaklines=true,
    xleftmargin=0pt,
    xrightmargin=0pt,
    breakindent=0pt,
]
PROBLEM

I have a coin. I've flipped it a bunch of times. I'm wondering what the true bias of the coin is.

At first glance, it appears to be a standard US penny, although I haven't examined it closely.

DATA

int num_flips; // number of times the coin was flipped
int num_heads; // number of times heads was observed in the flips

GOAL

real bias; // true bias of the coin
\end{lstlisting}%
\end{minipage}}}
\par\end{center}%
\end{minipage}\hfill{}%
\begin{minipage}[t]{0.5\textwidth}%
\begin{center}
{\scriptsize{\ttfamily{}%
\begin{minipage}[t]{0.97\columnwidth}%
\begin{lstlisting}[
    basicstyle=\scriptsize\ttfamily,
   frame=single,
   rulecolor=\color{gray!100},
   backgroundcolor=\color{gray!03},
    breaklines=true,
    xleftmargin=0pt,
    xrightmargin=0pt,
    breakindent=2\baselineskip,
]
{"num_flips": 20, "num_heads": 14}
\end{lstlisting}%
\end{minipage}}}
\par\end{center}%
\end{minipage}

Out of 1024 models generated, 983 (96.0\%) compiled and allowed for
inference. \ref{fig:coin-looks-models} shows four models, while the
results of inference are shown in \ref{fig:coin-looks-inference},
\ref{fig:coin-looks-posteriors}, and \ref{fig:coin-looks-bma}.

In this case, there are seven models similar to the top-scoring model
{\color{col1}$m^{(960)}$} \eqref{fig:coin-looks-top-model} that
each have weight of around 0.1, and seven models similar to {\color{col2}$m^{(949)}$}
\eqref{fig:coin-looks-lower-model} with a weight of around 0.03.
All of these have priors with a reasonable amount of uncertainty.
There are many models similar to {\color{col3}$m^{(661)}$} that
have prior with low variance, but these have a low marginal-likelihood
and so get very low weight. The final estimated posterior in \ref{fig:coin-looks-bma}
thus resembles a mixture of the posteriors of {\color{col1}$m^{(960)}$}
and {\color{col2}$m^{(949)}$}.
\begin{flushleft}
\begin{figure}[H]
\begin{minipage}[t]{0.48\textwidth}%
\begin{center}
{\scriptsize\subfloat[{\color{col1}$m^{(960)}$}, the highest-scoring model.\label{fig:coin-looks-top-model}]{\begin{centering}
{\scriptsize{}%
\begin{minipage}[t]{0.97\columnwidth}%
\begin{lstlisting}[
    basicstyle=\scriptsize\ttfamily\color{col1},    
   tabsize=2,
   frame=single,
   rulecolor=\color{gray!100},
   backgroundcolor=\color{gray!03},
    breaklines=true,
    xleftmargin=0pt,
    xrightmargin=0pt,
    breakindent=2\baselineskip,
]
data {
	int num_flips; 
	int num_heads; 
}
parameters {
	real<lower=0, upper=1> bias;
}
model {
	bias ~ uniform(0, 1);
	num_heads ~ binomial(num_flips, bias);
}
\end{lstlisting}%
\end{minipage}}{\scriptsize\par}
\par\end{centering}
{\scriptsize}{\scriptsize\par}}}
\par\end{center}%
\end{minipage}\hfill{}%
\begin{minipage}[t]{0.5\textwidth}%
\begin{center}
{\scriptsize\subfloat[{\color{col2}$m^{(949)}$}, a slightly lower-scoring model.\label{fig:coin-looks-lower-model}]{\begin{centering}
{\scriptsize{}%
\begin{minipage}[t]{0.97\columnwidth}%
\begin{lstlisting}[
    basicstyle=\scriptsize\ttfamily\color{col2},  
   tabsize=2,
   frame=single,
   rulecolor=\color{gray!100},
   backgroundcolor=\color{gray!03},
    breaklines=true,
    xleftmargin=0pt,
    xrightmargin=0pt,
    breakindent=2\baselineskip,
]
data {
	int num_flips;
	int num_heads;
}
parameters {
	real<lower=0, upper=1> bias;
}
model {
	bias ~ normal(0.5, 0.1);
	num_heads ~ binomial(num_flips, bias);
}
\end{lstlisting}%
\end{minipage}}{\scriptsize\par}
\par\end{centering}
{\scriptsize}{\scriptsize\par}}}
\par\end{center}%
\end{minipage}

\begin{minipage}[t]{0.48\textwidth}%
\begin{center}
{\scriptsize\subfloat[{\color{col3}$m^{(661)}$}, from the most common \textquotedblleft cluster\textquotedblright{}
of models.\label{fig:coin-looks-cluster-model}]{\begin{centering}
{\scriptsize{}%
\begin{minipage}[t]{0.97\columnwidth}%
\begin{lstlisting}[
    basicstyle=\scriptsize\ttfamily\color{col3},  
   tabsize=2,
   frame=single,
   rulecolor=\color{gray!100},
   backgroundcolor=\color{gray!03},
    breaklines=true,
    xleftmargin=0pt,
    xrightmargin=0pt,
    breakindent=2\baselineskip,
]
data {
	int num_flips;
	int num_heads;
}
parameters {
	real<lower=0, upper=1> bias;
}
model {
	bias ~ beta(60,60);
	num_heads ~ binomial(num_flips, bias);
}
\end{lstlisting}%
\end{minipage}}{\scriptsize\par}
\par\end{centering}
{\scriptsize}{\scriptsize\par}}}
\par\end{center}%
\end{minipage}\hfill{}%
\begin{minipage}[t]{0.5\textwidth}%
\begin{center}
{\scriptsize\subfloat[{\color{col4}$m^{(1)}$}, the lowest-scoring model]{\begin{centering}
{\scriptsize{}%
\begin{minipage}[t]{0.97\columnwidth}%
\begin{lstlisting}[
    basicstyle=\scriptsize\ttfamily\color{col4},  
   tabsize=2,
   frame=single,
   rulecolor=\color{gray!100},
   backgroundcolor=\color{gray!03},
    breaklines=true,
    xleftmargin=0pt,
    xrightmargin=0pt,
    breakindent=2\baselineskip,
]
data {
	int num_flips;
	int num_heads;
}
parameters {
	real<lower=0, upper=1> bias;
}
model {
	bias ~ beta(600,600);
	num_heads ~ binomial(num_flips, bias);
}
\end{lstlisting}%
\end{minipage}}{\scriptsize\par}
\par\end{centering}
{\scriptsize}{\scriptsize\par}}}
\par\end{center}%
\end{minipage}

\caption{Example models for the \textbf{coin (looks)} problem, sorted by estimated
marginal likelihoods. \label{fig:coin-looks-models}}
\end{figure}
\par\end{flushleft}

\begin{flushleft}
\begin{figure}
\begin{centering}
\includegraphics[width=0.5\textwidth]{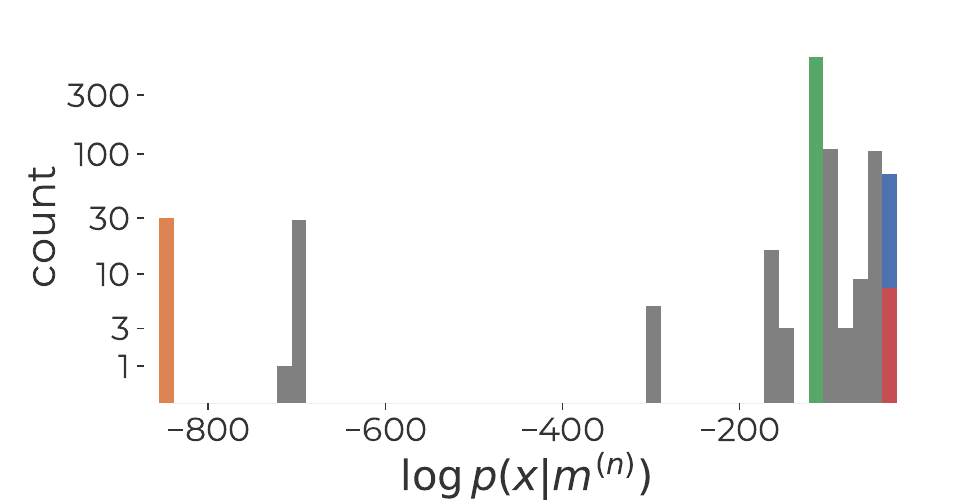}\includegraphics[width=0.5\textwidth]{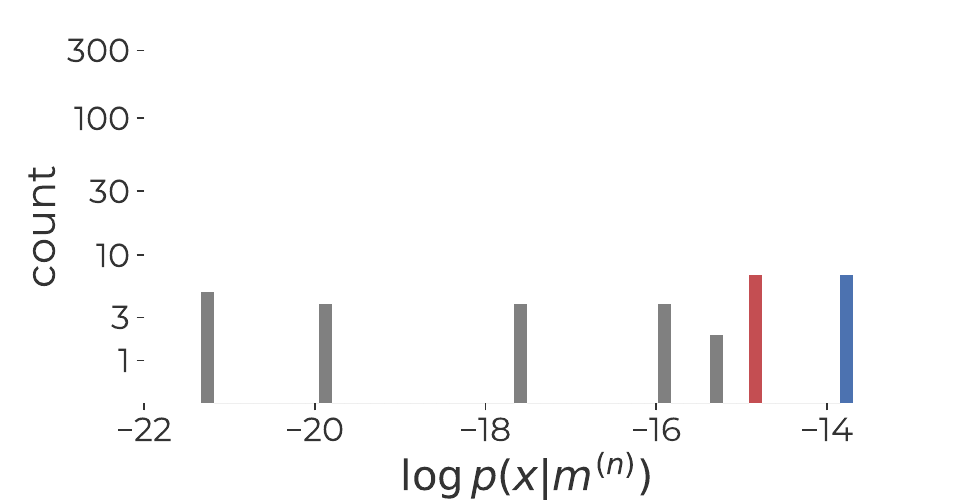}
\par\end{centering}
\caption{Estimated marginal likelihoods $p\protect\pp{x\vert m^{(n)}}$ for
each of the valid models $n$ for the \textbf{coin (looks)} problem.
Because of the many order of magnitude, two different ranges are shown.
Bars are colored for four example models from \ref{fig:coin-looks-models}.
If multiple models map to the same bin, all colors are shown stacked.
\label{fig:coin-looks-inference}}
\end{figure}
\par\end{flushleft}

\begin{flushleft}
\begin{figure}
\begin{centering}
\includegraphics[width=0.5\textwidth]{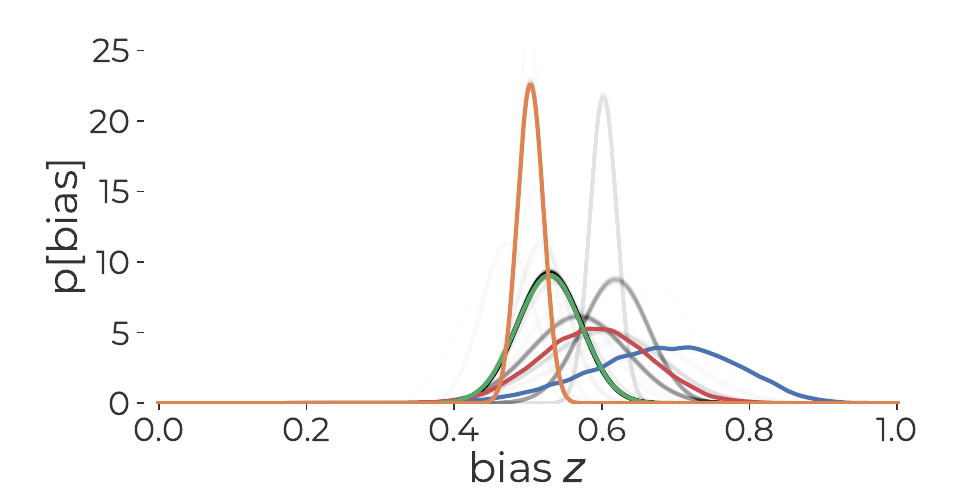}
\par\end{centering}
\caption{The estimated posteriors for each of the valid models $n$ for the
\textbf{coin (looks)} problem. Posteriors are colored for the four
example models from \ref{fig:coin-looks-models}. \label{fig:coin-looks-posteriors}}
\end{figure}
\par\end{flushleft}

\begin{flushleft}
\begin{figure}
\begin{centering}
\includegraphics[width=0.5\columnwidth]{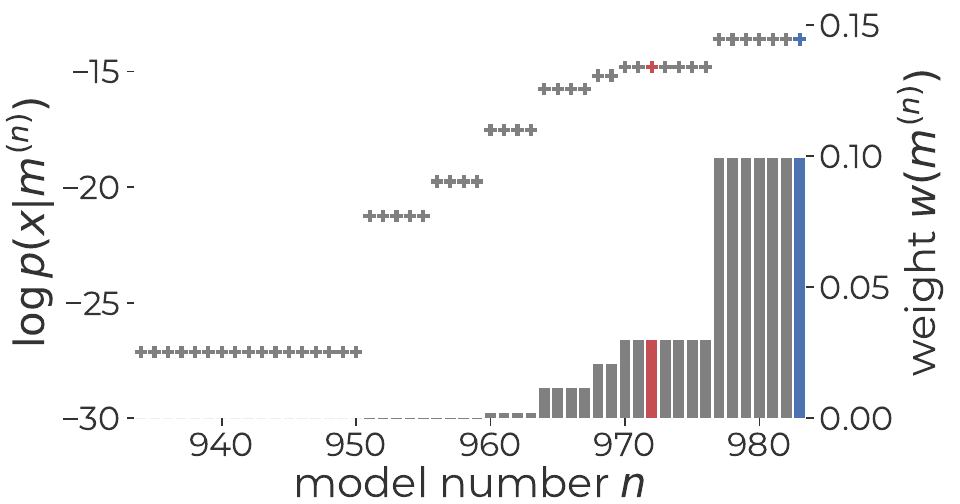}\includegraphics[width=0.5\textwidth]{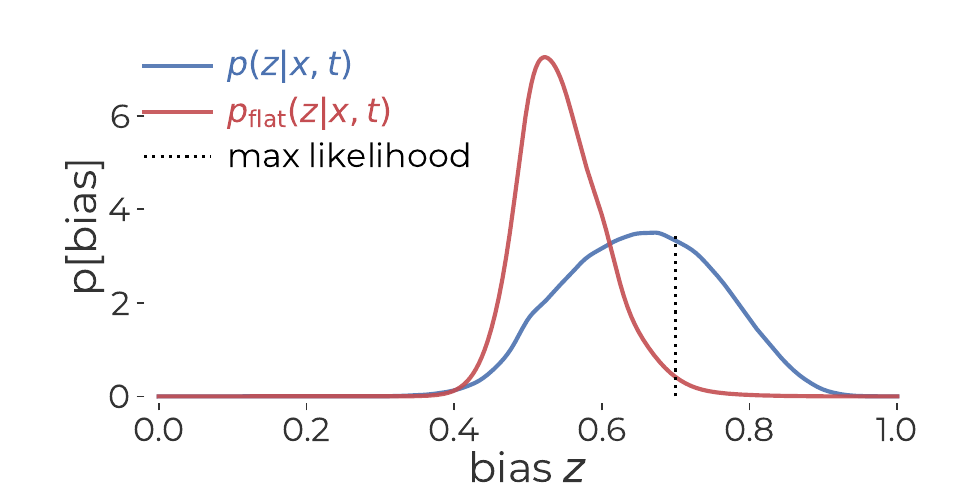}
\par\end{centering}
\caption{Left: Estimated log-marginal likelihoods $\log p\protect\pp{x\vert m^{(n)}}$
(\textquotedblleft +\textquotedblright{} symbols, left axis) and weights
$w(m^{(n)})$ (bars, right axis) for the top 50 models $n$ for the
\textbf{coin (looks)} problem. Right: The final estimated posterior,
compared to a \textquotedblleft flat\textquotedblright{} average. The
maximum-likelihood estimator of $14/20=0.7$ is also shown for reference.\label{fig:coin-looks-bma}}
\end{figure}
\par\end{flushleft}

\cleardoublepage{}

\subsection{Coin (bent)\label{subsec:coin-bent-full}}

In this third version of the coin problem, the user states that the
coin looks ``bent'', implying great uncertainty about the true bias.
As before, the observed data are 20 flips, out of which 14 were heads.

\begin{minipage}[t]{0.48\textwidth}%
\begin{center}
{\scriptsize{\ttfamily{}%
\begin{minipage}[t]{0.97\columnwidth}%
\begin{lstlisting}[
    basicstyle=\scriptsize\ttfamily,     % Typewriter font (monospaced)
   tabsize=1,
   frame=single,
   rulecolor=\color{gray!100},
   backgroundcolor=\color{gray!03},
    breaklines=true,
    xleftmargin=0pt,
    xrightmargin=0pt,
    breakindent=0pt,
]
PROBLEM

I have a coin. I've flipped it a bunch of times. I'm wondering what the true bias of the coin is.

I can see that the penny is quite bent, so I'm pretty sure the bias is different from 0.5, though I can't tell in what direction.

DATA

int num_flips; // number of times the coin was flipped
int num_heads; // number of times heads was observed in the flips

GOAL

real bias; // true bias of the coin
\end{lstlisting}%
\end{minipage}}}
\par\end{center}%
\end{minipage}\hfill{}%
\begin{minipage}[t]{0.5\textwidth}%
\begin{center}
{\scriptsize{\ttfamily{}%
\begin{minipage}[t]{0.97\columnwidth}%
\begin{lstlisting}[
    basicstyle=\scriptsize\ttfamily,
   frame=single,
   rulecolor=\color{gray!100},
   backgroundcolor=\color{gray!03},
    breaklines=true,
    xleftmargin=0pt,
    xrightmargin=0pt,
    breakindent=2\baselineskip,
]
{"num_flips": 20, "num_heads": 14}
\end{lstlisting}%
\end{minipage}}}
\par\end{center}%
\end{minipage}

\vspace{-0.1cm}

Out of 1024 models generated, 848 (82.8\%) compiled and allowed for
inference. \ref{fig:coin-bent-models} shows four models, while the
results of inference are shown in \ref{fig:coin-bent-inference},
\ref{fig:coin-bent-posteriors}, and \ref{fig:coin-bent-bma}.

In this case, many models contribute substantially to the posterior.
While {\color{col2}$m^{(811)}$} has a slightly lower estimated
marginal likelihood than the highest scoring model ({\color{col1}$m^{(830)}$}),
models similar to it were generated many times, and so it wields the
most influence in the posterior.
\begin{flushleft}
\begin{figure}[H]
\begin{minipage}[t]{0.48\textwidth}%
\begin{center}
{\scriptsize\subfloat[{\color{col1}$m^{(830)}$}, the highest-scoring model]{\begin{centering}
{\scriptsize{}%
\begin{minipage}[t]{0.97\columnwidth}%
\begin{lstlisting}[
    basicstyle=\scriptsize\ttfamily\color{col1},    
   tabsize=2,
   frame=single,
   rulecolor=\color{gray!100},
   backgroundcolor=\color{gray!03},
    breaklines=true,
    xleftmargin=0pt,
    xrightmargin=0pt,
    breakindent=2\baselineskip,
]
data {
	int num_flips;
	int num_heads;
}
parameters {
	real bias_logit;
}
transformed parameters {
	real<lower=0, upper=1> bias = inv_logit(bias_logit);
}
model {
	bias_logit ~ uniform(-10, 10); // diffuse prior that distributes mass fairly evenly over [0,1] after logistic transform
	num_heads ~ binomial(num_flips, bias);
}
generated quantities {
	// nothing
}
\end{lstlisting}%
\end{minipage}}{\scriptsize\par}
\par\end{centering}
{\scriptsize}{\scriptsize\par}}}
\par\end{center}%
\end{minipage}\hfill{}%
\begin{minipage}[t]{0.5\textwidth}%
\begin{center}
{\scriptsize\subfloat[{\color{col2}$m^{(811)}$}, a typical model]{\begin{centering}
{\scriptsize{}%
\begin{minipage}[t]{0.97\columnwidth}%
\begin{lstlisting}[
    basicstyle=\scriptsize\ttfamily\color{col2},  
   tabsize=2,
   frame=single,
   rulecolor=\color{gray!100},
   backgroundcolor=\color{gray!03},
    breaklines=true,
    xleftmargin=0pt,
    xrightmargin=0pt,
    breakindent=2\baselineskip,
]
data {
	int num_flips;
	int num_heads;
}
parameters {
	real<lower=0, upper=1> bias;
}
model {
	bias ~ beta(0.5, 0.5); // a broad prior, but with more mass near 0 and 1
	num_heads ~ binomial(num_flips, bias);
}
\end{lstlisting}%
\end{minipage}}{\scriptsize\par}
\par\end{centering}
{\scriptsize}{\scriptsize\par}}}
\par\end{center}%
\end{minipage}

\begin{minipage}[t]{0.48\textwidth}%
\begin{center}
{\scriptsize\subfloat[{\color{col3}$m^{(531)}$}, from the most common \textquotedblleft cluster\textquotedblright{}
of models]{\begin{centering}
{\scriptsize{}%
\begin{minipage}[t]{0.97\columnwidth}%
\begin{lstlisting}[
    basicstyle=\scriptsize\ttfamily\color{col3},  
   tabsize=2,
   frame=single,
   rulecolor=\color{gray!100},
   backgroundcolor=\color{gray!03},
    breaklines=true,
    xleftmargin=0pt,
    xrightmargin=0pt,
    breakindent=2\baselineskip,
]
data {
	int num_flips;
	int num_heads;
}
parameters {
	real<lower=0, upper=1> bias;
}
model {
	bias ~ beta(2,2); // symmetric around 0.5, low density at 0.5
	num_heads ~ binomial(num_flips, bias);
}
\end{lstlisting}%
\end{minipage}}{\scriptsize\par}
\par\end{centering}
{\scriptsize}{\scriptsize\par}}}
\par\end{center}%
\end{minipage}\hfill{}%
\begin{minipage}[t]{0.5\textwidth}%
\begin{center}
{\scriptsize\subfloat[{\color{col4}$m^{(1)}$}, the lowest-scoring model]{\begin{centering}
{\scriptsize{}%
\begin{minipage}[t]{0.97\columnwidth}%
\begin{lstlisting}[
    basicstyle=\scriptsize\ttfamily\color{col4},  
   tabsize=2,
   frame=single,
   rulecolor=\color{gray!100},
   backgroundcolor=\color{gray!03},
    breaklines=true,
    xleftmargin=0pt,
    xrightmargin=0pt,
    breakindent=2\baselineskip,
]
data {
	int num_flips;
	int num_heads;
}
parameters {
	real<lower=0, upper=1> bias;
}
model {
	bias ~ beta(20,30);
	num_heads ~ binomial(num_flips, bias);
}
\end{lstlisting}%
\end{minipage}}{\scriptsize\par}
\par\end{centering}
{\scriptsize}{\scriptsize\par}}}
\par\end{center}%
\end{minipage}

\caption{Example models for the \textbf{coin (bent)} problem, sorted by estimated
marginal likelihoods. \label{fig:coin-bent-models}}
\end{figure}
\par\end{flushleft}

\begin{flushleft}
\begin{figure}[H]
\begin{centering}
\includegraphics[width=0.5\textwidth]{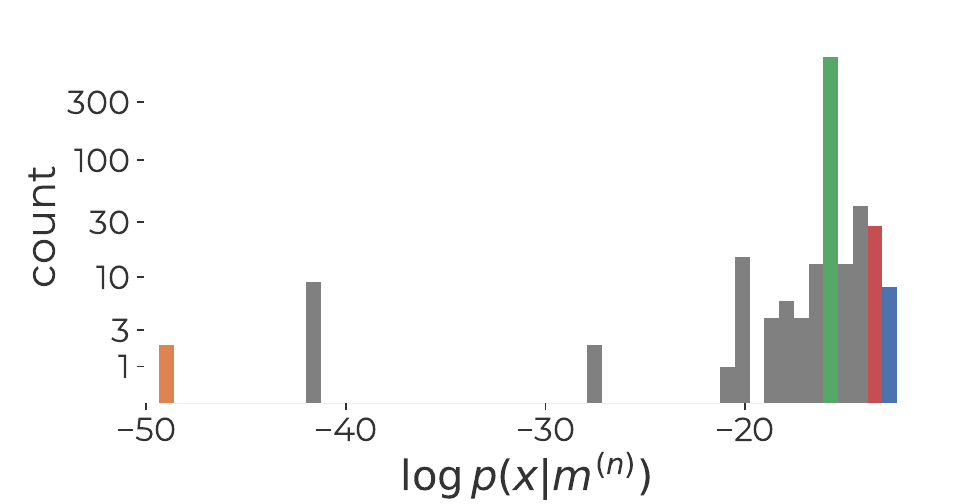}\includegraphics[width=0.5\textwidth]{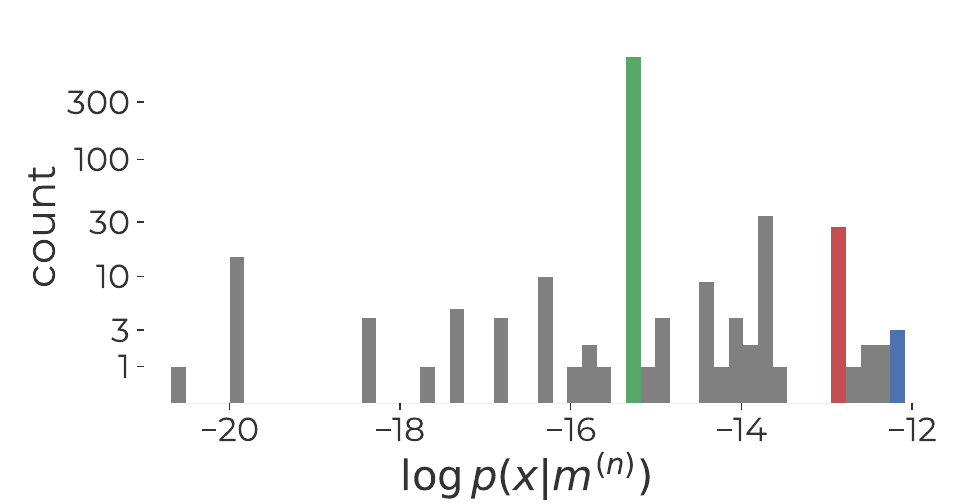}
\par\end{centering}
\caption{Estimated marginal likelihoods $p\protect\pp{x\vert m^{(n)}}$ for
each of the valid models $n$ for the \textbf{coin (bent)} problem.
Because of the many order of magnitude, two different ranges are shown.
Bars are colored for four example models from \ref{fig:coin-bent-models}.
If multiple models map to the same bin, all colors are shown stacked.
\label{fig:coin-bent-inference}}
\end{figure}
\par\end{flushleft}

\begin{flushleft}
\begin{figure}[H]
\begin{centering}
\includegraphics[width=0.5\textwidth]{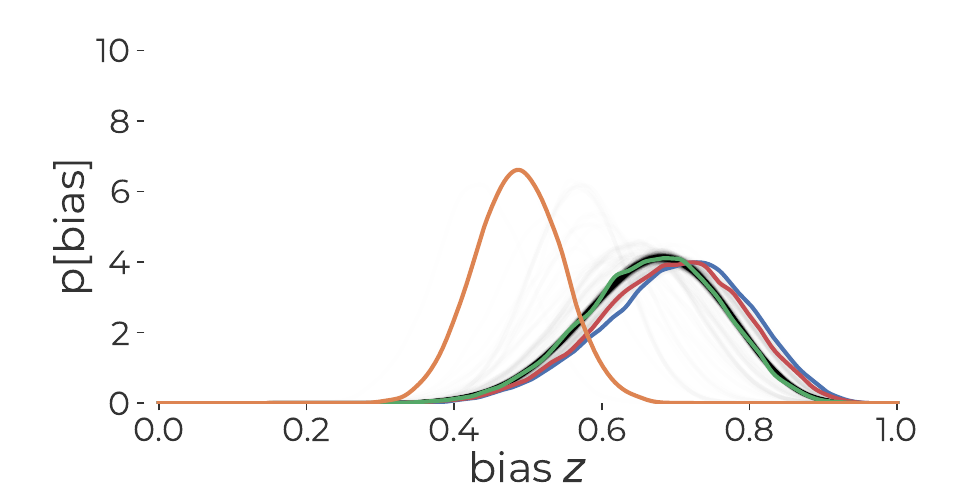}
\par\end{centering}
\caption{The estimated posteriors for each of the valid models $n$ for the
\textbf{coin (looks)} problem. Posteriors are colored for the four
example models from \ref{fig:coin-bent-models}. \label{fig:coin-bent-posteriors}}
\end{figure}
\par\end{flushleft}

\begin{flushleft}
\begin{figure}[H]
\begin{centering}
\includegraphics[width=0.5\columnwidth]{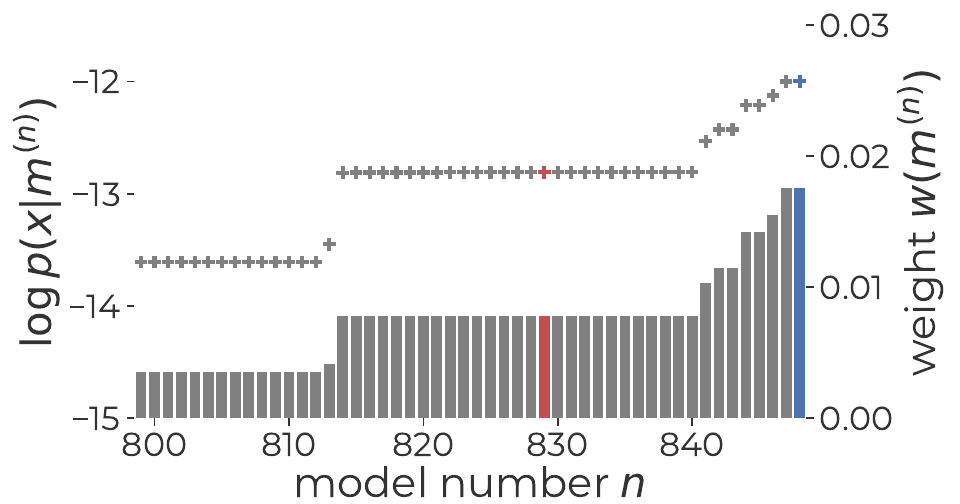}\includegraphics[width=0.5\textwidth]{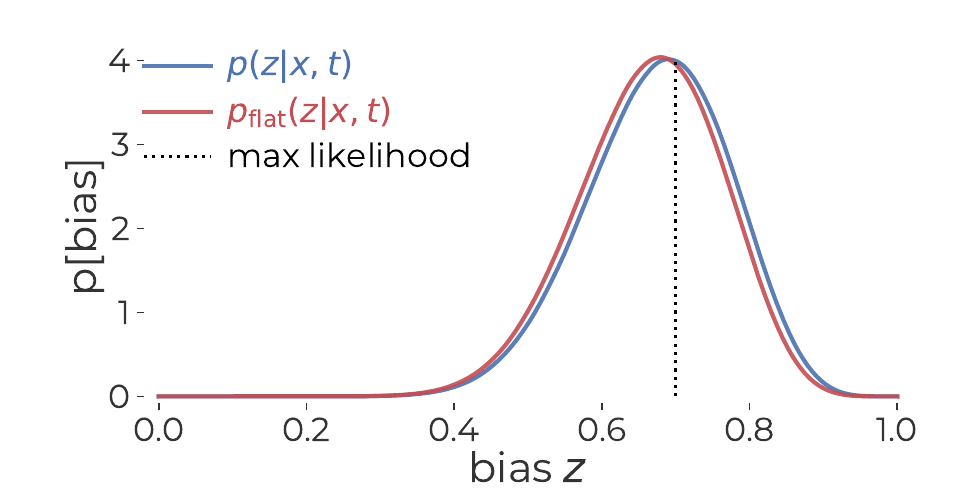}
\par\end{centering}
\caption{Left: Estimated log-marginal likelihoods $\log p\protect\pp{x\vert m^{(n)}}$
(\textquotedblleft +\textquotedblright{} symbols, left axis) and weights
$w(m^{(n)})$ (bars, right axis) for the top 50 models $n$ for the
\textbf{coin (bent)} problem. Right: The final estimated posteriors,
compared to a \textquotedblleft flat\textquotedblright{} average. The
maximum-likelihood estimator of $14/20=0.7$ is also shown for reference.\label{fig:coin-bent-bma}}
\end{figure}
\par\end{flushleft}

\cleardoublepage{}

\newpage{}

\subsection{Polling\label{subsec:polling-full}}

In this problem, the user describes a politician with fluctuating
levels of popularity and different pollsters who take polls at different
times. The goal is to predict the true popularity on each day.

\begin{lstlisting}[
    basicstyle=\scriptsize\ttfamily,     % Typewriter font (monospaced)
   frame=single,
   rulecolor=\color{gray!100},
   backgroundcolor=\color{gray!03},
    breaklines=true,
    xleftmargin=5pt,
    xrightmargin=5pt,
    breakindent=0\baselineskip,
]
PROBLEM

There is a politician. They have some true approval rating (between 0 and 100) that changes over time. There are some pollsters that measure the popularity of the politician. Each pollster does polls on some different subset of days. Each pollster does polls with some differing level of noise.

I have no idea what their popularity might be at the beginning of the year. Make sure to model the variability in the true popularity, and make sure to model a distribution over the noise levels of different pollsters.
DATA

int num_days; // number of days over which popularity is modeled
int num_pollsters; // number of pollsters
int max_polls; // maximum number of polls performed by any pollster
array[num_pollsters] int num_polls; // number of polls done by each pollster
array[num_pollsters, max_polls] int day; // on what day was each poll done? (day[i,j] is only meaningful when j <= num_polls[i])
array[num_pollsters, max_polls] real polls; // actual polling data (polls[i,j] is only meaningful when j <= num_polls[i])

GOAL

array[num_days] popularity; // true popularity of the politician on each day
\end{lstlisting}

\vspace{-0.2cm}A sample of the observed data is shown below, while
the full data is plotted in \ref{fig:polling-data}.\vspace{-0.2cm}

{\scriptsize{\ttfamily{}%
\noindent\begin{minipage}[t]{1\columnwidth}%
\begin{lstlisting}[
    basicstyle=\scriptsize\ttfamily,
   tabsize=2,
   frame=single,
   rulecolor=\color{gray!100},
   backgroundcolor=\color{gray!03},
    breaklines=true,
    xleftmargin=0pt,
    xrightmargin=0pt,
    breakindent=2\baselineskip,
]
{
    "num_days": 365,
    "num_pollsters": 3,
    "max_polls": 100,
    "num_polls": [5, 20, 100],
    "day": [[291, 16, ..., 1],
	    [365, 151, ..., 1],
	    [314, 66, ..., 57]],
    "polls": [[51.385352595786244, 49.741758128722864, ..., 0.0],
	      [58.24149596144896, 40.825747460445996, ..., 0.0],
	      [46.236574610204, 34.54862941696214, ..., 32.98157260195043]]
}
\end{lstlisting}%
\end{minipage}}}{\scriptsize\par}

\vspace{-0.2cm}
\begin{figure}[H]
\begin{centering}
\includegraphics[width=0.6\textwidth]{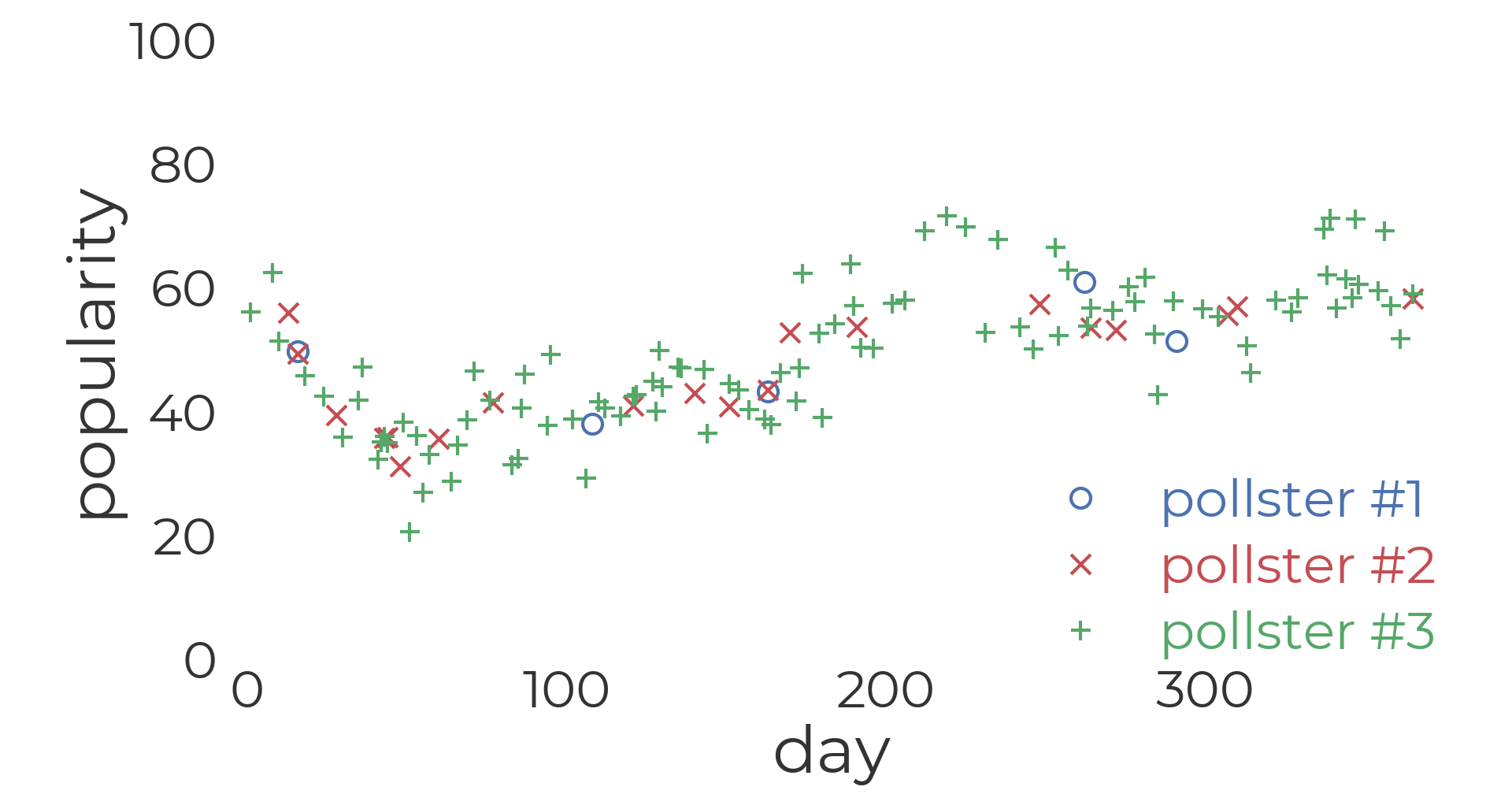}
\par\end{centering}
\caption{Observed data for the \textbf{polling} problem.\label{fig:polling-data}}
\end{figure}
\vspace{-0.2cm}

Out of 1024 models generated, 568 (55.5\%) compiled and allowed for
inference. \ref{fig:polling-models} shows three models. \ref{fig:polling-inference}
shows the estimated marginal likelihoods for each of the valid models.
\ref{fig:polling-posteriors} compares the estimated posteriors for
each valid models. Finally, \ref{fig:polling-bma} shows the estimated
marginal likelihoods and corresponding weights for the top-scoring
50 valid models, as well as the final estimated posterior $p\pp{z\vert x,t}$
compared to a ``flat'' average $p_{\mathrm{flat}}\pp{z\vert x,t}$.

In this case, the highest-scoring model {\color{col1}$m^{(568)}$}
\eqref{fig:polling-highest-model} is very ``simple'' yet has a
high enough marginal likelihood to take almost all weight in the final
posterior in \ref{fig:polling-bma}. The LLM did attempt to create
many more sophisticated models, e.g. hierarchical models. But these
often failed to compile or, as with {\color{col3}$m^{(89)}$}, simply
had a lower marginal-likelihood. 
\begin{flushleft}
\begin{figure}[H]
\begin{minipage}[t]{0.48\textwidth}%
\begin{center}
{\scriptsize\subfloat[{\color{col1}$m^{(568)}$}, the highest-scoring model. \label{fig:polling-highest-model}]{\begin{centering}
{\scriptsize{}%
\begin{minipage}[t]{0.97\columnwidth}%
\begin{lstlisting}[
    basicstyle=\scriptsize\ttfamily\color{col1},    
   tabsize=2,
   frame=single,
   rulecolor=\color{gray!100},
   backgroundcolor=\color{gray!03},
    breaklines=true,
    xleftmargin=0pt,
    xrightmargin=0pt,
    breakindent=2\baselineskip,
]
data {
	int num_days;
	int num_pollsters;
	int max_polls;
	array[num_pollsters] int num_polls;
	array[num_pollsters, max_polls] int day;
	array[num_pollsters, max_polls] real polls;
}
transformed data {
	vector[num_days] zeros = rep_vector(0, num_days);
}
parameters {
	vector[num_days] popularity;
	array[num_pollsters] real<lower=0> sigma;
}
model {
	//_polls are modeled by a distribution with some underlying mean and noise sigma 
	//which is pollster-dependent
	popularity[1] ~ uniform(0, 100);
	for (i in 2:num_days){
		popularity[i] ~ normal(popularity[i-1], 10);
	}
	for (i in 1:num_pollsters){
		sigma[i] ~ inv_gamma(5, 20);
		for (j in 1:num_polls[i]){
			polls[i,j] ~ normal(popularity[day[i,j]], sigma[i]);
		}
	}
}
\end{lstlisting}%
\end{minipage}}{\scriptsize\par}
\par\end{centering}
{\scriptsize}{\scriptsize\par}}}
\par\end{center}%
\end{minipage}\hfill{}%
\begin{minipage}[t]{0.5\textwidth}%
\begin{center}
{\scriptsize\subfloat[{\color{col2}$m^{(553)}$}, a lower-scoring model.]{\begin{centering}
{\scriptsize{}%
\begin{minipage}[t]{0.97\columnwidth}%
\begin{lstlisting}[
    basicstyle=\scriptsize\ttfamily\color{col2},  
   tabsize=2,
   frame=single,
   rulecolor=\color{gray!100},
   backgroundcolor=\color{gray!03},
    breaklines=true,
    xleftmargin=0pt,
    xrightmargin=0pt,
    breakindent=2\baselineskip,
]
data { ... }
parameters {
    array[num_days] real popularity;
    array[num_pollsters] real<lower=0> pollster_noise;
}
model {
    // popularity changes over time according to a simple random walk
    popularity[1] ~ normal(50,20); // we have no idea what the initial popularity is
    for (i in 2:num_days){
        popularity[i] ~ normal(popularity[i-1], 5); // not too much change from one day to the next
    }
    // observed polls are noisy versions of the true popularity
    for (i in 1:num_pollsters){
        for (j in 1:num_polls[i]){
            polls[i,j] ~ normal(popularity[day[i,j]], pollster_noise[i]);
        }
    }
    // prior over the noise levels of the pollsters
    for (i in 1:num_pollsters){
        pollster_noise[i] ~ cauchy(0, 5);
    }
}
generated quantities {
    array[num_days] real popularity_sample = popularity;
}
\end{lstlisting}%
\end{minipage}}{\scriptsize\par}
\par\end{centering}
{\scriptsize}{\scriptsize\par}}}
\par\end{center}%
\end{minipage}

\noindent\begin{minipage}[t]{1\textwidth}%
\begin{center}
{\scriptsize\subfloat[{\color{col3}$m^{(89)}$}, a lower-scoring (hierarchical) model.\label{fig:polling-heir-model}]{\begin{centering}
{\scriptsize{}%
\begin{minipage}[t]{0.97\columnwidth}%
\begin{lstlisting}[
    basicstyle=\scriptsize\ttfamily\color{col3},  
   tabsize=2,
   frame=single,
   rulecolor=\color{gray!100},
   backgroundcolor=\color{gray!03},
    breaklines=true,
    xleftmargin=0pt,
    xrightmargin=0pt,
    breakindent=2\baselineskip,
]
data { ... }
parameters {
	real initial_popularity;
	real<lower=0> step_size;
	array[num_days-1] real steps;
	vector<lower=0>[num_pollsters] pollster_noise;
	real<lower=0> pollster_noise_mean;
	real<lower=0> pollster_noise_sigma;
}
transformed parameters{
	array[num_days] real popularity;
	popularity[1] = initial_popularity;
	for (i in 2:num_days){
		popularity[i] = popularity[i-1] + steps[i-1];
	}
}
model {
	initial_popularity ~ normal(50,20); // very vague prior over initial popularity
	step_size ~ normal(0,5);
	for (i in 1:(num_days-1)){
		steps[i] ~ normal(0,step_size); // random walk for the popularity
	}
	pollster_noise_mean ~ normal(5,10); // very vague prior over the mean of the noise
	pollster_noise_sigma ~ cauchy(0,2); // very vague prior over the sigma of the noise
	pollster_noise ~ normal(pollster_noise_mean, pollster_noise_sigma); // hierarchical prior for the noise of each pollster
	for (i in 1:num_pollsters){
		for (j in 1:num_polls[i]){
			polls[i,j] ~ normal(popularity[day[i,j]], pollster_noise[i]);
		}
	}
}
\end{lstlisting}%
\end{minipage}}{\scriptsize\par}
\par\end{centering}
{\scriptsize}{\scriptsize\par}}}
\par\end{center}%
\end{minipage}

\caption{Example models for the \textbf{polling} problem, sorted by estimated
marginal likelihoods. The \texttt{data} block was always identical,
so is truncated for space for some models. \label{fig:polling-models}}
\end{figure}
\par\end{flushleft}

\begin{flushleft}
\begin{figure}[H]
\begin{centering}
\includegraphics[width=0.5\textwidth]{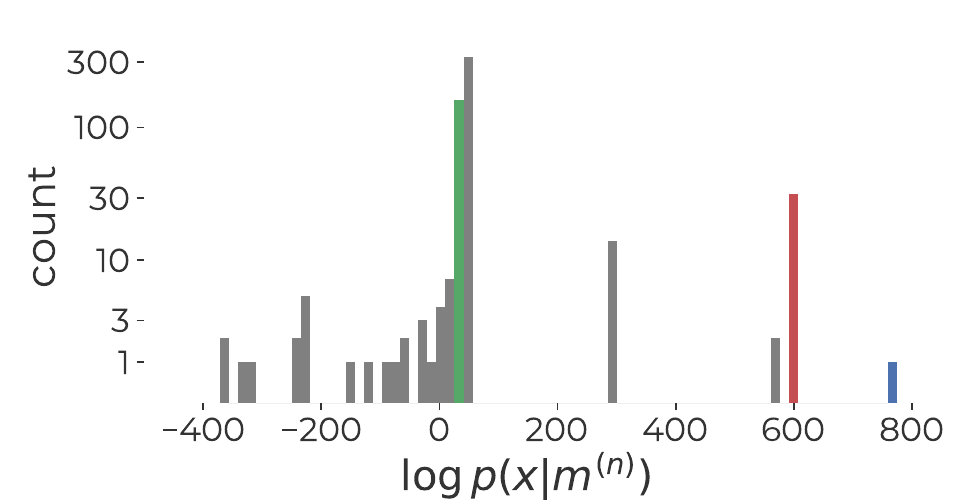}
\par\end{centering}
\caption{Estimated marginal likelihoods $p\protect\pp{x\vert m^{(n)}}$ for
each of the valid models $n$ for the \textbf{polling} problem. Bars
are colored for four example models from \ref{fig:polling-models}.
\label{fig:polling-inference}}
\end{figure}
\par\end{flushleft}

\begin{flushleft}
\begin{figure}[H]
\begin{centering}
\includegraphics[width=0.5\textwidth]{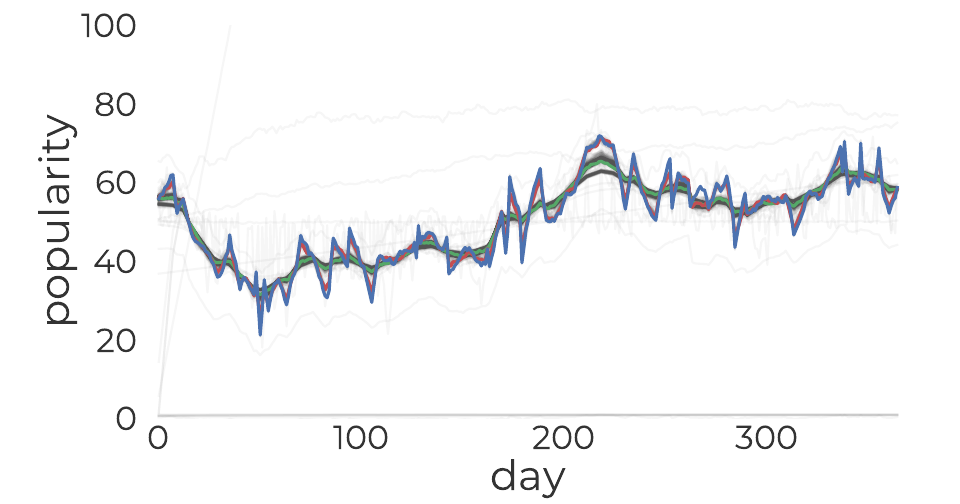}
\par\end{centering}
\caption{The estimated posteriors for each of the valid models $n$ for the
\textbf{polling} problem. Posteriors are colored for the four example
models from \ref{fig:polling-models}. \label{fig:polling-posteriors}}
\end{figure}
\par\end{flushleft}

\begin{flushleft}
\begin{figure}[H]
\begin{centering}
\includegraphics[width=0.5\columnwidth]{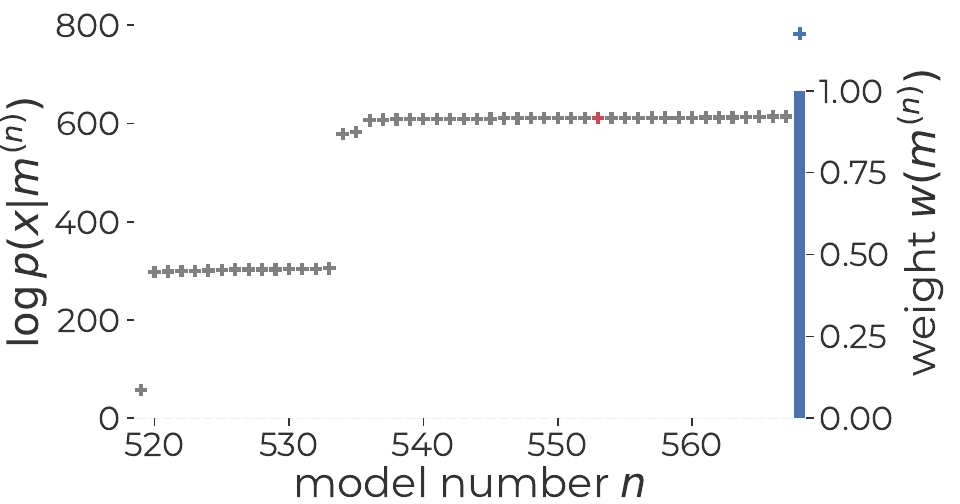}\includegraphics[width=0.5\textwidth]{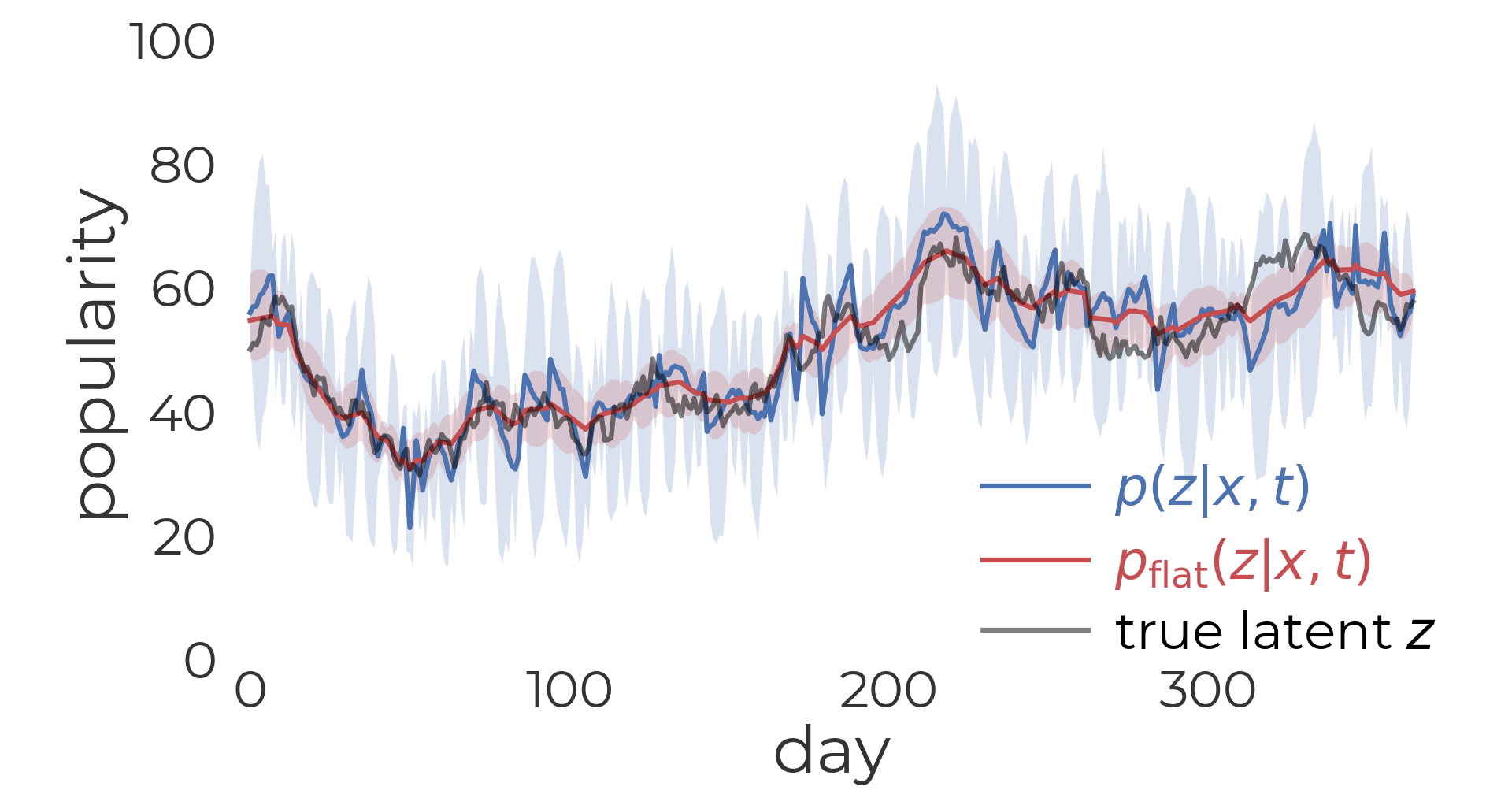}
\par\end{centering}
\caption{Left: Estimated log-marginal likelihoods $\log p\protect\pp{x\vert m^{(n)}}$
(\textquotedblleft +\textquotedblright{} symbols, left axis) and weights
$w(m^{(n)})$ (bars, right axis) for the top 50 models $n$ for the
\textbf{polling} problem. Right: The final estimated posteriors, compared
to a \textquotedblleft flat\textquotedblright{} average.\label{fig:polling-bma}}
\end{figure}
\par\end{flushleft}

\cleardoublepage{}

\newpage{}

\subsection{City temperature\label{subsec:City-temperature}}

In this problem the user describes observing temperatures on neighboring
pairs of days for a set of different cities. Then, after observing
temperatures on a set of ``test'' days for each city, they wish
to predict the temperature on the following days.

\begin{lstlisting}[
    basicstyle=\scriptsize\ttfamily,     % Typewriter font (monospaced)
   frame=single,
   rulecolor=\color{gray!100},
   backgroundcolor=\color{gray!03},
    breaklines=true,
    xleftmargin=5pt,
    xrightmargin=5pt,
    breakindent=0\baselineskip,
]
PROBLEM

There is training data for a bunch of different cities on random days, along with the temperature for the following data. Then, there is test data for a bunch of different days in those same cities. Predict the temperature on the following days.

Assume that different cities tend to have different weather. And assume that the weather tends to be similar in neighboring days. So when guessing the weather in the next day, both the typical temperature in the city and the weather on the previous day will matter.

All temperatures are in Fahrenheit.

Be sure to use informative priors for all quantities.

DATA

int num_cities; // number of cities included
int num_days_train; // number of training days in each city
int num_days_test; // number of training days in each city
array[num_cities, num_days_train] real day1_temp_train; // temperature on the first day, training set
array[num_cities, num_days_train] real day2_temp_train; // temperature on the following day, training set
array[num_cities, num_days_test] real day1_temp_test; // temperature on the first day, test set

GOAL

array[num_cities, num_days_test] real day2_temp_test; // temperature on the following day, test set
\end{lstlisting}

In the given data, there are 5 cities, with 10 training day pairs,
and 10 test days.

\begin{lstlisting}[
    basicstyle=\scriptsize\ttfamily,
   tabsize=2,
   frame=single,
   rulecolor=\color{gray!100},
   backgroundcolor=\color{gray!03},
    breaklines=true,
    xleftmargin=5pt,
    xrightmargin=5pt,
    breakindent=2\baselineskip,
]
{
    "num_cities": 5,
    "num_days_train": 10,
    "num_days_test": 10,
    "day1_temp_train":
    [[61.0, 74.9, 61.5, 73.2, 78.1, 69.8, 55.2, 74.2, 53.9, 75.0],
     [73.0, 78.5, 80.2, 78.9, 82.0, 84.2, 79.0, 81.3, 77.3, 76.0],
     [79.3, 82.1, 82.2, 76.0, 81.2, 84.9, 79.0, 79.0, 81.1, 85.7],
     [81.0, 76.5, 82.0, 78.8, 77.6, 77.1, 76.9, 82.2, 83.7, 78.5],
     [79.8, 78.8, 77.8, 79.6, 76.8, 75.4, 80.4, 75.6, 81.0, 76.2]],
    "day1_temp_test":
    [[69.3, 77.5, 74.8, 65.3, 58.5, 87.0, 54.9, 48.5, 59.4, 72.9],
     [76.5, 48.2, 84.1, 66.6, 85.1, 81.9, 81.0, 77.4, 71.9, 77.7],
     [80.6, 78.9, 79.4, 84.3, 80.5, 81.7, 80.6, 82.7, 82.2, 79.7],
     [82.7, 74.7, 82.1, 77.9, 82.5, 76.2, 81.8, 76.5, 79.1, 82.1],
     [78.4, 75.1, 81.0, 75.0, 74.5, 74.2, 80.5, 76.5, 79.0, 77.8]],
    "day2_temp_train":
    [[61.5, 69.9, 62.0, 73.4, 78.8, 68.8, 53.1, 74.5, 49.4, 76.2],
     [76.4, 78.0, 82.4, 78.9, 80.2, 82.8, 79.4, 79.1, 78.3, 80.0],
     [79.5, 83.8, 77.9, 76.8, 78.5, 85.2, 79.8, 81.6, 80.0, 85.5],
     [80.5, 77.1, 86.2, 78.2, 75.5, 79.1, 78.9, 81.8, 84.2, 82.7],
     [80.8, 78.5, 75.7, 76.1, 74.0, 80.3, 81.6, 77.8, 79.5, 72.8]]
}
\end{lstlisting}

Out of 1024 models generated, 736 (71.9\%) compiled and allowed for
inference. Four example models are shown in \eqref{fig:city-models-A}
and \eqref{fig:city-models-B}. A histogram of estimated marginal
likelihoods is in \eqref{fig:city-inference}, while \ref{fig:city-weights}
shows the estimated marginal likelihoods and corresponding weights
for the top-scoring 50 valid models. Finally, \ref{fig:city-bma}compares
the final estimated posterior $p\pp{z\vert x,t}$ to a ``flat''
average $p_{\mathrm{flat}}\pp{z\vert x,t}$.

In this case, the top scoring model {\color{col1}$m^{(736)}$} \ref{fig:city-top-model}
models per-city mean temperatures and shared regression coefficients
to predict test temperatures. Many generated models used more shared
parameters. For example, {\color{col3}$m^{(688)}$} \eqref{fig:city-model-lower}
has global parameters for the distribution of temperatures and {\color{col4}$m^{(487)}$}
\eqref{fig:city-model-even-lower} also has a shared autocorrelation
parameter. However, these more complex models had lower marginal likelihoods.
Thus the final posterior is essentially the same as that of {\color{col1}$m^{(736)}$}.
\begin{flushleft}
\begin{figure}
\begin{centering}
{\scriptsize\subfloat[{\color{col1}$m^{(736)}$}, the highest-scoring model.\label{fig:city-top-model}]{\begin{centering}
{\scriptsize{}%
\begin{minipage}[t]{0.97\columnwidth}%
\begin{lstlisting}[
    basicstyle=\scriptsize\ttfamily\color{col1},  
   tabsize=2,
   frame=single,
   rulecolor=\color{gray!100},
   backgroundcolor=\color{gray!03},
    breaklines=true,
    xleftmargin=0pt,
    xrightmargin=0pt,
    breakindent=2\baselineskip,
]
data {
	int num_cities;
	int num_days_train;
	int num_days_test;
	array[num_cities, num_days_train] real day1_temp_train;
	array[num_cities, num_days_train] real day2_temp_train;
	array[num_cities, num_days_test] real day1_temp_test;
}
parameters {
	array[num_cities] real city_mean;
	array[num_cities, num_days_train] real day1_temp_adj;
	real beta;
	real alpha;
}
model {
	beta ~ normal(0.8, 0.2);
	alpha ~ normal(0, 5);
	for (i in 1:num_cities){
		city_mean[i] ~ normal(50, 20);
		for (j in 1:num_days_train){
			day1_temp_adj[i, j] ~ normal(0, 10);
			day2_temp_train[i, j] ~ normal(city_mean[i] + alpha + beta * (day1_temp_train[i, j] - city_mean[i] + day1_temp_adj[i, j]), 5);
		}
	}
}
generated quantities {
	array[num_cities, num_days_test] real day2_temp_test;
	for (i in 1:num_cities){
		for (j in 1:num_days_test){
			day2_temp_test[i, j] = normal_rng(city_mean[i] + alpha + beta * (day1_temp_test[i, j] - city_mean[i]), 5);
		}
	}
}
\end{lstlisting}
\vspace{-0.25cm}%
\end{minipage}}{\scriptsize\par}
\par\end{centering}
{\scriptsize}{\scriptsize\par}}}{\scriptsize\vspace{-0.25cm}}{\scriptsize\par}
\par\end{centering}
\begin{centering}
{\scriptsize\subfloat[{\color{col2}$m^{(735)}$}, the second-highest scoring model]{\begin{centering}
{\scriptsize{}%
\begin{minipage}[t]{0.97\columnwidth}%
\begin{lstlisting}[
    basicstyle=\scriptsize\ttfamily\color{col2},  
   tabsize=2,
   frame=single,
   rulecolor=\color{gray!100},
   backgroundcolor=\color{gray!03},
    breaklines=true,
    xleftmargin=0pt,
    xrightmargin=0pt,
    breakindent=2\baselineskip,
]
data {
	int num_cities;
	int num_days_train;
	int num_days_test;
	array[num_cities, num_days_train] real day1_temp_train;
	array[num_cities, num_days_train] real day2_temp_train;
	array[num_cities, num_days_test] real day1_temp_test;
}
parameters {
	array[num_cities] real city_mean;
	array[num_cities, num_days_train] real day_effect;
	real<lower=-1, upper=1> temp_change;
	real<lower=0> temp_sd;
}
model {
	for(i in 1:num_cities){
		city_mean[i] ~ normal(50,20);
		for(j in 1:num_days_train){
			day_effect[i,j] ~ normal(0,10);
		}
	}
	temp_change ~ uniform(-1, 1);
	temp_sd ~ cauchy(0, 5);
	for(i in 1:num_cities){
		for(j in 1:num_days_train){
			day2_temp_train[i,j] ~ normal(city_mean[i] + day_effect[i,j] + temp_change * (day1_temp_train[i,j] - city_mean[i] - day_effect[i,j]), temp_sd);
		}
	}
}
generated quantities {
	array[num_cities, num_days_test] real day2_temp_test;
	for(i in 1:num_cities){
		for(j in 1:num_days_test){
			day2_temp_test[i,j] = normal_rng(city_mean[i] + temp_change * (day1_temp_test[i,j] - city_mean[i]), temp_sd);
		}
	}
}
\end{lstlisting}
\vspace{-0.25cm}%
\end{minipage}}{\scriptsize\par}
\par\end{centering}
{\scriptsize}{\scriptsize\par}}}{\scriptsize\par}
\par\end{centering}
\caption{Example models for the \textbf{city temperature} problem. \label{fig:city-models-A}}
\end{figure}
\par\end{flushleft}

\begin{flushleft}
\begin{figure}
\begin{centering}
{\scriptsize\subfloat[{\color{col3}$m^{(688)}$}, a lower-scoring model.\label{fig:city-model-lower}]{\begin{centering}
{\scriptsize{}%
\begin{minipage}[t]{0.97\columnwidth}%
\begin{lstlisting}[
    basicstyle=\scriptsize\ttfamily\color{col3},  
   tabsize=2,
   frame=single,
   rulecolor=\color{gray!100},
   backgroundcolor=\color{gray!03},
    breaklines=true,
    xleftmargin=0pt,
    xrightmargin=0pt,
    breakindent=2\baselineskip,
]
data {
	int num_cities;
	int num_days_train;
	int num_days_test;
	array[num_cities, num_days_train] real day1_temp_train;
	array[num_cities, num_days_train] real day2_temp_train;
	array[num_cities, num_days_test] real day1_temp_test;
}
parameters {
	real mu;
	real<lower=0> sigma;
	array[num_cities] real city_mu;
	array[num_cities] real<lower=-1, upper=1> city_ar;
}
model {
	mu ~ normal(50, 20);
	sigma ~ normal(0, 20);
	for (c in 1:num_cities){
		city_mu[c] ~ normal(mu, sigma);
		city_ar[c] ~ uniform(-1, 1);
		for (d in 1:num_days_train){
			day2_temp_train[c, d] ~ normal(city_mu[c] + city_ar[c] * (day1_temp_train[c, d] - city_mu[c]), 5);
		}
	}
}
generated quantities {
	array[num_cities, num_days_test] real day2_temp_test;
	for (c in 1:num_cities){
		for (d in 1:num_days_test){
			day2_temp_test[c, d] = normal_rng(city_mu[c] + city_ar[c] * (day1_temp_test[c, d] - city_mu[c]), 5);
		}
	}
}
\end{lstlisting}%
\end{minipage}}{\scriptsize\par}
\par\end{centering}
{\scriptsize}{\scriptsize\par}}}{\scriptsize\par}
\par\end{centering}
\begin{centering}
{\scriptsize\subfloat[{\color{col4}$m^{(487)}$}, a typical but lower-scoring model. \label{fig:city-model-even-lower}]{\begin{centering}
{\scriptsize{}%
\begin{minipage}[t]{0.97\columnwidth}%
\begin{lstlisting}[
    basicstyle=\scriptsize\ttfamily\color{col4},  
   tabsize=2,
   frame=single,
   rulecolor=\color{gray!100},
   backgroundcolor=\color{gray!03},
    breaklines=true,
    xleftmargin=0pt,
    xrightmargin=0pt,
    breakindent=2\baselineskip,
]
data {
	int num_cities;
	int num_days_train;
	int num_days_test;
	array[num_cities, num_days_train] real day1_temp_train;
	array[num_cities, num_days_train] real day2_temp_train;
	array[num_cities, num_days_test] real day1_temp_test;
}
parameters {
	vector[num_cities] city_mean;
	vector<lower=0, upper=1>[num_cities] city_autocorrelation;
	real global_mean;
	real<lower=0, upper=1> global_autocorrelation;
	real<lower=0> sigma;
}
model {
	global_mean ~ normal(60, 10);
	global_autocorrelation ~ uniform(0, 1);
	for (i in 1:num_cities){
		city_mean[i] ~ normal(global_mean, 10);
		city_autocorrelation[i] ~ uniform(0, 1);
		for (j in 1:num_days_train){
			day2_temp_train[i, j] ~ normal(day1_temp_train[i, j] * city_autocorrelation[i] + city_mean[i] * (1 - city_autocorrelation[i]), sigma);
		}
	}
	sigma ~ cauchy(0, 5);
}
generated quantities {
	array[num_cities, num_days_test] real day2_temp_test;
	for (i in 1:num_cities){
		for (j in 1:num_days_test){
			day2_temp_test[i, j] = normal_rng(day1_temp_test[i, j] * city_autocorrelation[i] + city_mean[i] * (1 - city_autocorrelation[i]), sigma);
		}
	}
}
\end{lstlisting}%
\end{minipage}}{\scriptsize\par}
\par\end{centering}
{\scriptsize}{\scriptsize\par}}}{\scriptsize\par}
\par\end{centering}
\caption{Example models for the \textbf{city temperature} problem. \label{fig:city-models-B}}
\end{figure}
\cleardoublepage{}
\par\end{flushleft}

\begin{flushleft}
\begin{figure}[H]
\begin{centering}
\includegraphics[width=0.5\textwidth]{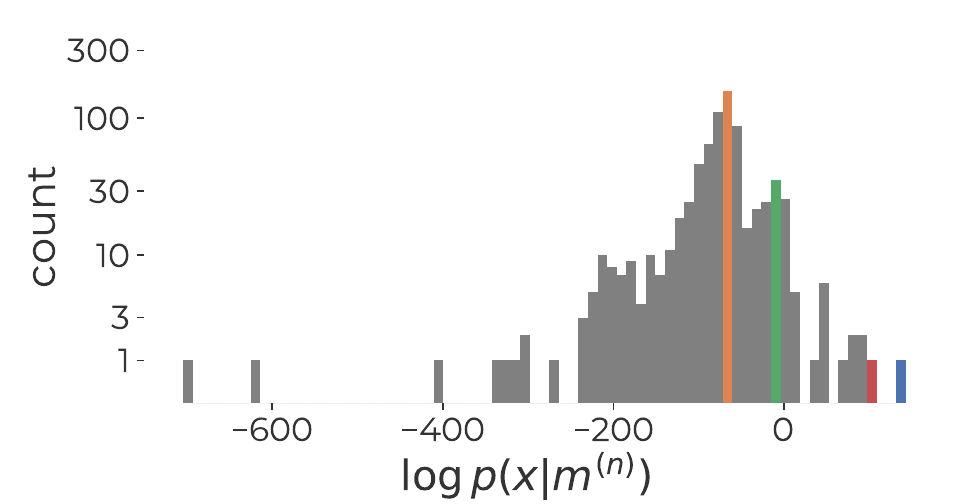}
\par\end{centering}
\caption{Estimated marginal likelihoods $p\protect\pp{x\vert m^{(n)}}$ for
each of the valid models $n$ for the \textbf{city temperature} problem.
Bars are colored for four example models from \eqref{fig:city-models-A}
and \eqref{fig:city-models-B}. \label{fig:city-inference}}
\end{figure}
\par\end{flushleft}

\begin{flushleft}
\begin{figure}[H]
\begin{centering}
\vspace{-1cm}\includegraphics[width=0.5\textwidth]{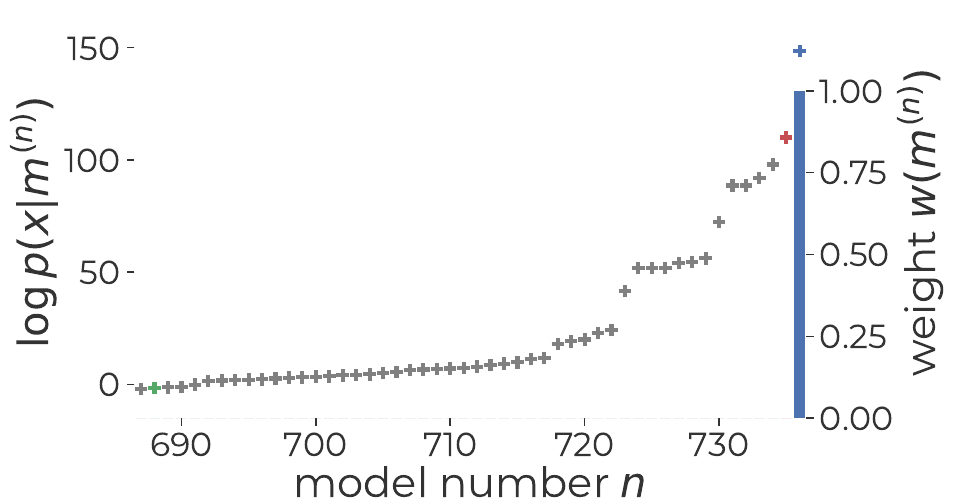}
\par\end{centering}
\caption{Estimated log-marginal likelihoods $p\protect\pp{x\vert m^{(n)}}$
and weights $w(m^{(n)})$ for the top 50 models $n$ for the \textbf{city
temperature} problem. markers are colored for four example models
from \ref{fig:city-models-A} and \ref{fig:city-models-B}. \label{fig:city-weights}}
\end{figure}
\par\end{flushleft}

\begin{figure}[H]
\begin{centering}
\vspace{-1cm}\includegraphics[width=0.85\columnwidth]{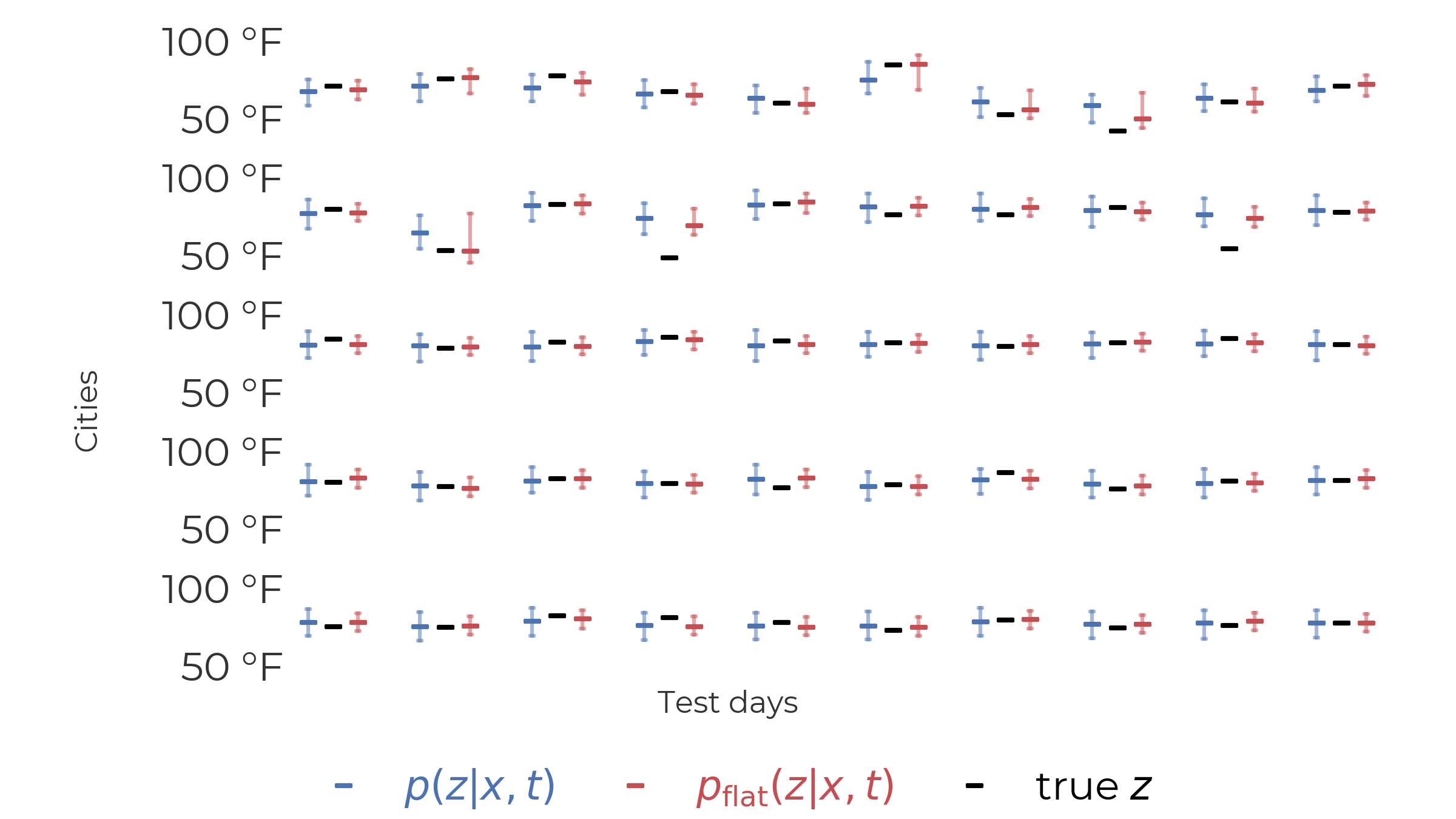}
\par\end{centering}
\caption{Medians and 90\% credible intervals computed for each city on each
test day for the \textbf{city temperature} problem. In this case,
both LLB and $p_{\mathrm{flat}}$ have coverage of 44/50=88\% of the
true values, indicating reasonable calibration, though the posteriors
under LLB are slightly more narrow.\label{fig:city-bma}}
\end{figure}

\cleardoublepage{}

\newpage{}

\subsection{Gold (small)\label{subsec:Gold-(small)}}

In this problem, the user describes a ``rod'' of some length, with
a varying density of gold atoms. Samples were taken at various positions
and tested to be gold or not. Finally, a set of future test locations
are given. The user ``forgot'' to give types for some variables.

\begin{lstlisting}[
    basicstyle=\scriptsize\ttfamily,     % Typewriter font (monospaced)
   frame=single,
   rulecolor=\color{gray!100},
   backgroundcolor=\color{gray!03},
    breaklines=true,
    xleftmargin=5pt,
    xrightmargin=5pt,
    breakindent=0\baselineskip,
]
PROBLEM

I have a metal rod. I've gone to various positions along the rod and sampled single atoms and tested if they were gold. (Yes or no.) I'd like to predict if gold would be detected if I tried measurements at various test locations. I assume the density of gold changes smoothly, but the actual density could be quite complicated.

DATA

real rod_length; // length of rod in meters
int num_train; // number of times rod was sampled
array[num_train] train_locs; // where samples taken in training data
array[num_train] gold_train; // was gold detected at each location (0 if no, 1 if yes)
int num_test; // hypothetical test locations
array[num_test] test_locs; // possible future locations for sampling

GOAL

array[num_test] gold_test; // would gold be detected at future locations
\end{lstlisting}

Synthetic data was generated by taking training locations $x$ uniformly
from 0 to 2.5. The true density of gold was $\sigma\pars{1+\sin(9x)-\frac{1}{2}x^{2}}$
where $\sigma\pp s=1/\pp{1+e^{-s}}$ denotes a sigmoid function. The
test locations were 100 points equally spaced from 0 to 2.5. A sample
of the data is shown below, while the full data is plotted in \ref{fig:gold-data}.
\begin{center}
{\scriptsize{\ttfamily{}%
\noindent\begin{minipage}[t]{1\columnwidth}%
\begin{lstlisting}[
    basicstyle=\scriptsize\ttfamily,
   tabsize=2,
   frame=single,
   rulecolor=\color{gray!100},
   backgroundcolor=\color{gray!03},
    breaklines=true,
    xleftmargin=0pt,
    xrightmargin=0pt,
    breakindent=2\baselineskip,
]
{
    "rod_length": 2.5,
    "num_train": 30,
    "train_locs": [2.000547344419736, 0.9166378396087462, 0.16507930158778522, ..., 1.5834750504530812],
    "gold_train": [0, 1, 1, ..., 1],
    "num_test": 100,
    "test_locs": [0.0, 0.025252525252525252, 0.050505050505050504, ..., 2.5],
}
\end{lstlisting}%
\end{minipage}}}{\scriptsize\par}
\par\end{center}

\begin{center}
\begin{figure}[H]
\begin{centering}
\includegraphics[width=0.5\textwidth]{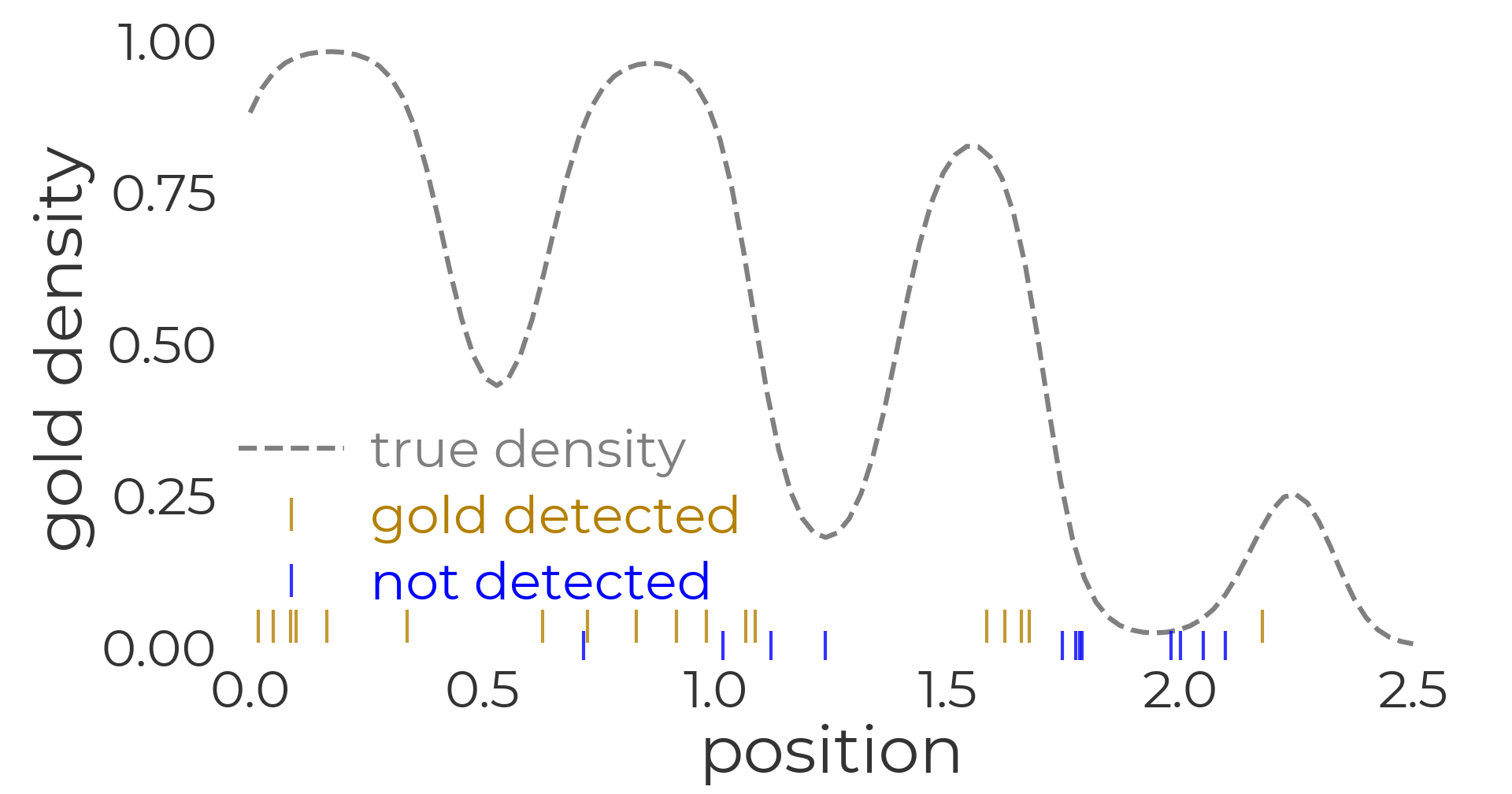}
\par\end{centering}
\caption{Observed data for the \textbf{gold (small)} problem.\label{fig:gold-data}}
\end{figure}
\par\end{center}

In preliminary testing, it was difficult for the LLM to generate valid
models for this problem. Thus, instead of 1000, $2^{14}=16384$ candidate
models were generated. Of these, 192 (1.2\%) compiled and allowed
for inference. The top three models and a low scoring model are shown
in \ref{fig:gold-small-top}, \ref{fig:gold-small-2nd}, \ref{fig:gold-small-3rd},
and \ref{fig:gold-small-typical}. Histograms of estimated marginal
likelihoods are in \ref{fig:gold-small-inference}, while \ref{fig:gold-small-posteriors}
shows the estimated posterior for each model, and \ref{fig:gold-small-bma}
shows the weights for the top 50 models and the final estimated posterior.

In this case, essentially all posterior weight is given to the top-three
scoring models, {\color{col1}$m^{(192)}$} \ref{fig:gold-small-top},
{\color{col2}$m^{(191)}$} \ref{fig:gold-small-2nd}, and {\color{col1}$m^{(190)}$}
\ref{fig:gold-small-3rd}. The final posterior appears a bit better
than a flat average.

\cleardoublepage{}
\begin{flushleft}
{\scriptsize{}
\begin{figure}[H]
\begin{centering}
{\scriptsize{}%
\begin{minipage}[t]{0.97\columnwidth}%
\begin{lstlisting}[
    basicstyle=\scriptsize\ttfamily\color{col1},    
   tabsize=2,
   frame=single,
   rulecolor=\color{gray!100},
   backgroundcolor=\color{gray!03},
    breaklines=true,
    xleftmargin=0pt,
    xrightmargin=0pt,
    breakindent=2\baselineskip,
]
data {
	real rod_length; 
	int num_train; 
	array[num_train] real train_locs; 
	array[num_train] int<lower=0, upper=1> gold_train; 
	int num_test; 
	array[num_test] real test_locs; 
}
transformed data {
	int<lower=1> N = num_train + num_test;
	array[N] real all_locs = append_array(train_locs, test_locs);
}
parameters {
	vector[N] eta;
	real<lower=0> rho;
}
transformed parameters {
	vector[N] f;
	{
		matrix[N, N] K;
		for (i in 1:N) {
			K[i, i] = 1 + 0.01; // diagonal element
			for (j in 1:(i-1)) {
				real dist = sqrt(square(all_locs[i] - all_locs[j]));
				K[i, j] = sqrt(5)/sqrt(3*rho)*pow(1+sqrt(5*dist/rho), -1.5 - 1/2)*exp(-sqrt(5*dist/rho/3)*1.2);
				K[j, i] = K[i, j]; // fill in lower triangle
			}
		}
		f = K * eta;
	}
}
model {
	eta ~ normal(0, 1);
	rho ~ normal(0, rod_length); // use rod length to specify the input scale of the function
	
	for (i in 1:num_train) {
		gold_train[i] ~ bernoulli_logit(f[i]);
	}
}
generated quantities {
	array[num_test] int<lower=0, upper=1> gold_test;
	for (i in 1:num_test) {
		gold_test[i] = bernoulli_logit_rng(f[num_train + i]);
	}
}
\end{lstlisting}%
\end{minipage}}{\scriptsize\par}
\par\end{centering}
{\scriptsize\caption{{\color{col1}$m^{(192)}$}, the highest-scoring model for the \textbf{gold(small)}
problem.\label{fig:gold-small-top}}
}{\scriptsize\par}
\end{figure}
}{\scriptsize\par}
\par\end{flushleft}

\begin{flushleft}
{\scriptsize{}
\begin{figure}[H]
\begin{centering}
{\scriptsize{}%
\begin{minipage}[t]{0.97\columnwidth}%
\begin{lstlisting}[
    basicstyle=\scriptsize\ttfamily\color{col2},    
   tabsize=2,
   frame=single,
   rulecolor=\color{gray!100},
   backgroundcolor=\color{gray!03},
    breaklines=true,
    xleftmargin=0pt,
    xrightmargin=0pt,
    breakindent=2\baselineskip,
]
data {
	real rod_length;
	int num_train;
	array[num_train] real train_locs;
	array[num_train] int<lower=0, upper=1> gold_train;
	int num_test;
	array[num_test] real test_locs;
}
transformed data {
	// creating a grid of locations to compute covariance matrix
	int num_latent = num_train + num_test;
	array[num_latent] real all_locs;
	for(i in 1:num_train){
		all_locs[i] = train_locs[i];
	}
	for(i in 1:num_test){
		all_locs[i+num_train] = test_locs[i];
	}
}
parameters {
	vector[num_latent] z;
}
transformed parameters {
	vector[num_latent] f;
	{
		matrix[num_latent, num_latent] K;
		for (i in 1:num_latent) {
			for (j in 1:num_latent) {
				real dist = sqrt((all_locs[i] - all_locs[j])^2);
				K[i, j] = (1 + sqrt(3 * dist)) * exp(-sqrt(3 * dist));
			}
		}
		f = rep_vector(0.5, num_latent) + cholesky_decompose(K) * z;
	}
}
model {
	z ~ std_normal();
	for(i in 1:num_train){
		gold_train[i] ~ bernoulli(inv_logit(f[i]));
	}
}
generated quantities {
	array[num_test] int<lower=0, upper=1> gold_test;
	for(i in 1:num_test){
		real true_gold_density = inv_logit(f[i+num_train]);
		gold_test[i] = bernoulli_rng(true_gold_density);
	}
}
\end{lstlisting}%
\end{minipage}}{\scriptsize\par}
\par\end{centering}
{\scriptsize\caption{{\color{col2}$m^{(191)}$}, the second-highest-scoring model for
the \textbf{gold(small)} problem.\label{fig:gold-small-2nd}}
}{\scriptsize\par}
\end{figure}
}{\scriptsize\par}
\par\end{flushleft}

\begin{flushleft}
{\scriptsize{}
\begin{figure}[H]
\begin{centering}
{\scriptsize{}%
\begin{minipage}[t]{0.97\columnwidth}%
\begin{lstlisting}[
    basicstyle=\scriptsize\ttfamily\color{col3},    
   tabsize=2,
   frame=single,
   rulecolor=\color{gray!100},
   backgroundcolor=\color{gray!03},
    breaklines=true,
    xleftmargin=0pt,
    xrightmargin=0pt,
    breakindent=2\baselineskip,
]
data {
	real rod_length;
	int num_train;
	array[num_train] real train_locs;
	array[num_train] int<lower=0, upper=1> gold_train;
	int num_test;
	array[num_test] real test_locs;
}
transformed data {
	int num_locs = num_train + num_test;
	array[num_locs] real all_locs;
	for(i in 1:num_train){
		all_locs[i] = train_locs[i];
	}
	for(i in 1:num_test){
		all_locs[i+num_train] = test_locs[i];
	}
}
parameters {
	real<lower=0> length;
	vector[num_locs] z;
}
transformed parameters {
	vector[num_locs] p;
	{
		matrix[num_locs, num_locs] K;
		for (i in 1:num_locs){
			for (j in 1:num_locs){
				K[i,j] = exp(-((all_locs[i]-all_locs[j])/length)^2);
			}
		}
		p = K * z;
	}
}
model {
	z ~ normal(0,1);
	length ~ normal(0.1,0.1);
	for (i in 1:num_train){
		gold_train[i] ~ bernoulli(1/(1+exp(-p[i])));
	}
}
generated quantities {
	array[num_test] int<lower=0, upper=1> gold_test;
	vector[num_locs] p_smooth;
	p_smooth = inv_logit(p);
	for (i in 1:num_test){
		gold_test[i] = bernoulli_rng(p_smooth[i+num_train]);
	}
}
\end{lstlisting}%
\end{minipage}}{\scriptsize\par}
\par\end{centering}
{\scriptsize\caption{{\color{col3}$m^{(190)}$}, the third-highest-scoring model for
the \textbf{gold(small)} problem.\label{fig:gold-small-3rd}}
}{\scriptsize\par}
\end{figure}
}{\scriptsize\par}
\par\end{flushleft}

\begin{flushleft}
{\scriptsize{}
\begin{figure}[H]
\begin{centering}
{\scriptsize{}%
\begin{minipage}[t]{0.97\columnwidth}%
\begin{lstlisting}[
    basicstyle=\scriptsize\ttfamily\color{col4},    
   tabsize=2,
   frame=single,
   rulecolor=\color{gray!100},
   backgroundcolor=\color{gray!03},
    breaklines=true,
    xleftmargin=0pt,
    xrightmargin=0pt,
    breakindent=2\baselineskip,
]
data {
	real rod_length;
	int num_train;
	array[num_train] real train_locs;
	array[num_train] int<lower=0, upper=1> gold_train;
	int num_test;
	array[num_test] real test_locs;
}
parameters {
	real<lower=0> rho; // length scale for the kernel
	real<lower=0> alpha; // scale for the kernel
	vector[num_train] eta; // latent variables
}
transformed parameters {
	vector[num_train] mu;
	{
		matrix[num_train, num_train] K;
		for (i in 1:num_train) {
			K[i, i] = 1 + alpha;
			for (j in 1:num_train) {
				K[i, j] = alpha * exp(-((train_locs[i] - train_locs[j])^2) / (2 * rho^2));
			}
		}
		mu = K * eta;
	}
}
model {
	rho ~ cauchy(0, 5);
	alpha ~ cauchy(0, 5);
	eta ~ normal(0, 1);
	gold_train ~ bernoulli_logit(mu);
}
generated quantities {
	vector[num_test] test_mu;
	{
		matrix[num_test, num_train] K_test_train;
		for (i in 1:num_test) {
			for (j in 1:num_train) {
				K_test_train[i, j] = alpha * exp(-((test_locs[i] - train_locs[j])^2) / (2 * rho^2));
			}
		}
		test_mu = (K_test_train * eta);
	}
	array[num_test] int<lower=0, upper=1> gold_test;
	for (i in 1:num_test) {
		gold_test[i] = bernoulli_logit_rng(test_mu[i]);
	}
}
\end{lstlisting}%
\end{minipage}}{\scriptsize\par}
\par\end{centering}
{\scriptsize\caption{{\color{col4}$m^{(171)}$}, a low-scoring model for the \textbf{gold(small)}
problem.\label{fig:gold-small-typical}}
}{\scriptsize\par}
\end{figure}
}{\scriptsize\par}
\par\end{flushleft}

\cleardoublepage{}
\begin{flushleft}
\begin{figure}[H]
\begin{centering}
\includegraphics[width=0.5\textwidth]{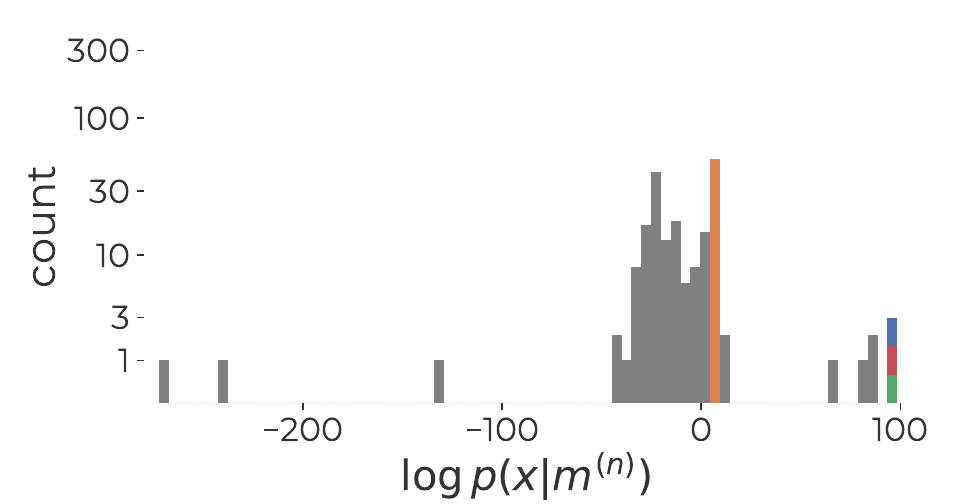}\includegraphics[width=0.5\textwidth]{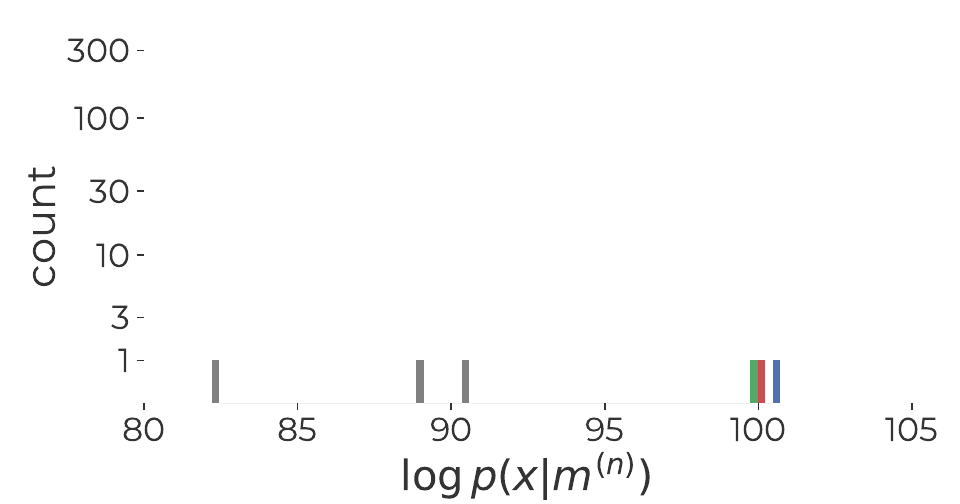}
\par\end{centering}
\caption{Estimated marginal likelihoods $p\protect\pp{x\vert m^{(n)}}$ for
each of the valid models $n$ for the \textbf{gold (small)} problem.
Because of the many order of magnitude, two different ranges are shown.
Bars are colored for four example models from \ref{fig:gold-small-top},
\ref{fig:gold-small-2nd}, \ref{fig:gold-small-3rd}, and \ref{fig:gold-small-typical}.
If multiple models map to the same bin, all colors are shown stacked.
\label{fig:gold-small-inference}}
\end{figure}
\par\end{flushleft}

\begin{flushleft}
\begin{figure}[H]
\begin{centering}
\includegraphics[width=0.5\columnwidth]{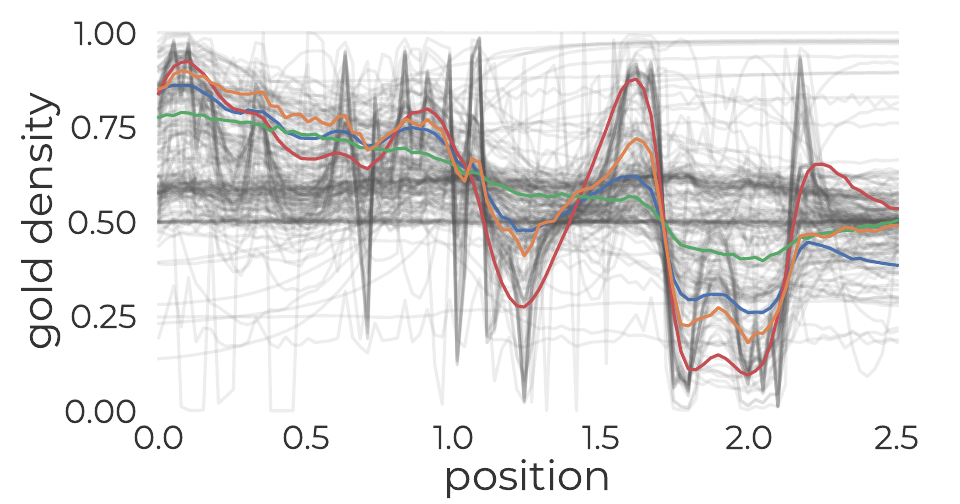}
\par\end{centering}
\caption{The estimated posteriors $\protect\E[z\vert m^{(n)},x]$ for each
of the valid models $n$ for the \textbf{gold (small)} problem. Lines
are colored for four example models from \ref{fig:gold-small-top},
\ref{fig:gold-small-2nd}, \ref{fig:gold-small-3rd}, and \ref{fig:gold-small-typical}.\label{fig:gold-small-posteriors}}
\end{figure}
\par\end{flushleft}

\begin{flushleft}
\begin{figure}[H]
\begin{centering}
\includegraphics[width=0.5\columnwidth]{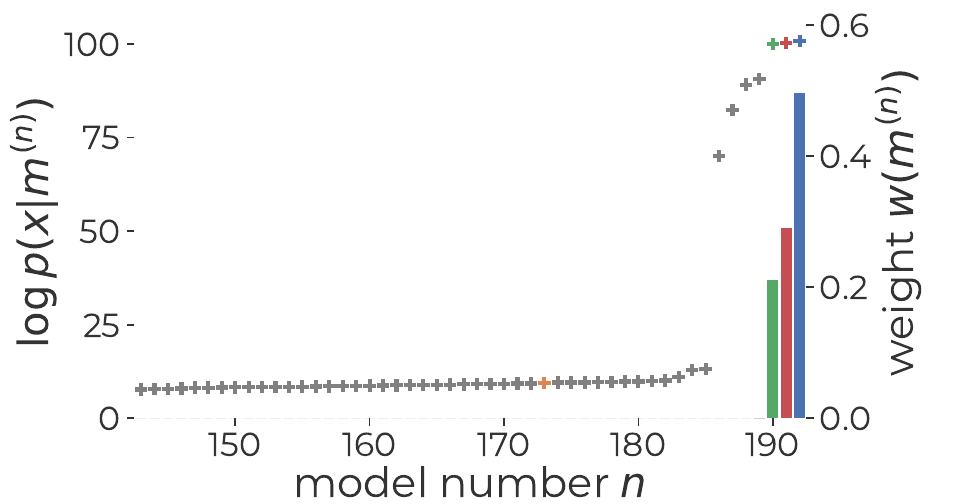}\includegraphics[width=0.5\textwidth]{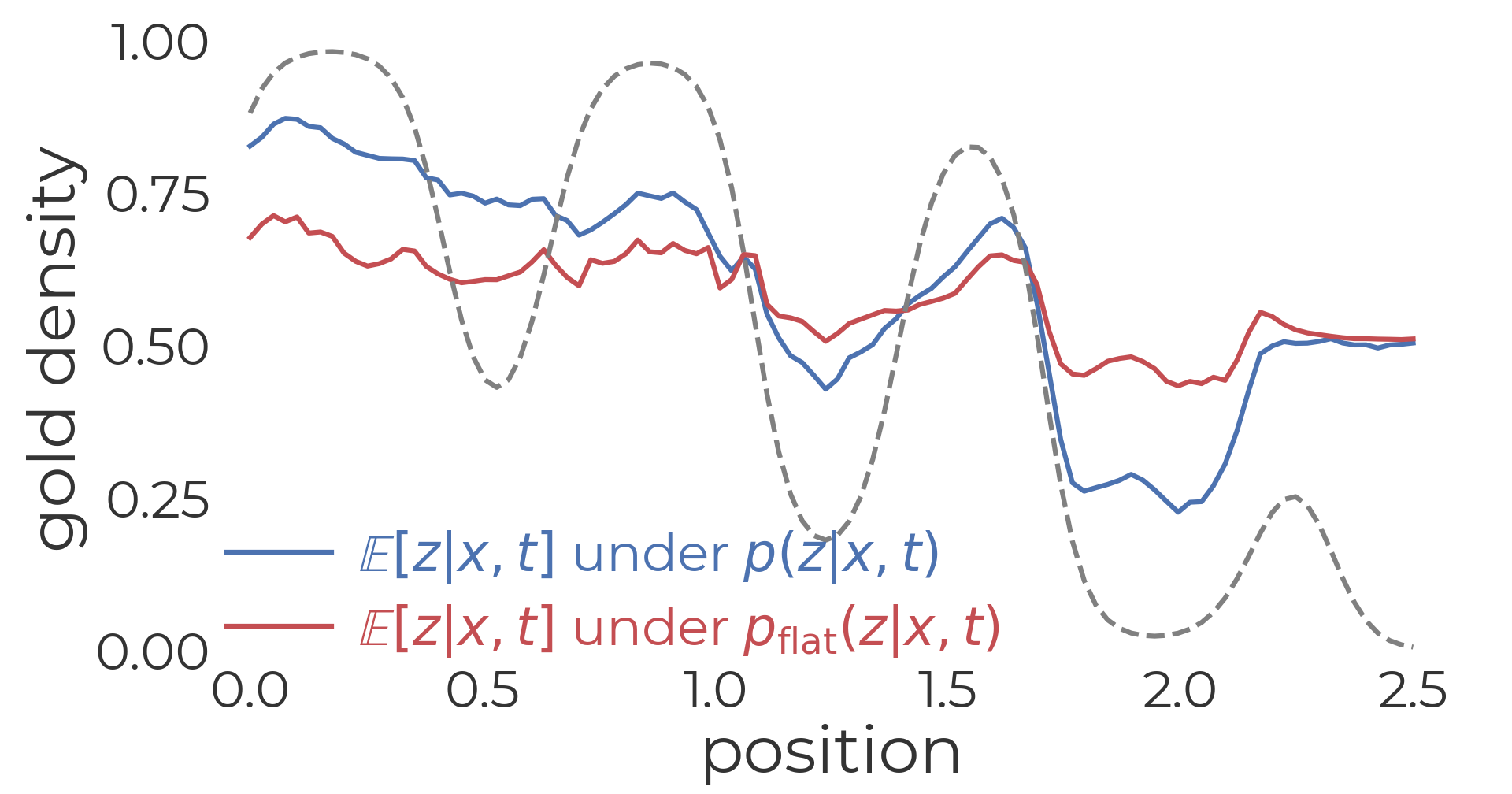}
\par\end{centering}
\caption{Left: The log-marginal likelihoods and weights for the top 50 models
for the \textbf{gold (small)} problem. Right: The final estimated
posteriors, compared to a \textquotedblleft flat\textquotedblright{}
average. The true density is shown for reference.\label{fig:gold-small-bma}}
\end{figure}
\par\end{flushleft}

\cleardoublepage{}

\newpage{}

\subsection{Gold (large)\label{subsec:Gold-(large)}}

This problem concerns the same gold problem, except with a larger
dataset with more observations. Synthetic data was generated by taking
150 training locations $x$ uniformly from 0 to 2.5. The true density
of gold was $\sigma\pars{1+\sin(9x)-\frac{1}{2}x^{2}}$ where $\sigma\pp s=1/\pp{1+e^{-s}}$
denotes a sigmoid function. The test locations were 100 points equally
spaced from 0 to 2.5. A sample of the data is shown below, while the
full data is plotted in \ref{fig:gold-large-data}.
\begin{center}
{\scriptsize{\ttfamily{}%
\noindent\begin{minipage}[t]{1\columnwidth}%
\begin{lstlisting}[
    basicstyle=\scriptsize\ttfamily,
   tabsize=2,
   frame=single,
   rulecolor=\color{gray!100},
   backgroundcolor=\color{gray!03},
    breaklines=true,
    xleftmargin=0pt,
    xrightmargin=0pt,
    breakindent=2\baselineskip,
]
{
    "rod_length": 2.5,
    "num_train": 150,
    "train_locs": [0.6563454972396541, 0.4490753033133052, ..., 1.3381547009467185],
    "gold_train": [1, 0, ..., 0],
    "num_test": 100,
    "test_locs": [0.0, 0.025252525252525252, ..., 2.5],
}
\end{lstlisting}%
\end{minipage}}}{\scriptsize\par}
\par\end{center}

\begin{center}
\begin{figure}[H]
\begin{centering}
\includegraphics[width=0.5\textwidth]{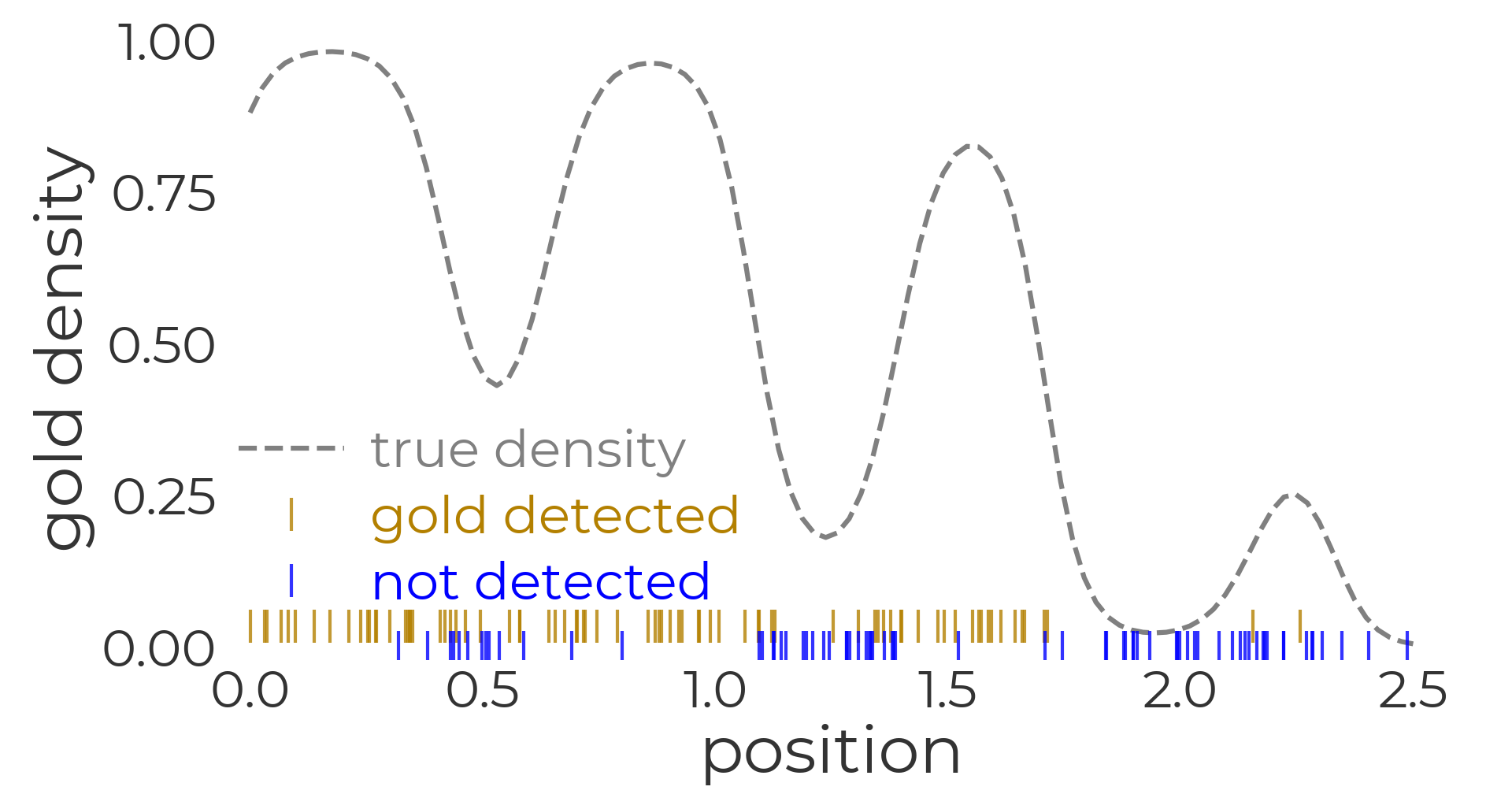}
\par\end{centering}
\caption{Observed data for the \textbf{gold (large)} problem.\label{fig:gold-large-data}}
\end{figure}
\par\end{center}

The same $2^{14}=16384$ candidate models were considered as in the
previous problem. Of these, only 71 (0.4\%) compiled and allowed for
inference with this larger dataset. The top three models and a low
scoring model are shown in \ref{fig:gold-large-top}, \ref{fig:gold-large-2nd},
\ref{fig:gold-large-3rd}, and \ref{fig:gold-large-typical}. Histograms
of estimated marginal likelihoods are in \ref{fig:gold-large-inference},
while \ref{fig:gold-large-posteriors} shows the estimated posterior
for each model, and \ref{fig:gold-large-bma} shows the weights for
the top 50 models and the final estimated posterior.

In this case, essentially all posterior weight is given to the top-three
scoring models, {\color{col1}$m^{(71)}$} \ref{fig:gold-large-top}.
This produces a final posterior quite close to the true density, and
much better than a flat average.

\cleardoublepage{}
\begin{flushleft}
{\scriptsize{}
\begin{figure}[H]
\begin{centering}
{\scriptsize{}%
\begin{minipage}[t]{0.97\columnwidth}%
\begin{lstlisting}[
    basicstyle=\scriptsize\ttfamily\color{col1},    
   tabsize=2,
   frame=single,
   rulecolor=\color{gray!100},
   backgroundcolor=\color{gray!03},
    breaklines=true,
    xleftmargin=0pt,
    xrightmargin=0pt,
    breakindent=2\baselineskip,
]
data {
	real rod_length; 
	int num_train; 
	array[num_train] real train_locs; 
	array[num_train] int<lower=0, upper=1> gold_train; 
	int num_test; 
	array[num_test] real test_locs; 
}
transformed data {
	int<lower=1> N = num_train + num_test;
	array[N] real all_locs = append_array(train_locs, test_locs);
}
parameters {
	vector[N] eta;
	real<lower=0> rho;
}
transformed parameters {
	vector[N] f;
	{
		matrix[N, N] K;
		for (i in 1:N) {
			K[i, i] = 1 + 0.01; // diagonal element
			for (j in 1:(i-1)) {
				real dist = sqrt(square(all_locs[i] - all_locs[j]));
				K[i, j] = sqrt(5)/sqrt(3*rho)*pow(1+sqrt(5*dist/rho), -1.5 - 1/2)*exp(-sqrt(5*dist/rho/3)*1.2);
				K[j, i] = K[i, j]; // fill in lower triangle
			}
		}
		f = K * eta;
	}
}
model {
	eta ~ normal(0, 1);
	rho ~ normal(0, rod_length); // use rod length to specify the input scale of the function
	
	for (i in 1:num_train) {
		gold_train[i] ~ bernoulli_logit(f[i]);
	}
}
generated quantities {
	array[num_test] int<lower=0, upper=1> gold_test;
	for (i in 1:num_test) {
		gold_test[i] = bernoulli_logit_rng(f[num_train + i]);
	}
}
\end{lstlisting}%
\end{minipage}}{\scriptsize\par}
\par\end{centering}
{\scriptsize\caption{{\color{col1}$m^{(71)}$}, the highest-scoring model for the \textbf{gold(large)}
problem.\label{fig:gold-large-top}}
}{\scriptsize\par}
\end{figure}
}{\scriptsize\par}
\par\end{flushleft}

\begin{flushleft}
{\scriptsize{}
\begin{figure}[H]
\begin{centering}
{\scriptsize{}%
\begin{minipage}[t]{0.97\columnwidth}%
\begin{lstlisting}[
    basicstyle=\scriptsize\ttfamily\color{col2},    
   tabsize=2,
   frame=single,
   rulecolor=\color{gray!100},
   backgroundcolor=\color{gray!03},
    breaklines=true,
    xleftmargin=0pt,
    xrightmargin=0pt,
    breakindent=2\baselineskip,
]
data {
	real rod_length;
	int num_train;
	array[num_train] real train_locs;
	array[num_train] int<lower=0, upper=1> gold_train;
	int num_test;
	array[num_test] real test_locs;
}
transformed data {
	int<lower=1> N = num_train + num_test;
	array[N] real all_locs = append_array(train_locs, test_locs);
}
parameters {
	real<lower=0> rho; // length scale for Matern covariance
	real<lower=0> alpha; // scale for Matern covariance
	vector[N] eta; // latent GP value at each location
}
model {
	rho ~ cauchy(0, 2); // prior for length scale
	alpha ~ cauchy(0, 2); // prior for scale
	eta ~ multi_normal(rep_vector(0, N), 
		gp_matern32_cov(all_locs, alpha, rho)); // sample from GP prior
	for (i in 1:num_train){
		gold_train[i] ~ bernoulli_logit(eta[i]);
	}
}
generated quantities {
	array[num_test] int<lower=0, upper=1> gold_test;
	for (i in 1:num_test){
		gold_test[i] = bernoulli_logit_rng(eta[i + num_train]);
	}
}
\end{lstlisting}%
\end{minipage}}{\scriptsize\par}
\par\end{centering}
{\scriptsize\caption{{\color{col2}$m^{(70)}$}, the second-highest-scoring model for
the \textbf{gold(large)} problem.\label{fig:gold-large-2nd}}
}{\scriptsize\par}
\end{figure}
}{\scriptsize\par}
\par\end{flushleft}

\begin{flushleft}
{\scriptsize{}
\begin{figure}[H]
\begin{centering}
{\scriptsize{}%
\begin{minipage}[t]{0.97\columnwidth}%
\begin{lstlisting}[
    basicstyle=\scriptsize\ttfamily\color{col3},    
   tabsize=2,
   frame=single,
   rulecolor=\color{gray!100},
   backgroundcolor=\color{gray!03},
    breaklines=true,
    xleftmargin=0pt,
    xrightmargin=0pt,
    breakindent=2\baselineskip,
]
data {
	real rod_length;
	int num_train;
	array[num_train] real train_locs;
	array[num_train] int<lower=0, upper=1> gold_train;
	int num_test;
	array[num_test] real test_locs;
}
parameters {
	vector[num_train] latent;
	real<lower=0> sigma;
	real<lower=0> rho;
}
model {
	sigma ~ cauchy(0, 2);
	rho ~ cauchy(0, 2);
	{
		matrix[num_train, num_train] K;
		for (i in 1:num_train) {
			K[i, i] = 1 + sigma^2;
			for (j in 1:(i - 1)) {
				K[i, j] = sigma^2 * exp(-((train_locs[i] - train_locs[j])^2) / (2 * rho^2));
				K[j, i] = K[i, j];
			}
		}
		latent ~ multi_normal(rep_vector(0, num_train), K);
	}
	gold_train ~ bernoulli_logit(latent);
}
generated quantities {
	vector[num_test] test_latent;
	{
		matrix[num_test, num_train] K_test;
		matrix[num_train, num_train] K_train;
		for (i in 1:num_train) {
			K_train[i, i] = 1 + sigma^2;
			for (j in 1:(i - 1)) {
				K_train[i, j] = sigma^2 * exp(-((train_locs[i] - train_locs[j])^2) / (2 * rho^2));
				K_train[j, i] = K_train[i, j];
			}
		}
		for (i in 1:num_test) {
			for (j in 1:num_train) {
				K_test[i, j] = sigma^2 * exp(-((test_locs[i] - train_locs[j])^2) / (2 * rho^2));
			}
		}
		test_latent = (K_test * inverse(K_train)) * latent;
	}
	array[num_test] int<lower=0, upper=1> gold_test;
	for (i in 1:num_test) {
		gold_test[i] = bernoulli_logit_rng(test_latent[i]);
	}
}
\end{lstlisting}%
\end{minipage}}{\scriptsize\par}
\par\end{centering}
{\scriptsize\caption{{\color{col3}$m^{(69)}$}, the third-highest-scoring model for the
\textbf{gold(large)} problem.\label{fig:gold-large-3rd}}
}{\scriptsize\par}
\end{figure}
}{\scriptsize\par}
\par\end{flushleft}

\begin{flushleft}
{\scriptsize{}
\begin{figure}[H]
\begin{centering}
{\scriptsize{}%
\begin{minipage}[t]{0.97\columnwidth}%
\begin{lstlisting}[
    basicstyle=\scriptsize\ttfamily\color{col4},    
   tabsize=2,
   frame=single,
   rulecolor=\color{gray!100},
   backgroundcolor=\color{gray!03},
    breaklines=true,
    xleftmargin=0pt,
    xrightmargin=0pt,
    breakindent=2\baselineskip,
]
data {
	real rod_length;
	int num_train;
	array[num_train] real train_locs;
	array[num_train] int<lower=0, upper=1> gold_train;
	int num_test;
	array[num_test] real test_locs;
}
parameters {
	real<lower=0> gp_sigma;
	real<lower=0> gp_rho;
	vector[num_train] gp_z;
}
transformed parameters {
	vector[num_train] eta;
	eta = gp_z;
	for (i in 1:num_train) {
		real sum = 0;
		for (j in 1:num_train) {
			sum += gp_sigma * exp(-((train_locs[i] - train_locs[j]) / gp_rho)^2) * gp_z[j];
		}
		eta[i] = sum;
	}
}
model {
	gp_sigma ~ inv_gamma(1,1);
	gp_rho ~ inv_gamma(1,1);
	gp_z ~ normal(0,1);
	gold_train ~ bernoulli_logit(eta);
}
generated quantities {
	vector[num_test] eta_test;
	for (i in 1:num_test) {
		real sum = 0;
		for (j in 1:num_train) {
			sum += gp_sigma * exp(-((test_locs[i] - train_locs[j]) / gp_rho)^2) * gp_z[j];
		}
		eta_test[i] = sum;
	}
	array[num_test] int<lower=0, upper=1> gold_test;
	for (i in 1:num_test) {
		gold_test[i] = bernoulli_logit_rng(eta_test[i]);
	}
}
\end{lstlisting}%
\end{minipage}}{\scriptsize\par}
\par\end{centering}
{\scriptsize\caption{{\color{col4}$m^{(51)}$}, a low-scoring model for the \textbf{gold(large)}
problem.\label{fig:gold-large-typical}}
}{\scriptsize\par}
\end{figure}
}{\scriptsize\par}
\par\end{flushleft}

\cleardoublepage{}
\begin{flushleft}
\begin{figure}[H]
\begin{centering}
\includegraphics[width=0.5\textwidth]{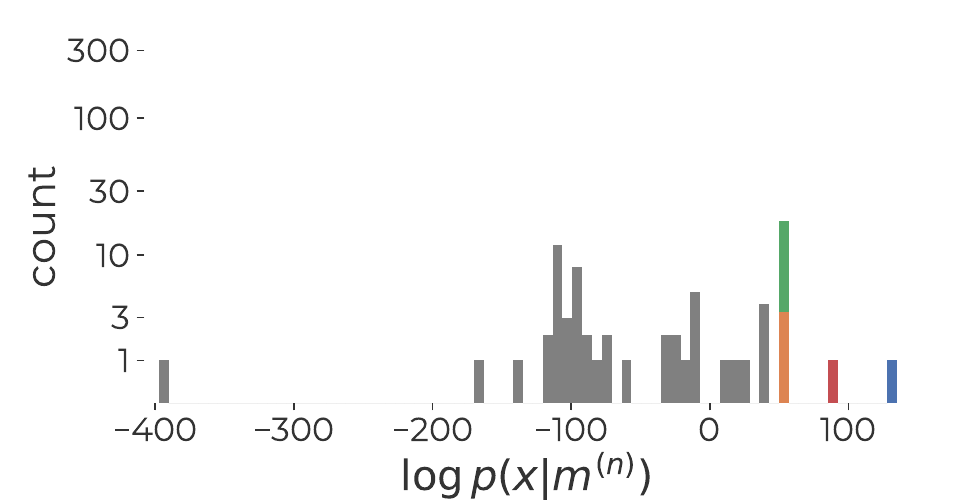}\includegraphics[width=0.5\textwidth]{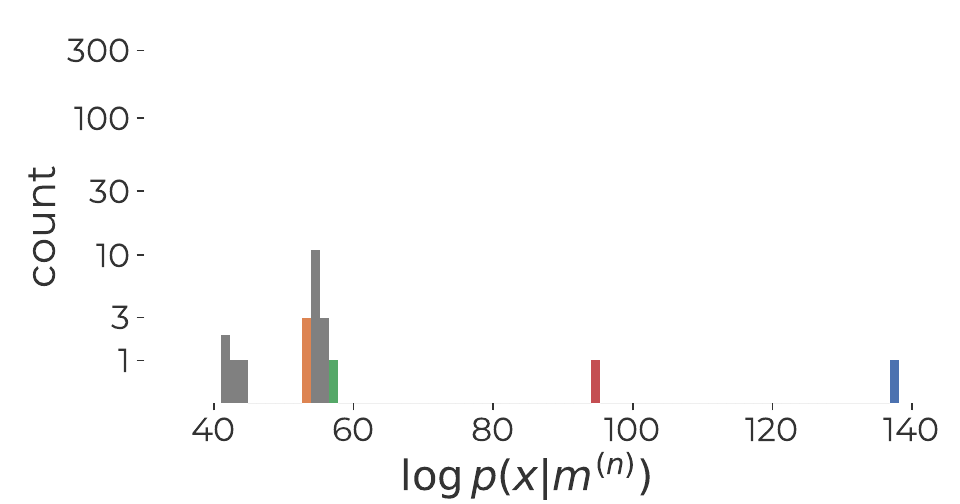}
\par\end{centering}
\caption{Estimated marginal likelihoods $p\protect\pp{x\vert m^{(n)}}$ for
each of the valid models $n$ for the \textbf{gold (large)} problem.
Because of the many order of magnitude, two different ranges are shown.
Bars are colored for four example models from \ref{fig:gold-large-top},
\ref{fig:gold-large-2nd}, \ref{fig:gold-large-3rd}, and \ref{fig:gold-large-typical}.
If multiple models map to the same bin, all colors are shown stacked.
\label{fig:gold-large-inference}}
\end{figure}
\par\end{flushleft}

\begin{flushleft}
\begin{figure}
\begin{centering}
\includegraphics[width=0.5\columnwidth]{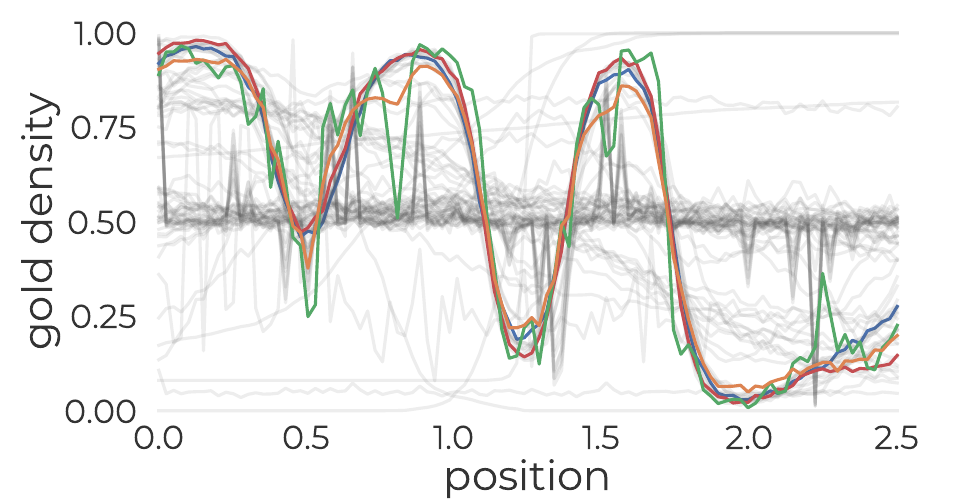}
\par\end{centering}
\caption{The estimated posteriors $\protect\E[z\vert m^{(n)},x]$ for each
of the valid models $n$ for the \textbf{gold (small)} problem. Lines
are colored for four example models from \ref{fig:gold-large-top},
\ref{fig:gold-large-2nd}, \ref{fig:gold-large-3rd}, and \ref{fig:gold-large-typical}.\label{fig:gold-large-posteriors}}
\end{figure}
\par\end{flushleft}

\begin{flushleft}
\begin{figure}
\begin{centering}
\includegraphics[width=0.5\columnwidth]{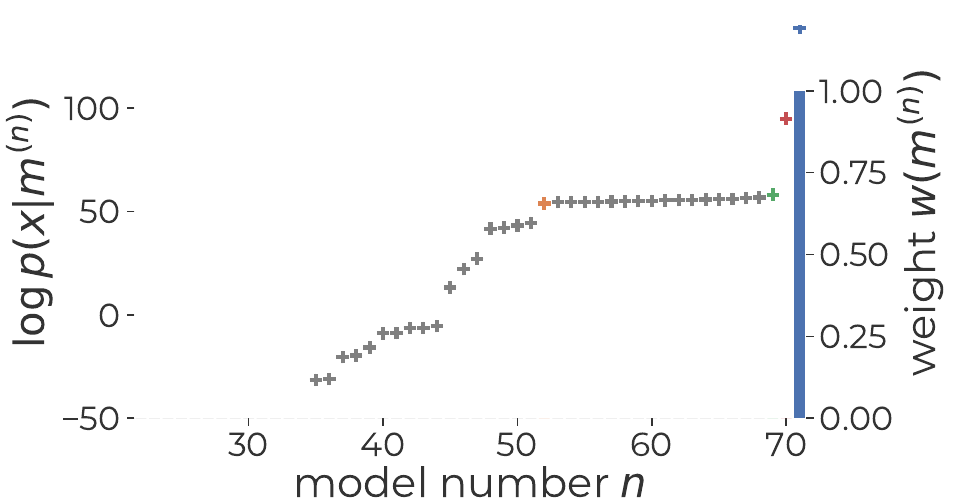}\includegraphics[width=0.5\textwidth]{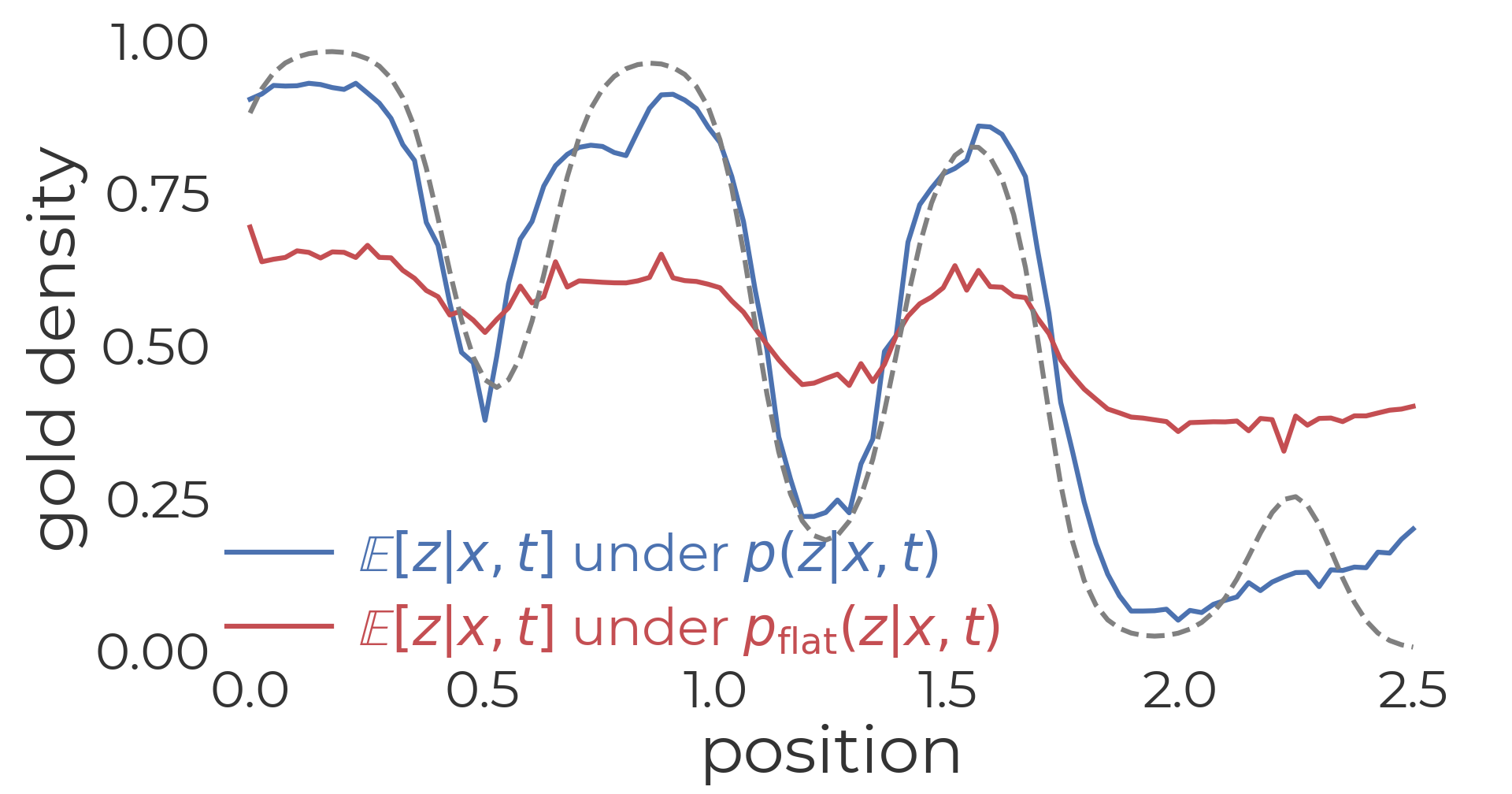}
\par\end{centering}
\caption{Left: The log-marginal likelihoods and weights for the top 50 models
for the \textbf{gold (small)} problem. Right: The final estimated
posteriors, compared to a \textquotedblleft flat\textquotedblright{}
average. The true density is shown for reference.\label{fig:gold-large-bma}}
\end{figure}
\par\end{flushleft}

\cleardoublepage{}

\newpage{}

\section{Additional analysis}

\subsection{Form of the self-normalized importance sampling weights\label{sec:Form-of-SNIS-weights}}

Consider self-normalized importance sampling with a target distribution
of $p\pp{m\vert x,t}$ and a proposal distribution of $p\pp{m\vert t}.$
The naive importance weights would thus be

\[
w^{(n)}=\frac{\frac{p\pp{m^{(n)}\vert x,t}}{p\pp{m^{(n)}\vert t}}}{\sum_{n'=1}^{N}\frac{p\pp{m^{(n')}\vert x,t}}{p\pp{m^{(n')}\vert t}}}.
\]
However, we know from \ref{eq:p(m|x_t)} that $p\pp{m\vert x,t}\propto p\pp{m\vert t}p\pp{x\vert m}$
or, more precisely
\begin{eqnarray*}
p\pp{m\vert x,t} & = & \frac{p\pp{m\vert t}p\pp{x\vert m}}{\sum_{m'}p\pp{m'\vert t}p\pp{x\vert m'}}=\frac{p\pp{m\vert t}p\pp{x\vert m}}{p\pp{x\vert t}}.
\end{eqnarray*}
From this, we obtain that
\[
\frac{p\pp{m\vert x,t}}{p\pp{m\vert t}}=\frac{p\pp{x\vert m}}{p\pp{x\vert t}}.
\]
Substituting this into the above form for $w^{(n)},$ we obtain that

\begin{eqnarray*}
w^{(n)} & = & \frac{\frac{p\pp{x\vert m^{(n)}}}{p\pp{x\vert t}}}{\sum_{n'=1}^{N}\frac{p\pp{x\vert m^{(n')}}}{p\pp{x\vert t}}}\\
 & = & \frac{p\pp{x\vert m^{(n)}}}{\sum_{n'=1}^{N}p\pp{x\vert m^{(n')}}}.
\end{eqnarray*}

\subsection{Variance of the self-normalized importance sampling estimator\label{sec:Variance-of-the-IS-estimator}}

Consider the importance sampling estimator from \ref{eq:IS-estimator},
i.e.
\[
\hat{\mu}=\sum_{n=1}^{N}\r w^{(n)}g\pp{\r m^{(n)}}.
\]

Standard theory \citep[e.g.][Eq. 9.8]{Owen_2013_MonteCarloTheory}
says that the variance of $\hat{\mu}$ is asymptotically 
\[
\lim_{N\rightarrow\infty}N\mathbb{V}\bb{\hat{\r{\mu}}}=\frac{\E_{p\pp{\r m\vert t}}\bracs{w\pp{\r m}^{2}\pars{g\pp{\r m}-\mu}^{2}}}{\E_{p\pp{\r m\vert t}}\bracs{w\pp{\r m}}^{2}},
\]
where $w\pp m=\frac{p\pp{m\vert x,t}}{p\pp{m\vert t}}.$ This simplifies
into
\begin{eqnarray*}
\lim_{N\rightarrow\infty}N\mathbb{V}\bb{\hat{\mu}} & = & \frac{\E_{p\pp{\r m\vert t}}\bracs{\frac{p\pp{\r m\vert x,t}^{2}}{p\pp{\r m\vert t}^{2}}\pars{g\pp{\r m}-\mu}^{2}}}{\E_{p\pp{\r m\vert t}}\bracs{\frac{p\pp{\r m\vert x,t}}{p\pp{\r m\vert t}}}^{2}}\\
 & = & \E_{p\pp{\r m\vert t}}\bracs{\frac{p\pp{\r m\vert x,t}^{2}}{p\pp{\r m\vert t}^{2}}\pars{g\pp{\r m}-\mu}^{2}}.
\end{eqnarray*}

Now, if we assume that $\verts{g\pp m-\mu}\leq\delta$ then this can
be bounded by
\begin{eqnarray*}
\lim_{N\rightarrow\infty}N\mathbb{V}\bb{\hat{\mu}} & \leq & \delta^{2}\E_{p\pp{\r m\vert t}}\bracs{\frac{p\pp{\r m\vert x,t}^{2}}{p\pp{\r m\vert t}^{2}}}\\
 & = & \delta^{2}\pars{1+\chi^{2}\pars{p\pp{\r m\vert x,t}\Vert p\pp{\r m\vert t}}}
\end{eqnarray*}

\subsection{Form of the self-normalized importance sampling weights with variational
inference\label{sec:Form-of-SNIS-weights-with-VI}}

Consider trying to estimate
\[
q\pp{z\vert x}=\E_{q\pp{\r m\vert x}}q\pp{z\vert x,\r m},
\]
where $q\pp{z\vert x,m}$ is some distribution and $q\pp{m\vert x}$
is as defined in \ref{cor:variational-q(m)}. We are interested in
estimating this expectation using SNIS with the proposal distribution
$p\pp{m\vert t}$. It follows that we should use an estimator
\[
\tilde{q}\pp{z\vert x}=\sum_{n=1}^{N}w^{(n)}q\pp{z\vert x,m^{(n)}},
\]
where $m^{(1)},\cdots,m^{(N)}\sim p\pp{m\vert t}$, and the self-normalized
weights are
\begin{eqnarray*}
w^{(m)} & = & \frac{\frac{q\pp{m^{(n)}\vert x}}{p\pp{m^{(n)}\vert t}}}{\sum_{n'=1}^{N}\frac{q\pp{m^{(n')}\vert x}}{p\pp{m^{(n')}\vert t}}}.
\end{eqnarray*}
Now, observe from \ref{eq:optimal-q-joint-divergence-ELBO-form} that
\[
q\pp{m\vert x}=\frac{p\pp{m\vert t}\exp\Bigl(\mathrm{ELBO}\pp{q\pp{\r z\vert x,m}\Vert p\pp{\r z,m\vert x,t}}\Bigr)}{\sum_{m'}p\pp{m'\vert t}\exp\Bigl(\mathrm{ELBO}\pp{q\pp{\r z\vert x,m'}\Vert p\pp{\r z,m'\vert x,t}}\Bigr)},
\]
where the denominator does not depend on $m$. It follows that the
weights are
\begin{eqnarray*}
w^{(m)} & = & \frac{\exp\Bigl(\mathrm{ELBO}\pp{q\pp{\r z\vert x,m^{(n)}}\Vert p\pp{\r z,m^{(n)}\vert x,t}}\Bigr)}{\sum_{n'=1}^{N}\exp\Bigl(\mathrm{ELBO}\pp{q\pp{\r z\vert x,m^{(n')}}\Vert p\pp{\r z,m^{(n')}\vert x,t}}\Bigr)}.
\end{eqnarray*}

\subsection{Analysis with inexact ELBO estimates\label{sec:Analysis-with-inexact-ELBO}}
\begin{thm}
Suppose that $p\pp{z,x,m\vert t}$ and $q\pp{z\vert x,m}$ are fixed
and we choose\label{thm:joint-divergence-approximate-ELBO}
\[
q\pp{m\vert x}\propto p\pp{m\vert x}\times\exp\pars{\mathrm{ELBO}\pp{q\pp{\r z\vert x,m}\Vert p\pp{\r z,x\vert m}}-\delta^{(\r m)}}
\]
where $\delta^{(m)}\geq0$ represents the ``slack'' in between the
true ELBO for model $m$ and the bound used for computing $q\pp{m\vert x}$.
Then the resulting joint divergence is
\begin{alignat}{1}
 & \KL{q\pp{\r z,\r m\vert x}}{p\pp{\r z,\r m\vert x,t}}\nonumber \\
 & =-\log\E_{p\pp{\r m\vert x,t}}\exp\Bigl(-\mathrm{KL}^{(\r m)}-\delta^{(\r m)}+\bar{\delta}\Bigr)\label{eq:optimal-joint-KL-inexact}\\
 & =\log p\pp{x\vert t}-\log\E_{p\pp{\r m\vert t}}\exp\pars{\mathrm{ELBO}^{(\r m)}-\delta^{(\r m)}+\bar{\delta}}\label{eq:optimal-joint-KL-inexact-ELBO-form}
\end{alignat}
where $\bar{\delta}=\E_{q\pp{\r m\vert x}}\delta^{(\r m)}$, and $\mathrm{KL}^{(m)}$
and $\mathrm{ELBO}^{(m)}$ represent the KL-divergence and ELBO from
\ref{eq:ELBO-decomp-1}.
\end{thm}

(Proof in \ref{subsec:Proof-with-inexact-ELBOs}.)

To interpret this result, first note that if all ELBOs are exact,
meaning $\delta^{(m)}=0$, then \ref{eq:optimal-joint-KL-inexact}
is equivalent to \ref{eq:optimal-q-joint-divergence} while \ref{eq:optimal-joint-KL-inexact-ELBO-form}
is equivalent to \ref{eq:optimal-joint-KL-ELBO-form}. This also shows
that using ``improved'' distributions and tolerating ``slack''
in the ELBO can only help, compared to using distributions where the
exact ELBO has the same value. That is, no matter the value of the
slack variables $\delta^{(m)}$, \ref{eq:optimal-joint-KL-inexact}
is never worse than \ref{eq:optimal-q-joint-divergence} and \ref{eq:optimal-joint-KL-inexact-ELBO-form}
is never worse than \ref{eq:optimal-joint-KL-ELBO-form}.

The obvious implementation of \ref{thm:joint-divergence-approximate-ELBO}
would be be an algorithm that loops over all models. This is given
explicitly as \ref{alg:Variational-LLB-inexact} in \ref{sec:Other-example-algorithms}.
Again, such an algorithm is intractable, but can be made practical
by combining it with SNIS. This is given as \ref{alg:self-normalized-LLB-inexact}.

An interesting special case is when MCMC is used to approximate each
posterior $p\pp{z\vert x,m}$ and we assume that these samples in
fact come from $p\pp{z\vert x,m}$. (See \ref{alg:LLB-using-MCMC-and-VI}
in \ref{sec:Other-example-algorithms}.) If this case, the best strategy
is simply to use the best possible lower-bound on the marginal likelihood,
computed by any method. It's see in this case that the joint divergence
reduces to
\[
-\log\sum_{m}p\pp{m\vert x,t}\exp\Bigl(-\delta^{(m)}+\bar{\delta}\Bigr).
\]
From this, we can see that if all errors $\delta^{(m)}$ are equal,
then $\delta^{(m)}=\bar{\delta}$ and the joint KL-divergence remains
zero. Also, if there are errors on models where $p\pp{m\vert x,t}$
is negligible, these also have negligible impact on the joint divergence.
However, if there are different levels of error on models where $p\pp{m\vert x,t}$
is large, this could impact the joint divergence.

\subsection{Relaxed bound\label{sec:Relaxed-bound}}
\begin{cor}
\label{cor:relaxed-bound}Under the assumptions of \ref{thm:variational-lemma},
\[
\KL{q\pp{z,m\vert x}}{p\pp{z,m\vert x,t}}\leq\E_{p\pp{m\vert x,t}}\KL{q\pp{z\vert x,m}}{p\pp{z\vert x,m}}\Bigr).
\]
\end{cor}

\begin{proof}
We have that
\begin{eqnarray*}
\E_{p\pp{m\vert x,t}}\KL{q\pp{z\vert x,m}}{p\pp{z\vert x,m}}\Bigr) & = & -\log\exp\pars{-\E_{p\pp{m\vert x,t}}\KL{q\pp{z\vert x,m}}{p\pp{z\vert x,m}}\Bigr)}\\
 & \geq & -\log\E_{p\pp{m\vert x,t}}\exp\pars{-\KL{q\pp{z\vert x,m}}{p\pp{z\vert x,m}}\Bigr)}\\
 & = & \KL{q\pp{z,m\vert x}}{p\pp{z,m\vert x,t}}.
\end{eqnarray*}
The first line is obvious. The second line follows from using Jensen's
inequality to see that
\[
\E_{p\pp{m\vert x,t}}\exp\Bigl(-\KL{q\pp{z\vert x,m}}{p\pp{z\vert x,m}}\Bigr)\geq\exp\Bigl(-\E_{p\pp{m\vert x,t}}\KL{q\pp{z\vert x,m}}{p\pp{z\vert x,m}}\Bigr).
\]
The third line is the form of $\KL{q\pp{z,m\vert x}}{p\pp{z,m\vert x,t}}$
from \ref{thm:variational-lemma}.
\end{proof}
\cleardoublepage{}


\section{Other example algorithms\label{sec:Other-example-algorithms}}

\begin{algorithm}[t]
\begin{enumerate}
\item Input textual description $t$ and data $x$.
\item For $n=1,2,\cdots,N$:
\begin{enumerate}
\item Sample $m^{(n)}\sim p\pp{m|t}$.\hfill{}\textcolor{gray}{// using
LLM}
\item Compute marginal evidence $p\pp{x\vert m^{(n)}}$ and posterior $p\pp{z\vert x,m^{(n)}}$.\hfill{}\textcolor{gray}{//
under PPL}
\end{enumerate}
\item Set $w^{(n)}\propto p\pp{x|m^{(n)}}$, where $\sum_{n=1}^{N}w^{(n)}=1$.
\item Use final posterior approximation\vspace{-8pt}
\[
p\pp{z|x,t}\approx\sum_{n=1}^{N}w^{(n)}p\pp{z\vert x,m^{(n)}}.
\]
\vspace{-9.3pt}
\end{enumerate}
\caption{LLB with exact inference and SNIS.\label{alg:self-normalized-LLB}}
\end{algorithm}

\begin{algorithm}[t]
\begin{enumerate}
\item Input textual description $t$ and data $x$.
\item For $n=1,2,\cdots,N$
\begin{enumerate}
\item Sample $m^{(n)}\sim p\pp{m\vert t}$.\hfill{}\textcolor{gray}{//
using LLM}
\item Maximize $\mathrm{ELBO}(m^{(n)})$ over some variational family $q\pp{z\vert x,m^{(n)}}$.\textcolor{gray}{\hspace{3cm}//
under PPL}
\end{enumerate}
\item Set $w^{(n)}\propto\exp\mathrm{ELBO}(m^{(n)})$, where $\sum_{n=1}^{N}w^{(n)}=1$.
\item Use final posterior approximation\vspace{-8pt}
\[
p\pp{z|x,t}\approx\sum_{n=1}^{N}w^{(n)}q\pp{z\vert x,m^{(n)}}.
\]
\vspace{-9.3pt}
\end{enumerate}
\caption{LLB with variational inference and SNIS.\label{alg:Variational-LLB-IS}}
\end{algorithm}
\begin{algorithm}
\begin{enumerate}
\item Input textual description $t$ and data $x$.
\item For all possible model strings $m$:
\begin{enumerate}
\item Compute model probability $p\pp{m\vert t}$ (using LLM)
\item Maximize $\mathrm{ELBO}\pp{q\pp{\r z\vert x,m^{(n)}}\Vert p\pp{\r z,x\vert m^{(n)}}}$
over some variational family $q$.
\end{enumerate}
\item Set $w^{(m)}\propto p\pp{m\vert t}\exp\Bigl(\mathrm{ELBO}\pp{q\pp{\r z\vert x,m^{(n)}}\Vert p\pp{\r z,x\vert m^{(n)}}}\Bigr)$.
\item Use final posterior approximation
\[
p\pp{z|x,t}\approx\sum_{m}w^{(m)}q\pp{z\vert x,m^{(n)}}.
\]
\end{enumerate}
\caption{Variational LLB with exact enumeration of models (theoretical)\label{alg:Variational-LLB}}
\end{algorithm}

\begin{algorithm}
\begin{enumerate}
\item Input textual description $t$ and data $x$.
\item For all possible model strings $m$:
\begin{enumerate}
\item Compute model probability $p\pp{m\vert t}$ (using LLM)
\item Compute some approximation to the posterior $q\pp{z\vert x,m}\approx p\pp{z\vert x,m}$.
\item Compute some bound on the ELBO
\[
L^{(m)}\leq\E_{q\pp{\r z\vert x,m}}\log\frac{p\pp{\r z,x\vert m}}{q\pp{\r z\vert x,m}}\leq\log p\pp{x\vert m}.
\]
\end{enumerate}
\item Set $w^{(m)}\propto p\pp{m\vert t}\exp\pars{L^{(m)}}$.
\item Use final posterior approximation
\[
p\pp{z|x,t}\approx\sum_{m}w^{(m)}q\pp{z\vert x,m}.
\]
\end{enumerate}
\caption{LLB with variational inference with inexact ELBOs (theoretical)\label{alg:Variational-LLB-inexact}}
\end{algorithm}

\begin{algorithm}[H]
\begin{enumerate}
\item Input textual description $t$ and data $x$.
\item For $n=1,2,\cdots,N$:
\begin{enumerate}
\item Sample $m^{(n)}\sim p\pp{m|t}$ (using an LLM).
\item Compute some approximation to the posterior $q\pp{z\vert x,m^{(n)}}\approx p\pp{z\vert x,m^{(n)}}$.
\item Compute some bound on the ELBO
\[
L^{(n)}\leq\E_{q\pp{z\vert x,m^{(n)}}}\log\frac{p\pp{z,x\vert m^{(n)}}}{q\pp{z\vert x,m^{(n)}}}\leq\log p\pp{x\vert m^{(n)}}.
\]
\end{enumerate}
\item Set $w^{(n)}\propto\exp L^{(n)}$.
\item Use final posterior approximation
\[
p\pp{z|x,t}\approx\sum_{n=1}^{N}w^{(n)}p\pp{z\vert x,m^{(n)}}.
\]
\end{enumerate}
\caption{LLB with ELBO bounds and SNIS.\label{alg:self-normalized-LLB-inexact}}
\end{algorithm}

\cleardoublepage{}

\onecolumn

\section{Proofs}

\subsection{Proof of \ref{thm:variational-lemma}\label{subsec:proof-of-variational-result}}
\begin{proof}
First, notice that minimizing
\[
\KL{q\pp{z,m\vert x}}{p\pp{z,m\vert x,t}}
\]
is equivalent to maximizing
\[
\mathrm{ELBO}\pp{q\pp{z,m\vert x}\Vert p\pp{z,x,m\vert t}}=\E_{q\pp{z\vert x,m}}\log\frac{p\pp{z,x,m,\vert t}}{q\pp{z\vert x,m}}.
\]
We will therefore attempt to minimize the latter. Set up Lagrangian
to find optimal $q\pp{m\vert x}$.
\begin{eqnarray*}
\mathcal{L} & = & \sum_{m}\int q\pp{m\vert t}q\pp{z|m,x}\log\frac{p\pp{z,x,m\vert t}}{q\pp{m\vert t}q\pp{z|m,x}}dz-\lambda\pars{\sum_{m}q\pp{m\vert t}-1}\\
\frac{d\mathcal{L}}{dq\pp{m\vert t}} & = & \int q\pp{z|m}\log\frac{p\pp{z,x,m\vert t}}{q\pp{m\vert t}q\pp{z|m,x}}dz-\int q\pp{m\vert t}q\pp{z|m,x}\frac{d}{dq\pp{m\vert t}}\log q\pp{m\vert t}dz-\lambda\\
 & = & \int q\pp{z|m}\log\frac{p\pp{z,x,m\vert t}}{q\pp{m\vert t}q\pp{z|m,x}}dz-1-\lambda\\
 & = & \int q\pp{z|m}\log\frac{p\pp{m\vert t}p\pp{x\vert m}p\pp{z\vert x,m}}{q\pp{m\vert t}q\pp{z|m,x}}dz-1-\lambda\\
 & = & \log\frac{p\pp{m\vert t}p\pp{x\vert m}}{q\pp{m\vert t}}-\KL{q\pp{z\vert x,m}}{p\pp{z\vert x,m}}-1-\lambda
\end{eqnarray*}
Solving for $\frac{d\mathcal{L}}{dq\pp m}=0$ we obtain that
\[
q\pp{m\vert t}\propto p\pp{m\vert t}p\pp{x\vert m}\exp\pars{-\KL{q\pp{z\vert x,m}}{p\pp{z\vert x,m}}}.
\]

Now, for simplicity, define $\KL{q\pp{z\vert x,m}}{p\pp{z\vert x,m}}=\Delta_{m}$.
Then, we can write

\begin{eqnarray*}
q\pp{m\vert x} & = & \frac{p\pp{m\vert t}p\pp{x\vert m}\exp\Bigl(-\Delta_{m})}{\sum_{m'}p\pp{m'\vert t}p\pp{x\vert m'}\exp\Bigl(-\Delta_{m'})},\\
p\pp{m\vert x,t} & = & \frac{p\pp{m\vert t}p\pp{x\vert m}}{\sum_{m'}p\pp{m'\vert t}p\pp{x\vert m'}}.
\end{eqnarray*}

Now, note that
\begin{eqnarray*}
\frac{q\pp{m\vert x}}{p\pp{m\vert x,t}} & = & \frac{p\pp{m\vert t}p\pp{x\vert m}\exp\Bigl(-\Delta_{m})}{\sum_{m'}p\pp{m'\vert t}p\pp{x\vert m'}\exp\Bigl(-\Delta_{m'})}\frac{\sum_{m'}p\pp{m'\vert t}p\pp{x\vert m'}}{p\pp{m\vert t}p\pp{x\vert m}}\\
 & = & \frac{\exp\Bigl(-\Delta_{m})\sum_{m'}p\pp{m'\vert t}p\pp{x\vert m'}}{\sum_{m'}p\pp{m'\vert t}p\pp{x\vert m'}\exp\Bigl(-\Delta_{m'})}\\
 & = & \frac{\exp\Bigl(-\Delta_{m})}{\sum_{m'}p\pp{m'\vert x,t}\exp\Bigl(-\Delta_{m'})}
\end{eqnarray*}

Thus we have that

\begin{eqnarray*}
\KL{q\pp{z,m\vert x}}{p\pp{z,m\vert x,t}} & = & \E_{q\pp{m\vert x}}\log\frac{q\pp{m\vert x}}{p\pp{m\vert x,t}}+\E_{q\pp{z,m\vert x}}\log\frac{q\pp{z\vert m,x}}{p\pp{z\vert m,x}}\\
 & = & \E_{q\pp{m\vert x}}\log\frac{\exp\Bigl(-\Delta_{m})}{\sum_{m'}p\pp{m'\vert x,t}\exp\Bigl(-\Delta_{m'})}+\E_{q\pp{m\vert x}}\Delta_{m}\\
 & = & \log\frac{1}{\sum_{m'}p\pp{m'\vert x,t}\exp\Bigl(-\Delta_{m'})}\\
 & = & -\log\E_{p\pp{m\vert x,t}}\exp\pp{-\Delta_{m}}
\end{eqnarray*}
\end{proof}
Note that since $\Delta_{m}\geq0$, it follows that $\exp\pp{-\Delta_{m}}\leq1$
and thus $\E_{p\pp{m\vert x,t}}\exp\pp{-\Delta_{m}}\leq1$ and so
$\log\E_{p\pp{m\vert x,t}}\exp\pp{-\Delta_{m}}\leq0$. Thus $\KL{q\pp{z,m\vert x}}{p\pp{z,m\vert x,t}}$
is non-negative, as required.

\subsection{Proof of \ref{cor:variational-q(m)}\label{subsec:Proof-of-ELBO-form}}
\begin{proof}
Recall that the \ref{thm:variational-lemma} found that the optimal
$q\pp{m\vert x}$ is
\begin{equation}
q\pp{m\vert x}\propto p\pp{m\vert t}p\pp{x\vert m}\times\exp\Bigl(-\KL{q\pp{\r z\vert x,m}}{p\pp{\r z\vert x,m}}\Bigr).\label{eq:previous-optimal-q(m|x)}
\end{equation}

And recall the ``ELBO decomposition'', i.e. the fact that
\begin{equation}
\log p\pp{x\vert m}=\mathrm{ELBO}\pp{q\pp{\r z\vert x,m}\Vert p\pp{\r z,x\vert m}}+\KL{q\pp{\r z\vert x,m}}{p\pp{\r z\vert x,m}}.\label{eq:ELBO-decomp}
\end{equation}
From which it follows that
\[
p\pp{x\vert m}\times\exp\pars{-\KL{q\pp{\r z\vert x,m}}{p\pp{\r z\vert x,m}}}=\mathrm{ELBO}\pp{q\pp{\r z\vert x,m}\Vert p\pp{\r z,x\vert m}}
\]

Substituting into \ref{eq:previous-optimal-q(m|x)} gives that
\[
q\pp{m\vert x}\propto p\pp{m\vert t}\exp\Bigl(\mathrm{ELBO}\pp{q\pp{\r z\vert x,m}\Vert p\pp{\r z,m\vert x,t}}\Bigr).
\]

Recall also that \ref{thm:variational-lemma} found the joint divergence
resulting from using the optimal $q\pp{m\vert x}$ is
\begin{eqnarray*}
\KL{q\pp{\r z,\r m\vert x}}{p\pp{\r z,\r m\vert x,t}} & = & -\log\E_{p\pp{\r m\vert x,t}}\exp\Bigl(-\KL{q\pp{\r z\vert x,\r m}}{p\pp{\r z\vert x,\r m}}\Bigr)\\
 & = & -\log\E_{p\pp{\r m\vert x,t}}\exp\pars{\mathrm{ELBO}\pp{q\pp{\r z\vert x,m}\Vert p\pp{\r z,x\vert m}}-\log p\pp{x\vert m}}\\
 & = & -\log\sum_{m}\frac{p\pp{m\vert x,t}}{p\pp{x\vert m}}\exp\pars{\mathrm{ELBO}\pp{q\pp{\r z\vert x,m}\Vert p\pp{\r z,x\vert m}}}\\
 & = & -\log\sum_{m}\frac{p\pp{m\vert t}}{p\pp{x\vert t}}\exp\pars{\mathrm{ELBO}\pp{q\pp{\r z\vert x,m}\Vert p\pp{\r z,x\vert m}}}\\
 & = & \log p\pp{x\vert t}-\log\E_{p\pp{\r m\vert t}}\exp\pars{\mathrm{ELBO}\pp{q\pp{\r z\vert x,\r m}\Vert p\pp{\r z,x\vert\r m}}}
\end{eqnarray*}

Where we have used that
\[
p\pp{m\vert x,t}=\frac{p\pp{m\vert t}p\pp{x\vert m}}{p\pp{x\vert t}}
\]
and so
\[
\frac{p\pp{m\vert x,t}}{p\pp{x\vert m}}=\frac{p\pp{m\vert t}}{p\pp{x\vert t}}.
\]
\end{proof}
To understand this, suppose each $q$ was exact so that 
\[
\mathrm{ELBO}\pp{q\pp{\r z\vert x,m}\Vert p\pp{\r z,x\vert m}}=\log p\pp{x\vert m}.
\]
Then
\begin{eqnarray*}
\KL{q\pp{\r z,\r m\vert x}}{p\pp{\r z,\r m\vert x,t}} & = & \log p\pp{x\vert t}-\log\E_{p\pp{\r m\vert t}}p\pp{x\vert\r m}\\
 & = & \log p\pp{x\vert t}-\log p\pp{x\vert t}\\
 & = & 0.
\end{eqnarray*}

\subsection{Proof with inexact ELBOS (\ref{thm:joint-divergence-approximate-ELBO})}

\label{subsec:Proof-with-inexact-ELBOs}
\begin{proof}
\end{proof}

\begin{thm}
Suppose that $p\pp{z,x,m\vert t}$ is fixed and $q\pp{z\vert x,m}=p\pp{z\vert x,m}$
is exact but $q\pp{m\vert x}\propto p\pp{m\vert t}p\pp{x\vert t}\exp\pp{-\delta^{(m)}}$
is not exact (e.g. because it is chosen based on ELBOs rather than
true marginal likelihoods). Then,
\begin{eqnarray*}
\KL{q\pp{z\vert x}}{p\pp{z\vert x,t}} & \leq & \KL{q\pp{m\vert x}}{p\pp{m\vert x,t}}\\
 & = & -\log\E_{p\pp{m\vert x,t}}\exp\pp{-\delta^{(m)}}-\E_{q\pp{m\vert x}}\delta^{(m)}.
\end{eqnarray*}
\end{thm}

\begin{proof}
\begin{eqnarray*}
\KL{q\pp{m\vert x}}{p\pp{m\vert x,t}} & = & \E_{q\pp{m\vert x}}\log\frac{q\pp{m\vert x}}{p\pp{m\vert x,t}}\\
 & = & \E_{q\pp{m\vert x}}\log\frac{p\pp{m\vert t}p\pp{x\vert m}\exp\pp{-\delta^{(m)}}}{p\pp{m\vert t}p\pp{x\vert m}}+\log\frac{\sum_{m}p\pp{m\vert t}p\pp{x\vert m}}{\sum_{m}p\pp{m\vert t}p\pp{x\vert m}\exp\pp{-\delta^{(m)}}}\\
 & = & \E_{q\pp{m\vert x}}\bracs{-\delta^{(m)}}+\log\frac{\sum_{m}p\pp{m\vert t}p\pp{x\vert m}}{\sum_{m}p\pp{m\vert t}p\pp{x\vert m}\exp\pp{-\delta^{(m)}}}\\
 & = & \log\frac{\sum_{m}p\pp{m\vert t}p\pp{x\vert m}}{\sum_{m}p\pp{m\vert t}p\pp{x\vert m}\exp\pp{\bar{\delta}-\delta^{(m)}}}\\
 & = & \log\frac{p\pp{x\vert t}}{\sum_{m}p\pp{m\vert t}p\pp{x\vert m}\exp\pp{\bar{\delta}-\delta^{(m)}}}\\
 & = & -\log\E_{p\pp{m\vert x,t}}\exp\pp{\bar{\delta}-\delta^{(m)}}\\
 & = & -\log\E_{p\pp{m\vert x,t}}\exp\pp{-\delta^{(m)}}-\bar{\delta}
\end{eqnarray*}
\end{proof}
Jensen's inequality ($-\log$ is convex) gives us the easy upper bound
\begin{eqnarray*}
\KL{q\pp{m\vert x}}{p\pp{m\vert x,t}} & = & -\log\E_{p\pp{m\vert x,t}}\exp\pp{\bar{\delta}-\delta^{(m)}}\\
 & \leq & -\E_{p\pp{m\vert x,t}}\log\exp\pp{\bar{\delta}-\delta^{(m)}}\\
 & = & \E_{p\pp{m\vert x,t}}\delta^{(m)}-\E_{q\pp{m\vert x}}\delta^{(m)}
\end{eqnarray*}

To understand this, note that the difference of$p\pp{m\vert x,t}$
and $q\pp{m\vert x}$ is precisely that $q\pp{m\vert x}$ is smaller
when $\delta^{(m)}$ is larger.

\cleardoublepage{}

\section{Experimental details\label{subsec:Experimental-Details}}

Models were generated using Llama-3.3-70B \citep{Grattafiori_2024_Llama3Herd,Meta_2024_Llama33Model}
with 4-bit AWQ quantization \citep{lin2023awq}. In testing, LLMs
seemed much better at writing Stan code \citep{Carpenter_2017_StanProbabilisticProgramming}
than other PPLs like NumPyro \citep{Phan_2019_ComposableEffectsFlexible}
or PyMC \citep{Abril-Pla_2023_PyMCModernComprehensive}, possibly
due to more code being available. We thus used Stan, though the possibility
of creating unnormalized distributions in Stan poses some difficulties
\eqref{subsec:Validating-models}.

The LLM system prompt \eqref{sec:System-prompt} explains the PROBLEM
/ DATA / GOAL format illustrated in the experiments and asks the LLM
to first write a THOUGHTS block explaining how it plans to model the
problem, followed by a MODEL block of Stan code. For in-context learning,
we provided six example inputs along with high quality outputs \eqref{sec:examples}.
We rejected any models that did not compile or that used certain language
constructs that might lead to unnormalized models \eqref{subsec:Validating-models}.

Models were generated using a single A100. With continuous batching
(parallel inference), model generation was reasonably fast---for
example, it took around 14 minutes to generate 1000 models for the
city temperature problem and around 11 minutes for the polling problem.

For inference, it is necessary to reliably and automatically approximate
posteriors and marginal likelihoods for thousands of models. To approximate
the posterior, we simply used Stan's default NUTS \citep{Hoffman_2014_NoUturnSamplerAdaptively}
sampler with 10,000 iterations. To approximate the marginal likelihood,
we used importance-weighted variational bounds \citep{Burda_2015_ImportanceWeightedAutoencoders},
with a Gaussian proposal distribution. In general, it is known that
such bounds are improved when the proposal distribution would minimize
the $\chi^{2}$ divergence to the target. However, given the difficulty
of minimizing such divergences \citet{Geffner_2021_DifficultyUnbiasedAlpha}
we elected instead to use a proposal distribution $q\pp{z\vert x,m}$
minimizing $\KL{p\pp{\r z\vert x,m}}{q\pp{\r z\vert x,m}},$ which,
given samples from $p$, amounts to moment matching, i.e. matching
the empirical mean and variance of the samples from MCMC. The full
inference strategy is given explicitly as \ref{alg:hybrid-MCMC-VI-SNIS}
.

\begin{algorithm}[H]
\begin{enumerate}
\item Input textual description $t$ and data $x$.
\item For $n=1,2,\cdots,N$:
\begin{enumerate}
\item Sample $m^{(n)}\sim p\pp{m|t}$ using an LLM.
\item Draw samples $z^{(n,1)},\cdots z^{(n,K)}\sim p\pp{z\vert x,m^{(n)}}$
using MCMC.
\item Set $q\pp{z\vert x,m^{(n)}}$ to be a Gaussian with mean and covariance
matching the sample $z^{(n,1)},\cdots,z^{(n,K)}$.
\item Estimate the importance-weighted ELBO $L^{(n)}=\hat{\mathbb{E}}\log\hat{\E}\ \frac{p\pp{\r z\vert x,m^{(n)}}}{q\pp{\r z\vert x,m^{(n)}}},$
where the inner expectation is estimated using 25 samples $\r z\sim q$
and the outer approximate expectation uses 10,000 repetitions.
\end{enumerate}
\item Set $w^{(n,k)}\propto\exp L^{(n)}$.
\item Return the set of samples $\{z^{(n,k)}\}$ where $z^{(n,k)}$ is given
weight $w^{(n)}.$
\end{enumerate}
\caption{The hybrid MCMC / VI / SNIS algorithm used in the experiments.\label{alg:hybrid-MCMC-VI-SNIS}}
\end{algorithm}

Measuring the amount of compute used for inference is difficult since
it was run on a heterogeneous cluster. However, running MCMC inference
on 1000 models is of course quite expensive and required tens of hours
of CPU time for each model. Simply compiling 1000 Stan models (on
CPU) is more time consuming than generating 1000 models using an LLM,
though of course both of these steps are embarrassingly parallel.

\cleardoublepage{}

\subsection{System prompt\label{sec:System-prompt}}

The following system prompt was used when generating all models:

\begin{lstlisting}[
    basicstyle=\small\ttfamily,
    frame=single,
    rulecolor=\color{gray!100},
    backgroundcolor=\color{gray!03},
    breaklines=true,          
    xleftmargin=0pt,          
    breakindent=0\baselineskip,
    lineskip=-3pt,
]
You are StanWriter. Given a description of a problem, you write Stan code.

You always first write THOUGHTS and then describe your model in words. Then you write MODEL and write the Stan code.

Be creative when coming up with your model!

When you declare arrays, ALWAYS use the syntax "array[2,5] int x" NEVER the old syntax "int x[2,5]".

In the generated quantities block ALWAYS use the _rng syntax like "x ~ normal_rng(mu,sigma)" NEVER "x ~ normal(mu,sigma)".

NEVER write ```stan before your code.

ALWAYS give priors for all variables. NEVER use implicit/improper priors.

NEVER use "target += ..." syntax.
\end{lstlisting}

\subsection{Validating models\label{subsec:Validating-models}}

We rejected any model that did not contain the string \texttt{MODEL},
or where the text after model failed to compile.

Using Stan for the purpose of this algorithm, has one significant
disadvantage: It is easy in Stan to create models that are not normalized.
This of course poses no issue when doing MCMC sampling, but is significant
here since it can affect the marginal likelihood. Unnormalized models
can be created in several ways:
\begin{enumerate}
\item By default, Stan uses flat / unnormalized / improper priors.
\item One can use the same variable in a sampling statement more than once,
e.g. one may write ``\texttt{x \textasciitilde{} normal(0,10)}''
and then later write ``\texttt{x \textasciitilde{} normal(3,5)}''.
Both of these statement increment the density, meaning the result
is no longer normalized.
\item One can directly manipulate the target using the ``\texttt{target
+=} '' construct, essentially leading to arbitrary changes with no
probabilistic interpretation.
\item One can transform variables on the ``left-hand side'' e.g. write
``\texttt{log(x) \textasciitilde{} normal(0,1)}''. This again leads
to an unnormalized density.
\item If variables are constrained in Stan, this constraint is not reflected
in the normalization.
\end{enumerate}
Except for the \#5, none of these are too critical, and one can use
a system prompt (and in-context examples) that discourages use. This
works quite well, but (probably due to frequent use in example models)
\#3 remains common. We thus reject all models that contain the string
``\texttt{target +=} '', as if they did not compile. Checking usage
of \#1, \#2, and \#4 is not trivial to do programmatically, and thus
we relied instead on spot checks in models given high weight. Only
a handful of cases were found.

This leaves issue \#5, which unfortunately is difficult to avoid due
to frequent usage of constraints in practice, e.g. to sample from
truncated normal distributions. We simply accept this bias in practice,
though note the bias is always downwards, e.g. towards a lower marginal
likelihood.

\cleardoublepage{}

\section{Examples\label{sec:examples}}

This section provides the example inputs and outputs given to the
LLM to aid in-context learning.

\subsection{Weight}
\begin{center}
\textbf{INPUT}
\par\end{center}

\begin{lstlisting}[
    basicstyle=\scriptsize\ttfamily,
    frame=single,
    rulecolor=\color{gray!100},
    backgroundcolor=\color{gray!03},
    breaklines=true,          
    xleftmargin=0pt,          
    breakindent=0\baselineskip,
]
PROBLEM

Given the weight of a bunch of students, infer true mean of all students. Use reasonably informative priors.

DATA

int num_students; // number of students
array[num_students] real student_weight; // weight of each student in kilograms

GOAL

real weight_mean; // true mean weight for all students
\end{lstlisting}
\begin{center}
\textbf{OUTPUT}
\par\end{center}

\begin{lstlisting}[
    basicstyle=\scriptsize\ttfamily,
    frame=single,
    rulecolor=\color{gray!100},
    backgroundcolor=\color{gray!03},
    breaklines=true,          
    xleftmargin=0pt,          
    breakindent=0\baselineskip,
]
THOUGHTS

I will model this problem by creating two latent variables. `weight_mean` will reflect the true mean weight of all students, and `weight_sigma` will reflect the standard deviation. Since I was instructed to use reasonably informative priors, I'll assume a normal prior for `weight_mean` with a mean of 80kg and a standard deviation of 20kg. For the `weight_sigma`, I'll choose a lognormal prior with mu=log(10) and sigma=log(20). This reflects that my best guess for the standard deviation is around 10 kg, but could be significantly larger or smaller.

MODEL

data {
	int num_students;
	array[num_students] real student_weight;
}
parameters {
	real weight_mean;
	real<lower=0> weight_sigma;
}
model {
	weight_mean ~ normal(80,20);
	weight_sigma ~ lognormal(log(10),log(20));
	for (i in 1:num_students){
		student_weight[i] ~ normal(weight_mean, weight_sigma);
	}
}
\end{lstlisting}

\cleardoublepage{}

\subsection{Future weight}
\begin{center}
\textbf{INPUT}
\par\end{center}

\begin{lstlisting}[
    basicstyle=\scriptsize\ttfamily,
    frame=single,
    rulecolor=\color{gray!100},
    backgroundcolor=\color{gray!03},
    breaklines=true,          
    xleftmargin=0pt,          
    breakindent=0\baselineskip,
]
PROBLEM

Given the weight of a bunch of students, infer weight of a new student drawn from the same distribution. Use reasonably informative priors.

DATA

int num_students; // number of students
array[num_students] real student_weight; // weight of each student in kilograms

GOAL

real new_student_weight; // weight for a future student, sampled from the same distribution
\end{lstlisting}
\begin{center}
\textbf{OUTPUT}
\par\end{center}

\begin{lstlisting}[
    basicstyle=\scriptsize\ttfamily,
    frame=single,
    rulecolor=\color{gray!100},
    backgroundcolor=\color{gray!03},
    breaklines=true,          
    xleftmargin=0pt,          
    breakindent=0\baselineskip,
]
data {
	int num_students;
	array[num_students] real student_weight;
}
parameters {
	real weight_mean;
	real<lower=0> weight_sigma;
}
model {
	weight_mean ~ normal(70,30); 
	weight_sigma ~ student_t(2.5,0,5); // heavy tails
	for (i in 1:num_students){
		student_weight[i] ~ normal(weight_mean, weight_sigma);
	}
}
generated quantities {
	real new_student_weight = normal_rng(weight_mean, weight_sigma);
}
\end{lstlisting}

\cleardoublepage{}

\subsection{Movies}
\begin{center}
\textbf{INPUT}
\par\end{center}

\begin{lstlisting}[
    basicstyle=\scriptsize\ttfamily,
    frame=single,
    rulecolor=\color{gray!100},
    backgroundcolor=\color{gray!03},
    breaklines=true,          
    xleftmargin=0pt,          
    breakindent=0\baselineskip,
]
PROBLEM

Given a bunch of binary features of a movie, predict if a user will like it or not.

DATA

int num_train; // number of movies to train on
int num_test; // number of movies to test on
int num_features; // number of binary features for each movie
array[num_train, num_features] int train_features; // binary features
array[num_train] int train_like; // did the user like the movie
array[num_train, num_features] int test_features; // binary features

GOAL

array[num_test] int test_like; // does the user like each movie in the test set?

\end{lstlisting}
\begin{center}
\textbf{OUTPUT}
\par\end{center}

\begin{lstlisting}[
    basicstyle=\scriptsize\ttfamily,
    frame=single,
    rulecolor=\color{gray!100},
    backgroundcolor=\color{gray!03},
    breaklines=true,          
    xleftmargin=0pt,          
    breakindent=0\baselineskip,
]
THOUGHTS

I will model this problem by creating a vector `beta` of size `num_features`. When the inner-product of this vector is taken with the features of a movie, it gives the logit for how likely the user is to like that movie. I'll sample each entry of that vector from a standard normal, which seems like a reasonable scale when being combined with binary features. I'll then sample both `train_like` and `test_like` by taking the inner-product of `beta` with each row of `train_features` and `test_features`, respectively.

MODEL

data {
	int num_train;
	int num_test;
	int num_features;
	array[num_train, num_features] int<lower=0, upper=1> train_features;
	array[num_train] int<lower=0, upper=1> train_like;
	array[num_test, num_features] int<lower=0, upper=1> test_features;
}
parameters {
	vector[num_features] beta;
}
model {
	beta ~ normal(0,1);
	for (n in 1:num_train){
	    train_like[n] ~ bernoulli_logit(to_row_vector(train_features[n]) * beta);
	}
}
generated quantities {
	array[num_test] int<lower=0, upper=1> test_like;
	for (n in 1:num_test){
		test_like[n] = bernoulli_logit_rng(to_row_vector(test_features[n]) * beta);
	}
}
\end{lstlisting}

\cleardoublepage{}

\subsection{Collaborative filtering}
\begin{center}
\textbf{INPUT}
\par\end{center}

\begin{lstlisting}[
    basicstyle=\scriptsize\ttfamily,
    frame=single,
    rulecolor=\color{gray!100},
    backgroundcolor=\color{gray!03},
    breaklines=true,          
    xleftmargin=0pt,          
    breakindent=0\baselineskip,
]
PROBLEM

Given a bunch of like ratings of various users of various movies, predict if users will like future movies. Do this by inferring underlying features of users and movies.

DATA

int num_users; // number of users
int num_movies; // number of movies
int num_ratings; // number of observed ratings
array[num_ratings] int user; // what user did rating
array[num_ratings] int movie; // what movie was rated
array[num_train, num_features] int like; // did user like the movie (1 if yes, 0 if no)

GOAL

int array[num_users, num_movies] all_ratings; // would each user like each movie?

\end{lstlisting}
\begin{center}
\textbf{OUTPUT}
\par\end{center}

\begin{lstlisting}[
    basicstyle=\scriptsize\ttfamily,
    frame=single,
    rulecolor=\color{gray!100},
    backgroundcolor=\color{gray!03},
    breaklines=true,          
    xleftmargin=0pt,          
    breakindent=0\baselineskip,
]
THOUGHTS

I will model this problem by assuming that both users and movies can be described in terms of `num_features`=10 features. I'll create a latent variable `user_features` that describes the features for each user and a latent variable `movie_features` that describes the features for each movie. I will assume standard cauchy priors for all these features, indicating that most features are small, but some might be quite large. Finally, I assume that the probability of a given user liking a given movie is given by a bernoulli_logit distribution with a score consisting of the inner-product of the user and movie features. To create the desired output `all_ratings`, I will create a generating quantities block in which I loop over all pairs of users and movies.

MODEL

data {
	int num_users; // number of users
	int num_movies; // number of movies
	int num_ratings; // number of observed ratings
	array[num_ratings] int user; // what user did rating
	array[num_ratings] int movie; // what movie was rated
	array[num_ratings] int like; // did user like the movie (1 if yes, 0 if no)
}
transformed data{
	int num_features = 10;	
}
parameters {
	array[num_users, num_features] real user_features;
	array[num_movies, num_features] real movie_features;
}
model {
	for(k in 1:num_features){
		for(i in 1:num_users){
			user_features[i,k] ~ cauchy(0,1);
		}
		for(j in 1:num_movies){
			movie_features[j,k] ~ cauchy(0,1);
		}
	}
	for (n in 1:num_ratings){
		array[num_features] real x = user_features[user[n]];
		array[num_features] real y = movie_features[movie[n]];
		real score = dot_product(x,y);
		like[n] ~ bernoulli_logit(score);
	}
}
generated quantities {
	array[num_users, num_movies] int all_ratings;
	for(i in 1:num_users){
		for(j in 1:num_movies){
			array[num_features] real x = user_features[i];
			array[num_features] real y = movie_features[j];
			real score = dot_product(x,y);
			all_ratings[i,j] = bernoulli_logit_rng(score);
		}
	}
}				
\end{lstlisting}

\cleardoublepage{}

\subsection{Vaccines}
\begin{center}
\textbf{INPUT}
\par\end{center}

\begin{lstlisting}[
    basicstyle=\scriptsize\ttfamily,
    frame=single,
    rulecolor=\color{gray!100},
    backgroundcolor=\color{gray!03},
    breaklines=true,          
    xleftmargin=0pt,          
    breakindent=0\baselineskip,
]
PROBLEM

There is data for a bunch of different vaccine trials. There are different types of vaccines. Given this data, infer the mean response rate for each trial type. Assume that different vaccines of the same type tend to have similar response rates and that different types of vaccines also usually have response rates that aren't that different.

DATA

int num_types; // number of different types of vaccines
int num_trials; // number of trials
array[num_trials] int type; // what type of vaccine tested in each trial
array[num_trials] int num_subjects; // number of subjects in each trial
array[num_trials] int responders; // number of subjects in each trial that responded

GOAL

array[num_types] response_rate; // mean response rate for a new vaccine of each type
\end{lstlisting}
\begin{center}
\textbf{OUTPUT}
\par\end{center}

\begin{lstlisting}[
    basicstyle=\scriptsize\ttfamily,
    frame=single,
    rulecolor=\color{gray!100},
    backgroundcolor=\color{gray!03},
    breaklines=true,          
    xleftmargin=0pt,          
    breakindent=0\baselineskip,
]
THOUGHTS

I'll create a hierarchical model. The idea is that each trial has some response rate that is drawn from a per-trial-type distribution. Then the parameters for the per-trial-type distributions are all drawn from some global distribution. By sampling the three levels at once (individual trial response rates, per-trial-type response rates, and global rates) all information can be optimally shared.

In more detail, I'll first create parameters `a` and `b` for each trial type. These determine the distribution over true response rates for different trials of each type. The true response rate for trial `i`, `per_trial_response_rate[i]`, is drawn from a beta distribution with parameters `a[type[i]]` and `b[type[i]]`. Then the (observed) number of responders in each trial will be drawn from a binomial distribution with `num_subjects[i]` repeats and a `per_trial_response_rate[i]` success rate. To create a prior over `a` and `b`, I will draw a 2-dimensional vector `z[n]` for each trial type from some multi normal distribution parameterized by `mu` and `L`. This will be mapped to `a` and `b` in the transformed parameters block through an inverse logit transformation. I'll sample `mu` from a broad multivariate normal distribution and `L` from the distribution that results from sampling from a broad inverse Wishart distribution and then computing the Cholesky factorization.

MODEL

data {
	int num_types;
	int num_trials;
	array[num_trials] int type;
	array[num_trials] int num_subjects;
	array[num_trials] int responders;
}
parameters {
	cholesky_factor_cov[2] L;
	vector[2] mu;
	array[num_types] vector[2] z; // parameters for response rate dist for each type
	array[num_trials] real<lower=0, upper=1> per_trial_response_rate;
}
transformed parameters{
	// map unconstrained parameters for each dist type to pair on 0-1 interval
	array[num_types] real a;
	array[num_types] real b;
	array[num_types] real response_rate;
	for(n in 1:num_types){
		a[n] = inv_logit(z[n,1]);
		b[n] = inv_logit(z[n,2]);
		response_rate[n] = a[n]/(a[n]+b[n]); // mean of beta distribution is a/(a+b)
	}
}
model {
	mu ~ multi_normal([0,0],[[5,0],[0,5]]);
	L ~ inv_wishart_cholesky(3.0, [[1,0],[0,1]]);
	for(n in 1:num_types){
		z[n] ~ multi_normal_cholesky(mu, L);
	}
	for(i in 1:num_trials){
		per_trial_response_rate[i] ~ beta(a[type[i]], b[type[i]]);
		responders[i] ~ binomial(num_subjects[i], per_trial_response_rate[i]);
	}
}
\end{lstlisting}

\cleardoublepage{}

\subsection{Water}
\begin{center}
\textbf{INPUT}
\par\end{center}

\begin{lstlisting}[
    basicstyle=\scriptsize\ttfamily,
    frame=single,
    rulecolor=\color{gray!100},
    backgroundcolor=\color{gray!03},
    breaklines=true,          
    xleftmargin=0pt,          
    breakindent=0\baselineskip,
]
PROBLEM

I start with 1 kg of water. Every day, some fraction of the water (between 0 and 1) evaporates, and I measure how much is remaining. My measurements have a noise of something around 10 g. Predict how much water evaporates every day. It could be anything between 0 and 1, but I suspect it's something around 0.25. The first day with 1kg isn't measured.

DATA

int num_days;
array[num_days] measurements; // measurements in kg

GOAL

real evaporation; // what fraction evaporates each day?

\end{lstlisting}
\begin{center}
\textbf{OUTPUT}
\par\end{center}

\begin{lstlisting}[
    basicstyle=\scriptsize\ttfamily,
    frame=single,
    rulecolor=\color{gray!100},
    backgroundcolor=\color{gray!03},
    breaklines=true,          
    xleftmargin=0pt,          
    breakindent=0\baselineskip,
]
THOUGHTS

I will model this problem by creating a latent variable `evaporation` which measures what fraction of the water disappears each day. I will sample this from a beta distribution with alpha=1 and beta=3 meaning a mean of around 0.25. Given the evaporation, I will use a for loop to define an array `current_water` which is the true (unobserved) amount of water remaining after each day. Finally, I will sample `measurements[i]` for each day i from a normal distribution with a mean of water[i] and a standard deviation of 0.01, chosen because 0.01 kg = 10g.

MODEL

data {
	int num_days;
	array[num_days] real measurements;
}
parameters {
	real<lower=0, upper=1> evaporation;
}
transformed parameters{
	array[num_days] real<lower=0, upper=1> water;
	real current_water = 1.0;
	for(i in 1:num_days){
		current_water = current_water * (1-evaporation);
		water[i] = current_water;
	}
}
model {
	evaporation ~ beta(1,3);
	for(i in 1:num_days){
		measurements[i] ~ normal(water[i], 0.01);
	}
}
\end{lstlisting}

\onecolumn

\section*{NeurIPS Paper Checklist}

\textbf{1. Claims}

Question: Do the main claims made in the abstract and introduction
accurately reflect the paper's contributions and scope?

Answer: \answerYes{}

Justification: The abstract and introduction reflect the contribution
and scope of the paper, and do not include aspirational goals.

\textbf{2. Limitations}

Question: Does the paper discuss the limitations of the work performed
by the authors?

Answer: \answerYes{}

Justification: 

\textbf{3. Theory assumptions and proofs}

Question: For each theoretical result, does the paper provide the
full set of assumptions and a complete (and correct) proof?

Answer: \answerNA{}

Justification: Yes, typically in appendix.

\textbf{4~ Experimental result reproducibility}

Question: Does the paper fully disclose all the information needed
to reproduce the main experimental results of the paper to the extent
that it affects the main claims and/or conclusions of the paper (regardless
of whether the code and data are provided or not)?

Answer: \answerYes{}

Justification: All details disclosed.

\textbf{5. Open access to data and code}

Question: Does the paper provide open access to the data and code,
with sufficient instructions to faithfully reproduce the main experimental
results, as described in supplemental material?

Answer: \answerNo{}

Justification: It is difficult to provide code runnable by a third
party given the usage of a local cluster and many non-portable modifications.
Every effort has been made to make the results reproducible from the
given description.

\textbf{6. Experimental setting/details}

Question: Does the paper specify all the training and test details
(e.g., data splits, hyper-parameters, how they were chosen, type of
optimizer, etc.) necessary to understand the results?

Answer: \answerYes{}

Justification: 

\textbf{7. Experiment statistical significance}

Question: Does the paper report error bars suitably and correctly
defined or other appropriate information about the statistical significance
of the experiments?

Answer: \answerYes{}

Justification: 

\textbf{8. Experiments compute resources}

Question: For each experiment, does the paper provide sufficient information
on the computer resources (type of compute workers, memory, time of
execution) needed to reproduce the experiments?

Answer: \answerYes{}

Justification: 

\textbf{9. Code of ethics}

Question: Does the research conducted in the paper conform, in every
respect, with the NeurIPS Code of Ethics?

Answer: \answerYes{}

Justification: 

\textbf{10. Broader impacts}

Question: Does the paper discuss both potential positive societal
impacts and negative societal impacts of the work performed?

Answer: \answerNA{}

Justification: This paper does not appear to present significant social
impacts beyond the core goal (which is discussed) of making probabilistic
models easier to write.

\textbf{11. Safeguards}

Question: Does the paper describe safeguards that have been put in
place for responsible release of data or models that have a high risk
for misuse (e.g., pretrained language models, image generators, or
scraped datasets)?

Answer: \answerNA{}

Justification:

\textbf{12. Licenses for existing assets}

Question: Are the creators or original owners of assets (e.g., code,
data, models), used in the paper, properly credited and are the license
and terms of use explicitly mentioned and properly respected?

Answer: \answerYes{}

Justification: 

\textbf{13. New assets}

Question: Are new assets introduced in the paper well documented and
is the documentation provided alongside the assets?

Answer: \answerNA{}

Justification:

\textbf{14. Crowdsourcing and research with human subjects}

Question: For crowdsourcing experiments and research with human subjects,
does the paper include the full text of instructions given to participants
and screenshots, if applicable, as well as details about compensation
(if any)?

Answer: \answerNA{}

Justification: 

\textbf{15. Institutional review board (IRB) approvals or equivalent
for research with human subjects}

Question: Does the paper describe potential risks incurred by study
participants, whether such risks were disclosed to the subjects, and
whether Institutional Review Board (IRB) approvals (or an equivalent
approval/review based on the requirements of your country or institution)
were obtained?

Answer: \answerNA{}

Justification: 

\textbf{16. Declaration of LLM usage}

Question: Does the paper describe the usage of LLMs if it is an important,
original, or non-standard component of the core methods in this research?
Note that if the LLM is used only for writing, editing, or formatting
purposes and does not impact the core methodology, scientific rigorousness,
or originality of the research, declaration is not required

Answer: \answerYes{}

Justification: It was used only for generating probabilistic models
as described clearly in the text. LLMs were not used in any other
fashion, e.g. for generating code or writing.
\end{document}